\documentclass[11pt]{article}
\usepackage{amssymb}
\usepackage{amsmath,amsfonts,amsthm}
	\usepackage{subcaption}
	\usepackage{multirow}
\usepackage{graphicx}
\usepackage{float}
\usepackage{hyperref}
\usepackage{natbib}
\usepackage{bbm}
\usepackage{natbib}
\usepackage{url}
\usepackage{tikz}
\usetikzlibrary{positioning,shapes.geometric}

\usepackage[mathscr]{eucal}
\usepackage{mathrsfs}

\newcommand{\be}{\begin{equation}}
\newcommand{\ee}{\end{equation}}
\newtheorem{lemma}{Lemma}
\newtheorem{theorem}{Theorem}

\newtheorem{assumption}{Assumption}

\newtheorem{remark}{Remark}

\usepackage[singlelinecheck=false]{caption}
\usepackage{enumerate}
\usepackage{amsmath}
\usepackage{natbib}
\bibpunct[, ]{(}{)}{,}{a}{}{,}%
\newcommand{\bProof}{\begin{proof}{Proof.}}
\newcommand{\eProof}{\hfill\Halmos\\  \end{proof}}
\usepackage{algorithm}
\usepackage{algorithmic}
\newcommand{\bea}{\begin{equation*}}
\newcommand{\eea}{\end{equation*}}

\newcommand{\gao}[1]{\textcolor{blue}{{#1}}} % insert
\usepackage{caption}

\usepackage{comment}

\begin{document}

\begin{center}
		\Large \bf Regret Bounds for Markov Decision Processes with Recursive Optimized Certainty Equivalents

	\end{center}
	\author{}
	\begin{center}
	{Wenhao Xu}\,\footnote{Department of Systems
			Engineering and Engineering Management, The Chinese University of Hong Kong, Hong Kong, China.
			Email: \url{whxu@se.cuhk.edu.hk.} },
		Xuefeng Gao\,\footnote{Department of Systems
			Engineering and Engineering Management, The Chinese University of Hong Kong, Hong Kong, China. Email: \url{xfgao@se.cuhk.edu.hk.}},
        Xuedong He\,\footnote{Department of Systems
			Engineering and Engineering Management, The Chinese University of Hong Kong, Hong Kong, China. Email: \url{xdhe@se.cuhk.edu.hk.}
		}
	\end{center}
	
	\begin{center}
		\today
	\end{center}

%% --------------------------------------------------------------------------------------------------------------
\begin{abstract}
The optimized certainty equivalent (OCE) is a family of risk measures that cover important examples such as entropic risk, conditional value-at-risk and mean-variance models. In this paper, we propose a new episodic risk-sensitive reinforcement learning formulation based on tabular Markov decision processes with recursive OCEs. We design an efficient learning algorithm for this problem based on value iteration and upper confidence bound. We derive an upper bound on the regret of the proposed algorithm, and also establish a minimax lower bound. Our bounds show that the regret rate achieved by our proposed algorithm has optimal dependence on the number of episodes and the number of actions. 
\end{abstract}

% and the number of actions. ({\bf Is it standard to study the error bound on number of action?})

%\textbf{Keywords.}

\section{Introduction}
Reinforcement learning (RL) studies the problem of sequential decision making in an unknown environment by carefully balancing between exploration and exploitation \citep{sutton2018reinforcement}. In the classical setting, it describes how an agent takes actions to maximize
\textit{expected cumulative rewards} in an environment typically modeled by a Markov decision
process (MDP, \cite{puterman2014markov}). %This \textit{risk-neutral} RL has been extensively studied in the past decades. 
However, 
optimizing the expected cumulative rewards alone is often not sufficient in many practical applications such as finance, healthcare and robotics.  Hence, it may be necessary to take into account of the risk preferences of the agent in the dynamic decision process. Indeed, a rich body of literature has studied \textit{risk-sensitive} (and safe) RL, incorporating risk measures such as the entropic risk measure and conditional value-at-risk (CVaR) in the decision criterion, see, e.g., \cite{shen2014risk, garcia2015comprehensive, tamar2016learning,chow2017risk, prashanth2018risk, fei2020risk} and the references therein.

%Despite this line of literature, regret bounds for risk-sensitive RL have been very limited. 

In this paper we study risk-sensitive RL for tabular MDPs with unknown transition probabilities
in the \textit{finite-horizon, episodic} setting, where an agent interacts with the MDP in episodes of a fixed length with finite state and action spaces. 
To incorporate risk sensitivity, we consider a broad and important class of risk measures known as Optimized Certainty Equivalent (OCE, \citep{ben1986expected, ben2007old}). The OCE is a (nonlinear) risk function which assigns a random variable $X$ to a real value, and it depends on a \textit{concave} utility function, see Equation \eqref{OCE} for the definition. 
%the maximum present value of a combination of the cash to be taken out from the uncertain income at present and the expected utility value of the
%remaining uncertain income. 
With an appropriate choice of the \textit{utility function}, OCE covers important examples of risk measures, including the entropic risk, CVaR and mean-variance models, as special cases, so it is popular in financial applications, such as portfolio optimization, and in the machine learning literature. See Section~\ref{SEC22} for details. Using this unified framework, we aim to develop efficient learning algorithms for risk-sensitive RL with OCEs and provide worst-case regret bounds, where the regret measures the sub-optimality of the learning algorithm compared to an optimal policy should the model parameters be completely known. 

% include risk measures
% in the learning process in order to

% There are substantial challenges in the problem formulation and algorithm design and analysis for risk-sensitive RL with OCE. 

We formulate a new risk-sensitive episodic RL problem with recursive OCEs.
The conventional objective in risk-sensitive MDPs (when the model is known) is to optimize a static risk measure/functional applied to the (possibly discounted) \textit{cumulative} rewards over the decision horizon \citep{ howard1972risk,marcus1997risk}. 
Except for the entropic risk measure, this approach typically suffers from the time-inconsistency issue, which prevents one from directly applying the dynamic programming principle \citep{artzner2007coherent}. In addition, the optimal policies can be non-Markovian, and are often difficult to compute \citep{mannor2011mean, du2022risk}. In view of the time-inconsistency issue and the computational difficulty, we consider an alternative approach, which is to consider MDPs with recursive risk measures  
\citep{ruszczynski2010risk,shen2013risk,bauerle2022markov}. In this approach, instead of optimizing a static risk measure of the cumulative rewards, one optimizes the value function
defined by a recursive application of a risk measure \textit{at each period}, which essentially replaces the expectation operator in the standard value iteration by the risk measure (OCEs in our setting).
This approach is also partly motivated by recursive utilities in the economics literature \citep{kreps1978temporal,epstein1989substitution}. Indeed, our recursive OCE model is a special case of the so-called dynamic mixture-averse preferences, which have an axiomatic foundation \citep{sarver2018dynamic} and is a special class of recursive utilities. The recursive structure in our OCE model implies time consistency and dynamic programming, which leads to a Bellman equation in a known environment; see for instance \citet{bauerle2022markov}. Our formulation of episodic RL with recursive OCEs is built on this Bellman equation. Due to the generality of OCE, our RL formulation unifies and generalizes several existing episodic RL formulations in the literature, including standard risk-neutral RL (see, e.g., \citet{azar2017minimax}), RL with entropic risk in \citet{fei2020risk}, and RL with iterated CVaR in \citet{du2022risk}. See Section 2.2 for details.

A special case of OCE is the entropic risk measure, which is obtained by setting the utility function in OCE to be an exponential function. In this special case, the recursive OCE model is equivalent to applying the entropic risk measure to the cumulative reward over the entire decision horizon. In general, the recursive OCE model and the model of applying OCE to the cumulative reward directly are different and can be applied in different problems to account for different attitudes toward risk. The former tends to lead to more conservative actions than the latter does, because in the former the agent is concerned about risk in every step and in every state; see \citet{du2022risk} for a detailed discussion and a concrete example in this regard in the context of recursive CVaR.

% We formulate a new risk-sensitive episodic RL problem with recursive OCE, building on studies of (known) MDPs with recursive measures \citep{bauerle2022markov}. 
% Instead of optimizing a static risk measure of the cumulative reward, we optimize value functions 
% defined by an recursive application of OCE at each period, which essentially replaces the expectation operator in the standard value iteration by OCE.

% Our formulation unifies and generalizes several existing episodic RL formulations in the literature, including risk-neutral RL (see, e.g., \cite{azar2017minimax}), RL with entropic risk in \cite{fei2021exponential}, and RL with iterated CVaR in \cite{du2022risk}. See Section 2.2 for details.

% One main challenge in risk-sensitive RL is to design risk-aware exploration mechanism. 
In this paper, we develop a model-based algorithm for risk-sensitive RL with recursive OCEs. Our algorithm is a  variant of the UCBVI (Upper Confidence Bound Value Iteration) algorithm in \cite{azar2017minimax} for risk-neural RL. The main novelty in our algorithm design is that the bonus term used to encourage exploration depends on the \textit{utility function} in the specific OCE that one considers. 
Theoretically, we prove regret bounds for our algorithm in learning MDPs with a wide family of recursive risk measures including the mean-variance criterion, by considering different utility functions in OCEs. Such bounds are new to the literature, to the best of our knowledge. 

% In the special case of CVaR, our bound either matches or improves previous bounds in \citep{du2022risk}, depending on the quantile level in CVaR. 

% Our bound iterated CVaR \citep{du2022risk} in the literature

% show that our regret upper bound either matches or improves state-of-art bounds for risk-sensitive RL with entropic risk \citep{fei2021exponential} and iterated CVaR \citep{du2022risk} in the literature. In addition to these two risk measures, we generate regret bounds for learning MDPs with a wide family of recursive risk measures including mean-variance, by considering different utility functions in OCEs. Such bounds are new to the literature, to the best of our knowledge.  %been studied in the literature. 
% See Section~\ref{sec:results} for details.

The regret analysis of algorithms for risk-sensitive RL is difficult mainly due to the nonlinearity of the objective \citep{fei2020risk}.  Although the structure of our regret analysis of the proposed algorithm follows the optimism principle in provably efficient risk-neutral RL (see, e.g., \cite{azar2017minimax, agarwal2019reinforcement}),  
we develop two new ingredients to overcome the difficulty in our risk-sensitive setting: (a) concentration bounds for the OCE of the next-state value function under the estimated transition distributions, and (b) a change-of-measure technique to bound the OCE of the estimated value function under the true transition distribution with an affine functional (see Equation~\eqref{eq:measure-change}).  Our concentration bounds for OCEs of value functions are different from recent results in \cite{la2022wasserstein} which rely on the Lipschitz continuity of the utility function. Our technique (b) is inspired by the regret analysis in \citep{du2022risk} for iterated CVaR, but it is much more general and thus is applicable to OCEs. Conceptually, the main insight is to use the fact that the OCE is a concave risk functional \citep[Theorem 2.1]{ben2007old}. Its (algebraic) subgradient is a linear functional \cite{ruszczynski2006optimization} which turns out to be in the form of an expectation with respect to a new probability distribution that is related to the true transition distribution via change-of-measure. 
This linearization method is crucial in carrying out the recursions (in the time parameter) in our regret analysis. Due to change-of-measure, the corresponding Radon-Nikodym derivative naturally appears in our analysis and we need to carefully bound it. 

% The regret analysis of the proposed algorithm is novel in two aspects: 

% to overcome the nonlinearity of the OCE, we provide two new technical ingredients in the proof: (1) concentration bounds for the OCE of the next-state value function, and (2) a measure-change technique to linearize the OCE which is a concave risk functional. 

%  OCE is a nonlinear functional of random variables and it is defined via an optimization problem. The existing methods for regret analysis of algorithms for episodic RL in the risk neutral setting critically rely on the fact that the expectation operator is linear (see, e.g., \cite{azar2017minimax, jin2018q}),  and so it is nontrivial to adapt these methods to our setting. 
% Technically, we develop a novel change-of-measure technique to cope with the non-linearity of OCE in the regret analysis of our algorithm. \gao{to add: concavity of OCE; functional derivative? tightness of the analysis and the relation to techniques in existing work}

In addition to the regret upper bound, we also establish a minimax lower bound. It shows that the regret rate achieved by our proposed algorithm has optimal dependence on the number of episodes $K$ and the number of actions $A$, up to logarithmic factors. The proof of our lower bound proof is built on the hard MDP instances constructed in \cite{domingues2021episodic} for tabular risk-neutral RL. The main novelty in our analysis lies in modifying such hard instances to adapt to the OCEs and bounding value functions which are defined recursively via OCEs that involve an optimization problem.

\subsection{Related Work}

Despite rich literature in risk-sensitive RL, there are fairly limited number of studies on regret minimization in risk-sensitive MDPs. We provide a concise review below, and leave the detailed comparisons of existing regret bounds (for entropic risk and CVaR only) with our bounds to Section~\ref{SECEX}.

To the best of our knowledge, the first regret bound for risk-sensitive tabular MDP is due to \cite{fei2020risk}, who study episodic RL with the goal of maximizing the \textit{entropic risk} of the cumulative rewards. By the pleasant properties of exponential functions in entropic risk, their RL formulation is in fact equivalent to our general (iterative) formulation when the OCE is entropic risk. 

The results in \cite{fei2020risk} have been improved in \cite{fei2021exponential} for tabular MDPs with entropic risk, where they design two model-free algorithms with improved regret bounds. In addition, these algorithms have been  extended to the function approximation setting in \cite{fei2021risk} and to non-stationary MDPs with variation budgets in \cite{ding2022non}. \cite{liang2022bridging} also consider RL with the entropic risk, and they use tools from  distributional RL \citep{bellemare2017distributional}. They propose algorithms with regret upper bounds matching the results in \cite{fei2021exponential}.

% Several works study the episodic risk-sensitive RL with different risk objective functions and provide the regret bounds for their algorithms under the principle of Optimism in the Face of Uncertainty (OFU). \cite{fei2020risk} study the episodic risk-sensitive RL with entropic risk objective function and work out the regret upper bounds for their RSVI and RSQ algorithms. They also construct a bandit model to derive the regret lower bound and prove the indispensability of the exponential coefficient in their regret upper bounds. \cite{fei2021exponential} propose exponential Bellman equation and improve the regret upper bounds for RSVI and RSQ algorithms. \cite{liang2022bridging} study the distributional RL with entropic risk objective function and obtain regret upper bounds for two DRL algorithms, which matches the improved results in \cite{fei2021exponential}. \cite{liang2022bridging} also improve the regret lower bound in \cite{fei2020risk} by adapting the hard MDP instances in \cite{domingues2021episodic}. 

\cite{du2022risk} propose Iterated CVaR RL, which is an episodic risk-sensitive RL formulation with the objective of maximizing the tail of the reward-to-go at each step. Their RL formulation is a special case of ours. \cite{du2022risk} study both regret miminization and best policy identification, and provide matching upper and lower bounds with respect to the number of episodes.

All the aforementioned studies focus on one single risk measure (entropic risk or CVaR only) for regret analysis in risk-sensitive MDPs. Their algorithms and analysis typically rely on the properties of the special risk measure they consider. 
\cite{bastani2022regret} study episodic RL with a class of risk-sensitive objectives known as \textit{spectral risk measures}, which includes CVaR as an example (but not the entropic risk and mean-variance criterion). 
They develop an upper-confidence-bound style algorithm and obtain a regret upper bound for their algorithm. Although spectral risk measures cover CVaR as an example, their work is different from \cite{du2022risk} and ours in that their objective is to optimize the (static) spectral risk of the cumulative rewards, rather than the value function obtained from iterative application of the risk measure at each time step.  
 See Appendix A in \cite{du2022risk} and Section 3 in \cite{bastani2022regret} for further discussions.

 \textbf{Paper Organization.} The rest of the paper is organized as follows: 
 we present the problem formulation in Section~\ref{sec:formulation} and describe the algorithm in Section~\ref{SECALG}. We state and discuss the main results in Section~\ref{sec:results}. We provide the proof sketch of our regret upper bound in Section~\ref{sec:sketch}, and conclude in Section~\ref{sec:conclusion}. Due to space constraints, proofs and experiments are given in the Appendix. 

% which has Markovian optimal policies compared with the CVaR RL. They also provide matching regret upper and lower bounds for ICVaR-RM and ICVaR-BPI algorithms. \cite{bastani2022regret} study the episodic risk-sensitive RL with spectral risk objective function and derive the regret upper bound for their Upper Confidence Bound Algorithm, but they haven't given a corresponding regret lower bound. Although spectral risk is the generalization of CVaR, their work is different from the work of \cite{du2022risk} in that they directly use the cumulative rewards and the corresponding optimal policies might be non-Markov, see Appendix A in \cite{du2022risk} and section 3 in \cite{bastani2022regret} for a detailed discussion.

%\gao{WH todo. Nips 2022 paper regret bound on spectral risk; is there any lower bound on risk based bandit?}

% General risk sensitive RL: 

% \subsection{Paper Organization.} In Section~\ref{sec:formulation}, we present the problem formulation. In Section~\ref{SECALG} we present the OCE-VI algorithm. Section~\ref{sec:results} discusses the regret bounds...

%%%%%%%%%%%%%%%%%%%%%%%%%%%%%%%%%%%%%%%%%%%%%%%%%%%%%%%%%%%%%%%%

%%%%%%%%%%%%%%%%%%%%%%%%%%%%%%%%%%%%%%%%%%%%%%%%%%%%%%%%%%%%%%%%

\section{Problem Formulation}\label{sec:formulation}
In this section, we introduce the optimized certainty equivalent (OCE), and formulate the risk-sensitive reinforcement learning problem with recursive OCE.

\subsection{The Optimized Certainty Equivalent}\label{SEC22}
We introduce OCE, following \cite{ben2007old}. 
%In this subsection, we introduce the definition of the optimized certainty equivalent (OCE). See, e.g., \cite{ben2007old} for a comprehensive discussion. %We also list its connections with the convex risk measures.
Let $u : \mathbb{R}\rightarrow [-\infty,\infty)$ be a nondecreasing, closed, concave utility function with effective domain dom $u=\{x\in\mathbb{R}\vert u(t)>-\infty\}\neq \emptyset$. Suppose $u$ satisfies $u(0)=0$ and $1\in \partial u(0)$, where $\partial u(\cdot)$ denotes the subdifferential of $u$. 
We denote this class of normalized utility functions by $U_0$. 
The optimized certainty equivalent (OCE) %of an uncertain outcome 
%is a decision theoretic criterion based on a utility function, and it 
is defined by 
\begin{equation}\label{OCE}
OCE^u(X)=\sup_{\lambda\in \mathbb{R}}\{\lambda+E[u(X-\lambda)]\},
\end{equation}
where $X$ is a bounded random variable (so that $OCE^u(X)$ is finite). The interpretation for OCE in \eqref{OCE} is as follows: a decision maker can consume part of the future uncertain income of $X$ dollars at present, and this is denoted by $\lambda$. The present value of $X$ then becomes $\lambda+E[u(X-\lambda)]$, and the OCE represents the optimal allocation of $X$ between present and future consumption. 
% a future uncertain income of $X$ dollars, and can consume part of $X$ at present. If she chooses to consume $\lambda$ dollars, the resulting present value of $X$ is then $\lambda+E[u(X-\lambda)]$. Thus, the sure present value of $X$, i.e., its
% certainty equivalent, is the result of an optimal allocation of $X$ between present and future consumption. 

OCE captures the risk attitude of a decision maker via the utility function $u$. With different choices of the utility functions, OCE covers important examples of popular risk measures, including the entropic risk measure, CVaR and mean-variance models, as special cases. See Table~\ref{TABLE1}.
% The complete formula for the utility function corresponding to mean-variance is $u_c(t)=(t-ct^2)1\{t\leq \frac{1}{2c}\}+\frac{1}{4c}1\{t> \frac{1}{2c}\}$, and the right support of $X$ should be $x_{max}\leq EX+\frac{1}{2c}$, see \citep[Example 2.2]. In the definition of CVaR, $q(\alpha)= \min\{x\vert F_X(x)\geq \alpha\}$ where $F_X$ is the cumulative distribution function of $X$.
Due to its tractability and flexibility, OCE has been applied in many areas including finance and machine learning; see, e.g., \cite{ben2007old, lee2020learning, la2022wasserstein}.

    % the preferences of agents, i.e. the utility functions, with the risk measures, which is useful to study the behaviors of the agents. With different utility functions, we are able to simulate different agents with different attitudes towards risks and construct different learning models. The following table 1 lists some of the popular risk measures and their corresponding utility functions.

\begin{table}[t]
\caption{Popular OCEs and corresponding utility functions. For CVaR, $q(\alpha)= \min\{x\vert F_X(x)\geq \alpha\}$ where $F_X$ is the cumulative distribution function of $X$ and $[-t]_{+}=\max\{-t,0\}$.} %For mean-variance, the precise definition of the utility function is given by $u_c(t)=(t-ct^2)1\{t\leq \frac{1}{2c}\}+\frac{1}{4c}1\{t> \frac{1}{2c}\}$.}
\label{TABLE1}
\vskip 0.1in
\begin{center}
\renewcommand\arraystretch{1.25}
\begin{small}
%\begin{sc}
\begin{tabular}{lll}
\hline Name & $OCE^u(X)$  & Utility function $u$ \\
\hline Mean & $\mathbf{E}[X]$ & $u(t)=t$ \\
Entropic risk & $\frac{1}{\beta} \log \mathbf{E}\left[e^{\beta X}\right]$ & $u_\beta(t)=\frac{1}{\beta} e^{\beta t}-\frac{1}{\beta}$ \\
CVaR & $\mathbf{E}\left[X \mid X\leq q\left(\alpha\right)\right]$ & $u_\alpha(t)=-\frac{1}{\alpha}[-t]_+$ \\
Mean-Variance & $\mathbf{E}[X]-c \cdot \text{Var}(X)$ & $u_c(t)=(t-ct^2)1\{t\leq \frac{1}{2c}\}+\frac{1}{4c}1\{t> \frac{1}{2c}\}$ \\

\hline
\end{tabular}
%\end{sc}
\end{small}
\end{center}
\vskip -0.1in
\end{table}

\subsection{Episodic Risk-Sensitive MDPs with Recursive OCE}
\label{sec:epis}
Consider a finite-horizon, tabular, non-stationary Markov decision process (MDP), $\mathcal{M}(\mathcal{S},\mathcal{A},H,\mathcal{P},r)$, where $\mathcal{S}$ is the set of states with $\vert \mathcal{S}\vert=S$, $\mathcal{A}$ is the set of actions with $\vert \mathcal{A}\vert=A$, $H$ is the number of steps in each episode, $\mathcal{P}$ is the transition matrix so that $P_h(\cdot\vert s,a)$ gives the distribution over states if action $a$ is taken for state $s$ at step $h\in [H]$, where $[H]=\{1,2,\cdots,H\}$, and $r_h:\mathcal{S}\times \mathcal{A}\rightarrow [0,1]$ is the deterministic reward function at step $h$. We define $s_{H+1}$ as the terminate state, which represents the end of an episode. A policy $\pi$ is a collection of $H$ functions $\mathcal{\varPi}:=\{\pi_h:\mathcal{S} \to \mathcal{A}\}_{h\in [H]}$.

The reinforcement learning agent repeatedly interacts with the MDP $\mathcal{M}:=\mathcal{M}(\mathcal{S},\mathcal{A},H,\mathcal{P},r)$ over $K$ episodes. %, where the horizon of each episode is $H$.
For simplicity (as in many prior studies \citep{azar2017minimax, du2022risk}) we assume that the reward function $(r_h(s,a))_{s\in\mathcal{S},a\in\mathcal{A}}$ is known, but the transition probabilities $(P_h(\cdot\vert s,a))_{s\in\mathcal{S},a\in\mathcal{A}}$ are unknown. In each episode $k=1,2,\cdots,K$, an arbitrary fixed initial state $s_1^k=s_1\in\mathcal{S}$ is picked.% for the $\mathcal{M}(\mathcal{S},\mathcal{A},H,\mathcal{P},r)$.
\footnote{The results of the paper can also be extended to the case where the initial states are drawn from a fixed distribution over $\mathcal{S}$.} An algorithm \textbf{algo} initializes and implements a policy $\pi^1$ for the first episode, and executes policy $\pi^k$ throughout episode $k$ based on the observed past data (states, actions and rewards) up to the end of episode $k-1$, $k=2,\cdots, K$.  %Let 

% We use $V_h^{\pi}:\mathcal{S}\to \mathbb{R}$ to denote the value function at step $h$ under policy $\pi$ and we use $Q_h^{\pi}:\mathcal{S}\times\mathcal{A}\to \mathbb{R}$ to denote $Q$-value function at step $h$. The formal definitions of the value function and the action-value function will be given later. Define a function class $\mathcal{V}_h^{\pi}=\{V_h^{\pi}\vert V_h^{\pi}:\mathcal{S}\to [0,H]\}$.

% The agent interacts with a tabular episodic MDP in the following way. At the initial time 1, the process is in state $s_1\in\mathcal{S}$, which is generated by the environment randomly. Then, at each step $h\in [H]$, the agent observes the state $s_h\in\mathcal{S}$ and take an action $a_h\in\mathcal{A}$ according to the policy $\pi\in\mathcal{\varPi}$. After this, the agent receives a reward $r_h(s_h,a_h)$, and then jumps to the next state $s_{h+1}\in\mathcal{S}$, which follows the distribution $P_h(s_{h+1}\vert s_h,a_h)$. The process ends when $s_{H+1}$ is reached.

To capture the (dynamic) risk in the decision making process of the agent, we propose a novel RL formulation with recursive OCEs based on the studies of MDPs with recurisve measures \citep{ruszczynski2010risk,bauerle2022markov}. Specifically, we use $V_h^{\pi}:\mathcal{S}\to \mathbb{R}$ to denote the value function at step $h$ under policy $\pi$ and we use $Q_h^{\pi}:\mathcal{S}\times\mathcal{A}\to \mathbb{R}$ to denote the state-action value function at step $h$. They are recursively defined as follows: for all $ h \in [H]$, $s\in \mathcal{S}$ and $a \in \mathcal{A}$, 
% \begin{equation}
% \left\{
\begin{align}
Q_h^{\pi}(s,a) & =r_h(s,a)+OCE^{u}_{s'\sim P_h(\cdot\vert s,a)}(V^{\pi}_{h+1}(s')), \label{eq:Qpi}\\
V_h^{\pi}(s) & =Q_h^{\pi}(s,\pi_h(s)), \quad \quad V_{H+1}^{\pi}(s) =0, \label{eq:Vpi}
\end{align}
% \right
% \end{equation}
where
% In the setting of reinforcement learning, we can define the corresponding optimized certainty equivalent(OCE) of the value function as 
\begin{equation}\label{OCES3}
\begin{aligned}
&OCE^{u}_{s'\sim P_h(\cdot\vert s,a)}( g(s'))=\sup_{\lambda\in \mathbb{R}}\{\lambda+E_{s'\sim P_{h}(\cdot\vert s,a)}[u(g(s')-\lambda)]\},
\end{aligned}
\end{equation} 
with $g:\mathcal{S} \rightarrow \mathbb{R}$ being a real-valued function. 
%, and
% $$E_{s'\sim P_{h}(\cdot\vert s,a)}[u( g (s')-\lambda)]\}=\sum_{s'\in \mathcal{S}}P_h(s'\vert s,a)u(g(s')-\lambda)$$
%and $g:\mathcal{S} \rightarrow \mathbb{R}$ is a real-valued function.

Note that in \eqref{eq:Qpi}--\eqref{eq:Vpi}, the risk measure OCE is applied to the next-state value at each period. 
Due to the generality of OCEs, the recursions in \eqref{eq:Qpi}--\eqref{eq:Vpi} cover and unify several existing frameworks: (a) when $u(t) = t$, the OCE becomes the mean, and \eqref{eq:Qpi}--\eqref{eq:Vpi} become the standard Bellman equation for the policy $\pi$ in risk-neutral RL; (b) when $u$ is an exponential function given in Table~\ref{TABLE1}, OCE becomes the entropic risk, and \eqref{eq:Qpi}--\eqref{eq:Vpi} recover the Bellman equation for the policy $\pi$ in risk-sensitive RL with entropic risk (see Equation~(3) in \cite{fei2021exponential}); (c) when $u$ is a piecewise linear function and the OCE becomes the CVaR, \eqref{eq:Qpi}--\eqref{eq:Vpi} reduce to the recursion of value functions in risk-sensitive RL with iterated CVaR (see Equation (1) in \cite{du2022risk}). We also remark that when the OCE is a coherent risk measure (e.g. CVaR), it has a dual or robust representation, and the recursion \eqref{eq:Qpi}--\eqref{eq:Vpi} can be interpreted as the Bellman equation of a distributionally robust MDP, see Section 6 of \cite{bauerle2022markov} for detailed discussions. 

% robust optimization of the expected total reward,
%The formal definitions of the value function and the action-value function will be given later. Define a function class $\mathcal{V}_h^{\pi}=\{V_h^{\pi}\vert V_h^{\pi}:\mathcal{S}\to [0,H]\}$.

% In the risk-neutral setting, the value functions are defined as the expectation of the cumulative rewards. And the Bellman equations can be obtained by decomposing the value functions into the immediate reward and the conditional expectation of the successor value functions. In the risk-sensitive setting, although sometimes we can define the value functions and the Bellman equations in this way, for example, the entropic risk measure in \cite{fei2020risk} and \cite{fei2021exponential}, it's usually impracticable to define them as in the risk-neutral setting, because the risk measures are nonlinear and we can not decompose the value functions when the expectation is replaced by the risk measures. 

% We use $V_h^{\pi}:\mathcal{S}\to \mathbb{R}$ to denote the value function at step $h$ under policy $\pi$ and we use $Q_h^{\pi}:\mathcal{S}\times\mathcal{A}\to \mathbb{R}$ to denote $Q$-value function at step $h$. The formal definitions of the value function and the action-value function will be given later. Define a function class $\mathcal{V}_h^{\pi}=\{V_h^{\pi}\vert V_h^{\pi}:\mathcal{S}\to [0,H]\}$.

%In the economic literature, recursive utilities are used for the risk measures, see \cite{miao2020economic}. 
% With recursive utility function, we can recurrently define $V_h^{\pi}$ and $Q_h^{\pi}$ as in \cite{bauerle2022markov}:

Because $\mathcal{S},\mathcal{A},H$ are finite, by Theorem 4.8 in \cite{bauerle2022markov}, there exists an optimal Markov policy $\pi^*$ which gives the optimal value function $V^*_h(s)=\max_{\pi\in \mathcal{\varPi}}V_h^{\pi}(s)$ for all $ s\in \mathcal{S}$ and $h\in [H]$. The optimal Bellman equation is given by
\begin{align}
Q_h^*(s,a) & =r_h(s,a)+OCE^{u}_{s'\sim P_h(\cdot\vert s,a)}(V^*_{h+1}(s')), \\
V_h^*(s) & =\max_{a\in \mathcal{A}}Q_h^*(s,a), \quad \quad V_{H+1}^*(s)=0.
\end{align}
The expected (total) regret for algorithm \textbf{algo} over $K$ episodes of interaction with the MDP $\mathcal{M}$ is then defined as 
\begin{equation}\label{eq:regret}
Regret(\mathcal{M},\textbf{algo},K)=E\left [\sum_{k=1}^K(V_1^*(s_1^k)-V_1^{\pi^k}(s_1^k))\right ],
\end{equation}
where the term $V_1^*(s_1^k)-V_1^{\pi^k}(s_1^k)$ measures the performance loss when the agent executes (suboptimal) policy $\pi^k$ in episode $k$. 
Our goal is to propose an efficient learning algorithm with a provable worst-case regret upper bound that scales sublinearly in $K$, as well as to establish a minimax lower bound. 

\section{The OCE-VI Algorithm}\label{SECALG}
In this section, we propose a model-based algorithm, denoted by OCE-VI, for risk-sensitive RL with recursive OCE. 
%the OCE-VI Algorithm for learning episodic tabular risk-sensitive MDPs. 

Before presenting the algorithm, we first introduce some notations. A state-action-state triplet $(s,a,s')$ means that the process is in state $s$, takes an action $a$ and then moves to state $s'$. Similarly, a state-action pair $(s,a)$ means that the process is in state s and takes an action $a$. At the beginning of the $k$-th episode, we set the observed cumulated visit counts to $(s,a,s')$ at step $h$ up to the end of episode $k-1$ as $N_h^k(s,a,s')$ for $s,s'\in \mathcal{S}$ and $a\in \mathcal{A}$, and the cumulated visit counts to $(s,a)$ at step $h$ up to the end of episode $k-1$ as $N_h^k(s,a)$ for $s\in \mathcal{S}$ and $a\in \mathcal{A}$. When $2\leq k\leq K$, for $s\in \mathcal{S},a\in \mathcal{A},s'\in\mathcal{S}$, the formulas for $N_h^k(s,a,s')$ and $N_h^k(s,a)$ are given by
\begin{equation*}
\begin{aligned}
&N_h^k(s,a,s')=\sum_{i=1}^{k-1} \mathrm{1}\{(s_h^i,a_h^i,s_{h+1}^i)=(s,a,s')\}, \\ &N_h^k(s,a)=\sum_{i=1}^{k-1} \mathrm{1}\{(s_h^i,a_h^i)=(s,a)\}.
\end{aligned}
\end{equation*}
When $k=1$, we set $N_h^k(s,a,s')=N_h^k(s,a)=0$ for $s\in \mathcal{S},a\in \mathcal{A},s'\in\mathcal{S}$.
Then the empirical transition probabilities are given by
\begin{equation*}
\hat{P}^k_h(s'\vert s,a)=\frac{N_h^k(s,a,s')}{\max\{1,N^k_h(s,a)\}}.
\end{equation*}
In particular, if $(s,a)$ has not been sampled before episode $k$, $\hat{P}^k_h(s'\vert s,a) =0$ for all $s'.$
% Note that when k=1, the process hasn't visited any state-action pair or state-action-state pair. So we set $\hat{P}^k_h(s'\vert s,a)=0$ for all $s,s'\in \mathcal{S}$ and $a\in \mathcal{A}$, and the agent choose actions arbitrarily in the first episode in our algorithm. 
% Then, we can define the empirical OCE as 
% \begin{align}\label{OCEV}
% OCE^{u}_{s'\sim \hat{P}^k_h(\cdot\vert s,a)}(V^{\pi}_{h+1}(s'))&=\sup_{\lambda\in \mathbb{R}}\{\lambda+E_{s'\sim \hat{P}^k_{h}(\cdot\vert s,a)}[u(V^{\pi}_{h+1}(s')-\lambda)]\}.
% \end{align}

Similar to UCBVI in \cite{azar2017minimax} for risk-neutral RL, the OCE-VI algorithm achieves exploration by awarding some bonus for exploring some state-action pairs during the learning process. We consider the bonus
\begin{equation}\label{S11}
b_h^k(s,a)=\vert u(-H+h)\vert\sqrt{\frac{2\log\left (\frac{SAHK}{\delta}\right )}{\max\{1,N^k_h(s,a)\}}},
\end{equation}
where $(s,a)\in \mathcal{S}\times\mathcal{A}$, and $\delta\in (0,1)$ is an input parameter in our algorithm. Importantly, the bonus here depends on the utility function $u$ in the OCE \eqref{OCE}, which is natural given that we study risk-sensitive RL with recursive OCE. 
The details of the OCE-VI algorithm are summarized in Algorithm~\ref{ALG}. 

\begin{algorithm*}[tb]
   \caption{The OCE-VI Algorithm}
   \label{ALG}
\begin{algorithmic}
   \STATE {\bfseries Input:} Parameters $\delta,\mathcal{S},\mathcal{A},H,K,r$ and an utility function $u \in U_0$
   \STATE Initialize $\hat{V}_{h}(s)\gets 0$, $N_h(s,a,s')\gets 0$ and $N_h(s,a)\gets 0$ for all $(s,a,h)\in\mathcal{S}\times\mathcal{A}\times [H+1]$.
   \FOR{episode $k=1,\cdots ,K$}
   \FOR{step $h=H,H-1,\cdots,1$}
   \FOR{$(s,a)\in \mathcal{S}\times \mathcal{A}$}
   \IF{$N_h(s,a)\geq 1$}
   \STATE Update $b_h(s,a)$ by \eqref{S11} according to the utility function $u$
   \STATE$\hat{Q}_h(s,a)\gets \min{\{r_h(s,a)+OCE^{u}_{s'\sim \hat{P}_h(\cdot\vert s,a)}(\hat{V}_{h+1}(s'))+b_h(s,a),H-h+1\}}$
   \STATE $\hat{V}_h(s)\gets \max_{a'\in \mathcal{A}}\hat{Q}_h(s,a')$
   \ELSE
   \STATE $\hat{Q}_h(s,a)\gets H-h+1$
   \ENDIF
   \ENDFOR
   \ENDFOR
   \STATE For all $h\in [H]$, take $\pi^k_h(s_h)\gets argmax_{a'\in \mathcal{A}}\hat{Q}_h(s_h,a')$
   \STATE Play the episode $k$ with policy $\pi^k$, update $N_h(s,a),N_h(s,a,s')$ and $\hat{P}_h(s'\vert s,a)$ for all $h\in [H]$
   \ENDFOR
\end{algorithmic}
\end{algorithm*}

\begin{remark}
The dependence of the bonus on the utility function $u$ sheds some light on how the degree of risk aversion affects the degree of exploration. 
%The comparative risk aversion in OCE is not fully discussed in the literature, as far as we know. We are only aware of a result of risk aversion in 
\citet{ben2007old} show that an agent with OCE preferences is weakly risk averse (i.e., any random payoff is less preferred by the agent to its mean) if and only if the utility function is dominated by the identify function (i.e., $u(x)\le x,x\in \mathbb{R}$). Now, consider two agents with recursive OCE preferences represented by utility functions $u_1$ and $u_2$, respectively. If $u_1$ is dominated by $u_2$ (i.e., $u_1(x)\le u_2(x),x\in\mathbb{R}$), then $|u_1(-H+h)|\ge |u_2(-H+h)|$ because $-H+h\le 0$ and $u_i(x)\le 0,x\le 0,i=1,2$. Consequently, the exploration bonus for agent 1 is larger than for agent 2. Therefore, if we interpret the dominance of $u_2$ over $u_1$ as a higher degree of risk aversion of agent 1 than that of agent 2, as suggested by the characterization of weak risk aversion \citep{ben2007old}, then in our algorithm for a more risk averse agent we need to have a larger bonus to encourage her to explore. 
We also remark that 
our bonus \eqref{S11} is based on Chernoff-Hoeffding’s concentration inequalities and it scales linearly with $|u(-H+h)|$.  It might be possible to design tighter bonuses that may depend on the utility function in a nonlinear manner. This is an open problem for risk-sensitive RL with recursive OCE and we leave it for future work.
%We, however, do not know whether the above conjecture is true. Thus, to fully answer your question, we need to first study comparative risk aversion in OCE, which can be a future work.
\end{remark}

\begin{remark}
The OCE-VI algorithm is computationally tractable. In each episode, the computational cost of the algorithm is similar to solving a known MDP with value iteration, except that one needs to to compute the quantity $OCE^{u}_{s'\sim \hat{P}_h(\cdot\vert s,a)}(\hat{V}_{h+1}(s'))$ when updating the Q function. For certain special utility functions such as those in Table~\ref{TABLE1}, this quantity can be explicitly computed because the state space is finite. In general, computing this OCE is equivalent to solving the optimization problem $\sup_{\lambda\in \mathbb{R}}\{\lambda+E_{s'\sim\hat{P}_h(\cdot\vert s,a)}[u(\hat{V}_{h+1}(s')-\lambda)]\}$. This is a one-dimensional concave optimization problem because the utility function $u$ is concave and $\hat{P}_h(\cdot \vert s,a)$ is a probability distribution. Because the state space is finite, we can exchange the expectation and the derivative/subgradient with respective to $\lambda$ in the first order optimality condition of the above optimization problem. 
Thus, when the utility function is differentiable, this concave optimization problem can be solved efficiently using the gradient descent or Newton's method. When the utility function is nondifferentiable, it can be solved with efficient proximal gradient methods; see, e.g., \cite{parikh2014proximal}.

\end{remark}

\section{Main Results}\label{sec:results}
In this section, we present our main results. Our first main result is an upper bound
on the expected regret of the proposed OCE-VI algorithm. 
%with respect to the regret bound for the episodic risk-sensitive MDP. 

%Both Theorem \ref{THM1} and Theorem \ref{THM3} are derived under the risk-averse setting, but they can be extended to the risk-seeking setting by setting the utility function $u$ as a convex function.

\begin{theorem}\label{THM1}
The expected regret of the OCE-VI algorithm satisfies
\begin{align*}
&Regret(\mathcal{M},\textbf{OCE-VI},K)\leq \tilde{\mathcal{O}}\left(\sum_{h=1}^H\vert u(-H+h)\vert S\sqrt{\prod\limits_{i=1}^{h-1} u_{-}'(-H+i)AK}\right),
\end{align*}
where 
$\tilde{\mathcal{O}}(\cdot)$ ignores the logarithmic factors in $S,A,H$ and $K$ and $u'_{-}(\cdot)$ is the left derivative of $u$.
\end{theorem} 
 The regret upper bound depends on the utility function $u$ in the OCE \eqref{OCE} via the term $\vert u(-H+h)\vert$, which comes from the bonus \eqref{S11}, and the term $\prod\limits_{i=1}^{h-1} u_{-}'(-H+i)$, which comes from bounding the Radon-Nikodym derivative arising from the linearization of the OCE as a concave functional in our regret analysis (see Equation~\eqref{eq:radon}). 
 % Note that $u$ is a nondecreasing and concave function with $u(0)=0$. So $\vert u(-H+h)\vert$ is the maximum value of $\vert u(t)\vert$ on $[-H+h,H-h]$, where $H-h$ is actually an upper bound for the value function at time $h+1.$ 
 We provide a sketch of the proof of Theorem~\ref{THM1} in Section~\ref{sec:sketch}, and give the full details in Appendix~\ref{app:B}.

We next present our second main result, which provides a minimax regret lower bound for RL with recursive OCE. We first state the following assumption.

\begin{assumption}\label{ASSU2}
The number of states and actions satisfy $S\geq 6,A\geq 2$, and there exists an integer $d$ such that $S=3+\frac{A^d-1}{A-1}$. In addition, the horizon $H$ satisfies $H\geq c_2 d$, where $c_2>2$ is a constant. %\gao{WH: need assumptions like $H > 3 d $ ? What happens if such assumptions are not satisfied? }
\end{assumption}

Assumption \ref{ASSU2} is adapted from Assumption 1 in \cite{domingues2021episodic}, who provide a minimax lower bound in the risk-neutral episodic RL setting. 
% We follow the steps in \cite{domingues2021episodic} to construct a hard MDP and calculate the regret lower bound for any Algorithm. We use $S-3$ states to construct a full $A$-ary tree of depth $d-1$. In this rooted tree, each parent node has exactly $A$ children and the total number of nodes is $S=3+\frac{A^d-1}{A-1}=3+\sum_{i=0}^{d-1}A^i$. We can use the tree to construct the hard MDP instances to derive the regret lower bound, which complements the result of Theorem \ref{THM1}. 
%Note that the assumption $S=3+\frac{A^d-1}{A-1}$ 
This assumption is imposed to simplify the analysis, more precisely the construction of hard MDP instances, and it can be relaxed following the discussion in Appendix D of \cite{domingues2021episodic}.  %On the other hand, \cite{domingues2021episodic} showed that the lower bound can become exponential in $H$ without an assumption such as $H \ge 3d.$

\begin{theorem}\label{THM3}
Under Assumption \ref{ASSU2}, for any algorithm \textbf{algo}, there exists an MDP $\mathcal{M}$ whose transition probabilities depend on $h$ such that
\begin{align*}
&Regret(\mathcal{M},\textbf{algo},K)\geq \frac{1}{18\sqrt{c_1 c_2}}\cdot \left[u\left(\left(1-\frac{2}{c_2}\right) H-\lambda^*\right)-u(-\lambda^*)\right]\sqrt{SAHK}
\end{align*}
for all $K\geq \frac{c_1 HSA}{2c_2}$, where the constants $c_1\geq 4,c_2>2$ and $\lambda^*$ satisfies 
\begin{align*}
1\in \left(1-\frac{2}{c_1}\right)\partial u\left(\left(1-\frac{2}{c_2}\right)H-\lambda^*\right)+\frac{2}{c_1}\partial u(-\lambda^*).
\end{align*}
\end{theorem}

% \begin{comment}
% \gao{ WH todo:
% \begin{itemize}
%     \item for entropic risk, does our lower bound improves existing results? the gap between lower and upper bounds?
%     \item for mean-variance, it generates new lower bound.
%     \item for general OCE, the gap between lower and upper bounds? It is tight in A and K ? 
% \end{itemize}
% }
% \end{comment}

% Theorem \ref{THM3} provides a regret lower bound for any Algorithm $\textbf{algo}$ under MDPs with OCE corresponding to risk-averse risk measures. The result of Theorem \ref{THM3} can also be extended to the risk-seeking setting by using the convex utility function and the corresponding optimization problem for OCE. 

Note that when $u(t)=t,$ OCE becomes expectation, and our regret lower bound in Theorem~\ref{THM3} is $\Omega(H\sqrt{SAHK})$, by choosing for instance $c_1=4$ and $c_2=3$. This recovers the (tight) regret lower bound in \cite{domingues2021episodic} in learning risk-neutral tabular MDP. For a general utility function $u$ in OCE,  
%the quantity $\lambda^*$ in Theorem \ref{THM3} may not be solved analytically, and 
the choices of  constants $c_1\geq 4, c_2>2$ should be based on the specific utility function to generate tighter lower bounds. For illustrations, we provide some examples in Section~\ref{SECEX}.

The proof of Theorem~\ref{THM3} is based on extending the proof of Theorem 9 in \cite{domingues2021episodic} to our risk-sensitive setting.
There are essential difficulties in this extension. These include how to construct hard MDP instances that adapt to the OCE, and how to bound the value functions defined recursively via OCE that involves an optimization problem. Due to space limitations, we provide the proof details in Appendix~\ref{sec:THM-lowerB}. % in the Supplementary Material.

\begin{remark}
For the simplicity of presentation, we focus on OCE in \eqref{OCE}, which exhibits the risk aversion property with $OCE^u(X) \le E[X]$, due to the concavity of the utility function; see Proposition 2.2 in \cite{ben2007old}. Our main results in the paper hold in the risk-seeking setting as well, where $OCE^u(X)$ is defined by $\inf_{\lambda\in \mathbb{R}}\{\lambda+E[u(X-\lambda)]\}$ with a convex utility function $u$. In this case, we need to use a bonus $b_h^k(s,a)=\vert u(-H+h)\vert \sqrt{\frac{2S\log\left(\frac{SAHK}{\delta}\right)}{\max\{1,N_h^k(s,a)\}}}$ in the OCE-VI algorithm. Compared with \eqref{S11}, this bonus has an extra term $\sqrt{S}$, which arises from a technical step in the proof for the risk-seeking case (see inequality (2) of Lemma \ref{L4}). The regret bounds still hold in this setting. 
%$\tilde{\mathcal{O}}\left(\sum_{h=1}^H\vert u(-H+h)\vert S\sqrt{AK}\right)$ remains unchanged though. 
% Risk averse vs Risk seeking  \gao{to add}
% The result of Theorem \ref{THM1} mainly works for the risk-averse setting. In order to study the risk-seeking OCE, we should consider the convex utility function and the optimization problem $\inf_{\lambda\in \mathbb{R}}\{\lambda+E_{s'\sim P_{h}(\cdot\vert s,a)}[u(g(s')-\lambda)]\}$. Although we can obtain the same regret upper bound as that in Theorem \ref{THM1}, the bonus becomes $b_h^k(s,a)=2\vert u(-H+h)\vert \sqrt{\frac{S\log\left(\frac{SAHK}{\delta}\right)}{\max\{1,N_h^k(s,a)\}}}$, because inequality (2) of Lemma \ref{L4} should follow from optimal solutions to $\min_{\lambda\in [0,H-h]}\{\lambda+E_{s'\sim \hat{P}^k_h(\cdot\vert s,a)}[u(V_{h+1}^*(s')-\lambda)]\}$, which are not deterministic quantities. 
\end{remark}

%Note that if $c_1$ and $c_2$ are large enough, which is discussed in detail in the proof of Theorem \ref{THM3}, the term $\left\vert u\left (-\left (1-\frac{1}{c_2}\right )H\right )\right\vert$ will be close to $\vert u(-H)\vert$. Thus, $\vert u(-H)\vert$ is an indispensable part of the regret upper bound in Theorem \ref{THM1}.

%According to the results in Theorem \ref{THM1} and Theorem \ref{THM3}, the OCE-VI Algorithm achieves a regret rate with the optimal dependence on the number of episodes $K$ and the number of actions $A$, up to logarithmic factors. Although $\vert u(-H)\vert$ is an essential factor in the regret upper bound by Theorem \ref{THM3}, it remains an open question whether the regret bound can be improved in terms of $H$. 

%%%%%%%%%%%%%%%%%%%%%%%%%%%%%%%%%%%%%%%%%%%%%%%%%%%%%%%%%%%%%%%
\subsection{Examples and Comparisons to Related Work}\label{SECEX}
We consider several specific utility functions and the resulting OCEs to illustrate our regret bounds in Theorems \ref{THM1} and \ref{THM3}. 
% For entropic risk and CVaR, we compare our results with the state-of-art regret bounds for risk-sensitive MDPs in the existing literature. %For entropic risk and CVaR, we show that our results can either match or improve state-of-art regret bounds for risk-sensitive MDPs in the existing literature. 
% We also present regret bounds for MDPs with other recursive risk measures (e.g. mean-variance) which appear to be new. % in the literature. 

\subsubsection{Mean-variance Model}
When the utility function is $u_c(t)=(t-ct^2)1\{t\leq \frac{1}{2c}\}+\frac{1}{4c}1\{t> \frac{1}{2c}\}$, the corresponding OCE is the celebrated mean-variance model \cite{markowitz1952}, where $c>0$ is a given risk parameter representing the degree of risk aversion. To the best of our knowledge, the following results are the first regret bounds for risk-sensitive MDPs with the recursive mean-variance model.

%$\tilde{\mathcal{O}}\left(\frac{(1+2c(H-1))^{\frac{H}{2}}-1}{\sqrt{1+2c(H-1)}-1}(H+cH^2)S\sqrt{AK}\right)$

\begin{itemize}
 \item Upper bound. Our regret upper bound in Theorem~\ref{THM1} is $\tilde{\mathcal{O}}\left((1+2cH)^{\frac{H-1}{2}}(H^2+cH^3)S\sqrt{AK}\right)$.

% According to Example 2.2 in \cite{ben2007old}, we need to assume that $X_{max}\leq \frac{1}{2c}+EX$, where $X_{max}$ is the essential supremum of the random variable $X$, to obtain the mean-variance OCE. Thus, we need to assume that $H\leq \frac{1}{2c}$, and then the corresponding worst-case regret bound is $\tilde{\mathcal{O}}\left(\left(\sqrt{2}\right)^{H-1}H^2S\sqrt{AK}\right)$.  %To the best of our knowledge, this is the first regret upper bound for the RL algorithm with mean-variance objective function.
\item Lower bound. We can choose $c_1=8,c_2=4$, and then $\lambda^*=\left(1-\frac{2}{c_1}\right)\left(1-\frac{2}{c_2}\right)H = 3H/8$. The regret lower bound in Theorem \ref{THM3} becomes $\Omega\left((H+\frac{1}{4}cH^2)\sqrt{SAHK}\right)$. 
% To the best of our knowledge, this is the first regret lower bound for the RL algorithm with mean-variance objective function.
\end{itemize}
%%%%%%%%%%%%%%%%%%%%%

\subsubsection{(Iterated) CVaR}
When the utility function is $u_{\alpha}(t)=-\frac{1}{\alpha}[-t]_{+},\alpha>0$, the corresponding OCE is CVaR, where $\alpha>0$ is the risk level of CVaR. Our RL formulation in Section~\ref{sec:epis} reduces to the one in \cite{du2022risk}, and our OCE-VI algorithm becomes their ICVaR algorithm with a smaller exploration bonus. 

\begin{itemize}
\item Upper bound. Our regret upper bound in Theorem \ref{THM1} becomes $\tilde{\mathcal{O}}\left(\frac{(\frac{1}{\sqrt{\alpha}})^H-1-H(\frac{1}{\sqrt{\alpha}}-1)}{(1-\sqrt{\alpha})^2}S\sqrt{AK}\right)$. When $0< \alpha\leq \frac{3-\sqrt{5}}{2}$, this upper bound can be further bounded by $\tilde{\mathcal{O}}\left(\left(\frac{1}{\sqrt{\alpha^{H+1}}}-\frac{H}{\sqrt{\alpha}}\right)S\sqrt{AK}\right)$. When $\frac{3-\sqrt{5}}{2}<\alpha< 1$, the regret bound can be further bounded by $\tilde{\mathcal{O}}\left(\frac{H^2 S\sqrt{AK}}{\sqrt{\alpha^{H+1}}}\right)$.
\cite{du2022risk} design the ICVaR algorithm and can obtain a worst-case regret upper bound of $\tilde{\mathcal{O}}\left(\frac{H^2 S\sqrt{AK}}{\sqrt{\alpha^{H+1}}}\right)$.\footnote{\cite{du2022risk} consider stationary MDPs, and we modify their regret bounds to adapt to our non-stationary setting. 
% The authors also provide a case-dependent regret bound that depends on $1/\sqrt{\min_{\pi,h,s:w_{\pi,h}(s)>0}w_{\pi,h}(s)}$, and we can derive such a regret bound with similar techniques.
} 
Our result improves the result of \citep{du2022risk} by a factor of $H^2$ when $0< \alpha\leq \frac{3-\sqrt{5}}{2}$. This is due to a smaller exploration bonus used in our algorithm compared with theirs.  % which is a more relevant range for $\alpha$. 
% \cite{du2022risk} bound the CVaR gap of value function shift by using distortion distribution of CVaR, see Lemma 8 in their paper, however, they consider the total regret without using the martingale argument, which can sometimes encounter extreme situations, so their regret upper bound has a coefficient with respect to worst probability.
\item Lower bound. We can choose $c_1=\frac{2}{\alpha}$ and $c_2=4$ in Theorem~\ref{THM3}, and let $\lambda^* = \left(1-\frac{3}{c_2}\right)H$. Then,
our regret lower bound becomes $\Omega\left(H\sqrt{\frac{SAHK}{\alpha}}\right)$ and it is problem-independent. This is in contrast with \cite{du2022risk}, who derive a regret lower bound that depends on some problem-dependent quantity, specifically, the minimum probability of visiting an available state under any feasible policy. %of $\Omega\left (\frac{H\sqrt{AK}}{\sqrt{\alpha\min_{\pi,h,s:w_{\pi,h}(s)>0}w_{\pi,h}(s)}}\right )$, which depends on the problem-dependent quantity $\min_{\pi,h,s:w_{\pi,h}(s)>0}w_{\pi,h}(s)$. 
% \gao{In Du's paper, they misuse the inequality in Auer et al.(2002) to obtain a tight regret lower bound, see page 35, should we point out their mistake? }
% Note that $\min_{\pi,h,s:w_{\pi,h}(s)>0}w_{\pi,h}(s)$ depends on the transition structure of the MDP and it can be even smaller than $1/K^2$ in theory. 

% and their regret lower bound becomes $\Omega\left(KH\sqrt{\frac{A}{\alpha}}\right)$, where $KH\sqrt{\frac{A}{\alpha}}> KH$, but we have $0\leq Regret(\mathcal{M},\textbf{algo},K)\leq KH$ for any $\textbf{algo}$ in Lemma \ref{L1}. So there might be some problem in their result.

% The corresponding $\lambda^*$ can be any element in $\left[0,\left(1-\frac{2}{c_2}\right)H\right]$. Note that the solution to (2.12) in Example 2.3 of \cite{ben2007old} is $F^{-1}(\alpha)$, because the random variable $X$ is a continuous random variable in their paper. However, in our setting, the value function is a discrete random variable with possible values $0$ and $\left(1-\frac{2}{c_2}\right)H$, so we can choose $c_1=\frac{2}{\alpha}$ and obtain the solution $\lambda^*\in\left[0,\left(1-\frac{2}{c_2}\right)H\right]$. Choose $c_2=4$, 
\end{itemize}
%%%%%%%%%%%%%%%%%%%

\subsubsection{Entropic Risk}
When the utility function is $u_{\beta}(t)=\frac{1}{\beta}e^{\beta t}-\frac{1}{\beta},\beta<0$, the corresponding OCE is entropic risk, where $\beta<0$ is a given risk parameter representing the degree of risk aversion. In this case, our RL formulation in Section~\ref{sec:epis} is equivalent to the one in \cite{fei2020risk}. Note, however, that our OCE-VI algorithm is model-based and is different from the model-free algorithms proposed in \cite{fei2020risk}.

\begin{itemize}
\item Upper bound. Our regret upper bound in Theorem \ref{THM1} for the OCE-VI algorithm becomes 

$\tilde{\mathcal{O}}\left (\exp(-\frac{\beta H^2}{4})\frac{\exp(-\beta H)-1}{-\beta}S\sqrt{AK}\right )$. 
% When $\beta\to 0$, the limit of $\frac{\frac{1}{\exp(-\beta)-1}\left(\exp(-\beta H)-1\right)-H}{-\beta}$ is $\frac{H^2-H}{2}$, which recovers the coefficient of regret upper bound in the risk-neutral setting. 
This bound has a factor that is exponential in $|\beta|H^2$, which is similar as the bounds in \cite{fei2020risk}. Recently, \cite{fei2021exponential} propose the RSVI2 and RSQ2 algorithms, and they manage to remove this factor.
% which achieves a regret of $\tilde{\mathcal{O}}\left (\frac{\exp(-\beta H)-1}{-\beta}HS\sqrt{AK}\right )$.
% \cite{fei2020risk} propose the RSVI algorithm, which achieves a regret of $\tilde{\mathcal{O}}\left (\exp(-\beta H^2)\frac{\exp(-\beta H)-1}{-\beta}HS\sqrt{AK}\right )$. 
% This is made possible by using the nice properties of the exponential utility function, in particular, \cite{fei2021exponential} design
Their algorithms are based on the nice properties of the exponential utility, in particular, the so-called exponential Bellman equation which takes the exponential on both sides of the Bellman equation in \cite{fei2020risk}. However, such techniques can not be applied to our general setting, because general utility functions do not possess the same nice properties as the exponential function. Even though our upper bound is worse than the one in \cite{fei2021exponential}, we show numerically that our algorithm can outperform their algorithms on randomly generated MDP instances. See Appendix~\ref{sec:experiment} for experimental details.

\item Lower Bound. We can choose $c_1=\frac{2}{e^{-\beta}-1}\cdot \exp\left(-\beta \left(1-\frac{2}{c_2}\right)H\right)$ and $c_2=6$ in Theorem~\ref{THM3}, and  
the corresponding $\lambda^*=\frac{1}{\beta}\log\left(\left(1-\frac{2}{c_1}\right)\exp\left(\beta\left(1-\frac{2}{c_2}\right)H\right)+\frac{2}{c_1}\right)$. Then our regret lower bound becomes $\Omega\left(\frac{\exp\left(-\frac{1}{3}\beta H\right)-1}{-\beta}\sqrt{SAHK}\right)$. By contrast, \cite{fei2020risk} derive a regret lower bound of $\Omega\left(\frac{\exp(\frac{1}{2}\vert\beta\vert H)-1}{\vert\beta\vert}\sqrt{K}\right)$, which does not depend on $S$ or $A$ (due to the simple structure of the hard instances they construct). \cite{liang2022bridging} derive a regret lower bound $\Omega\left(\frac{\exp(\frac{1}{6} |\beta| H)-1}{|\beta|}\sqrt{SAHK}\right)$ in the risk-seeking setting when $\beta>0$, but they mention that it is unclear whether a similar bound holds in the risk-averse setting when $\beta<0$; see page 30 of their paper.

% due to the simple structure of their bandit model.\footnote{\cite{liang2022bridging} correct the mistake in \cite{fei2020risk}, see Proposition 26 in \cite{liang2022bridging}.} \cite{liang2022bridging} improve the result of \cite{fei2020risk} by adapting the hard MDP instances in \cite{domingues2021episodic}, and their regret lower bound is $\Omega\left(\frac{\exp(\frac{1}{6}\beta H)-1}{\beta}\sqrt{SAHK}\right)$. However, they claim that their lower bound only works for the risk-seeking setting in page 30 of their paper. The main reason is that the upper bound of the transition probability $p$ in their Lemma 27 is so small that they cannot lower bound the expected regret under risk-averse setting. We overcome the technical difficulty in Lemma \ref{L16} and obtain the regret lower bound working for the risk-averse setting. 

%Lower bound in the literature, Compare with ours

\end{itemize}

%\subsubsection{OCEs with Polynomial Utility Functions}
%We also consider the OCE risk with a polynomial utility function $u_{\gamma}(t)=\frac{-([1-t]_{+})^{\gamma}+1}{\gamma},\gamma>1$, where the OCE cannot be computed explicitly. This example is taken from \citep[Section 2.1.4]{drapeau2012fourier}. The following results are also new regret bounds for risk-sensitive MDPs with recursive OCEs with polynomial utility functions.

%When $\gamma=2$, the polynomial utility function corresponds to Mean Variance with $c=1/2$.
%\gao{Gao: if space allowed, can add one more example of OCE; do later}
%\begin{itemize}
%\item Upper bound. Our regret upper bound in Theorem \ref{THM1} is $\tilde{\mathcal{O}}\left(\frac{(1+H)^{\gamma}-1}{\gamma}\cdot H^{\frac{(H-1)(\gamma-1)}{2}+1}S\sqrt{AK}\right)$.
%\item Lower bound. We can choose $c_1=2(\frac{3}{4}H)^{\gamma-1}, c_2=8$, and then $\lambda^*=\frac{3}{4}H-1$. The regret lower bound in Theorem \ref{THM3} becomes $\Omega\left(\frac{(\frac{3}{4}H)^{\gamma}}{\gamma}\cdot\sqrt{SAHK}\right)$.
%\end{itemize}
\subsection{Discussions on tightness of our regret bounds}

% \item Choose $c_1=4,c_2=3$, our regret lower bound is $\Omega(H\sqrt{SAHK})$, which recovers the regret lower bound in \cite{domingues2021episodic}. 
% \end{itemize}

%\begin{remark}[Tightness of the analysis]\label{REM3}
%Gap is $\sqrt{HS}$? Why we can not improve? What are the difficulties? Do we use holder inequality to bound the error in transition probabilities so that we have an exra $\sqrt{S}$?  Discuss with risk neutral setting. 

Theorems \ref{THM1} and \ref{THM3} imply that the OCE-VI algorithm achieves a regret rate with the optimal dependence on the number of episodes $K$ and the number of actions $A$, up to logarithmic factors. While the bounds on
$K$ are the most important as they imply the convergence rates of learning algorithms, it remains an important open question whether one can improve the dependence of these bounds on $H$ and $S$ to narrow down the gap between the upper and lower bounds in the risk-sensitive RL setting. We elaborate further on this issue below.

From Theorems~\ref{THM1} and \ref{THM3},  we can see that the gap between our upper and lower bounds in terms of $S$ is $\sqrt{S}$, where $S$ is the number of states. The extra $\sqrt{S}$ in our regret upper bound arises from a step in our proof where we apply an $L^1$ concentration bound for the $S$-dimensional empirical transition probability vector, see Equation~\eqref{eq:oce-concen2} in Section~\ref{sec:sketch}. This extra $\sqrt{S}$ factor can be removed in RL for \textit{risk-neutral} MDPs by directly maintaining confidence intervals on the optimal value function; see, e.g., \cite{azar2017minimax, zanette2019tighter}. However, it is not clear how to adapt this technique to our risk-sensitive setting, i.e., remove $\sqrt{S}$ in \eqref{eq:oce-concen2}. This is primarily because the estimated value functions $\hat V^k_h$ in our algorithm are not only random, but they also involve OCE which is nonlinear and defined by an optimization problem (so the optimizer is also random). 

%For instance, it is unclear how to 
% bound the term $\hat V^k_h - V_h^{\pi^k}$ (which bounds per period regret) recursively with two different OCEs (for empirical and true transition distributions, respectively) involved in $\hat V_h$ and $V_h^{\pi^k}$.

% In order to fix the gap in terms of $S$, we tried to adapt Section 7.4 in \cite{agarwal2019reinforcement}, which apply concentration inequality to the optimal value function directly, see Lemma 7.10 in their book. However, due to the measure change technique in Lemma \ref{L6} is not applicable when the probability measures of the two OCEs are different, we are unable to arrange the formulas like what \cite{agarwal2019reinforcement} have done in Lemma 7.10. 

There is an exponential gap in terms of $H$ between our upper and lower bounds. This gap is due to the linearization of OCE in the recursive procedure of the regret analysis of our algorithm. Indeed, if $u(t)=t$, the corresponding regret upper bound is $\tilde{\mathcal{O}}(H^2S\sqrt{AK})$, which does not have the exponential term of $H$. In the risk-neutral setting, one can improve the dependence of the upper bound on $H$ by considering Bernstein-style exploration bonus which is built from the empirical variance of the estimated value function $\hat V^k_h$ at the next state, see, e.g., \cite{azar2017minimax}. However, it is still an open problem how to use Bernstein bonus to improve the regret bound in risk-sensitive RL \citep{fei2021exponential, du2022risk}. In our RL setting with recursive OCEs, it is possible to design a Bernstein-type bonus, but it may not lead to improved regret bounds, at least within our current analysis framework. We provide some informal discussions below including the challenges in improving bounds.
%The difficulty arises due to the non-linearity of the OCE. 
%Technically, the inequality (26) in \cite{azar2017minimax}, which is obtained from a recursive application of law of total variance and is a key ingredient to obtain tight dependency on $H$  in the regret bound for risk-neutral RL, doesn't hold under our risk-sensitive setting.

First, to ensure optimism with Bernstein-type (or variance-related) bonuses, we need analogous results to Lemmas~\ref{L3} and \ref{L4} in the appendix, which provide concentration bounds for OCEs of next-state value functions. Using Bernstein inequality instead of Hoeffding inequality, the confidence bound in Lemma \ref{L3} becomes 
\begin{align*}
&\sqrt{\frac{2\operatorname{Var}_{s'\sim P_h(\cdot\vert s,a)}\left(u(V_{h+1}^*(s')-\lambda^*_{h+1})\right)\log\left(\frac{SAHK}{\delta}\right)}{N_h^k(s,a)}} \\
&+\text {lower order term}.
\end{align*}
This bound allows us to design a Bernstein-type bonus $b_h^k(s,a)$ in the form of 
\begin{align*}
&2 \underbrace{ \sqrt{\frac{\operatorname{Var}_{s' \sim \hat{P}_h^{k}\left(\cdot \mid s,a\right)}\left(u(\hat{V}_{h+1}^k(s')-\hat{\lambda}^k_{h+1})\right)\log\left(\frac{SAHK}{\delta}\right)}{N_h^k(s,a)}}}_{\text {main term }} \\
&+\text {lower order term}.
\end{align*}
Compared with the Bernstein bonus in the risk-neutral RL setting (see e.g. \citet{azar2017minimax}), $\hat{V}_{h+1}^k(s')$ inside the variance operator is replaced by $u(\hat{V}_{h+1}^k(s')-\hat{\lambda}^k_{h+1})$ in our risk-sensitive RL setting. We use this approach because the OCE involves an optimization problem and we need to `linearize' it (i.e., remove the sup in the definition of OCE) and work with the utility $u$ applied to the value function first in order to derive concentration bounds for OCEs.
With this new bonus, we might be able to get the same regret bound as the one presented in the current paper.

However, it is difficult to get improved bounds as we explain below. In the risk-neutral setting, \citet{azar2017minimax} 
%use Bernstein bonus to improve the regret by $\sqrt{H}$. As they have mentioned in page 4, the key idea is to replace the loose bound $H$ with $\operatorname{Var}_{x_{h+1}\sim P_h(\cdot\vert x_h,a_h)}\left(V^*_{h+1}(x_{h+1})\right)$, which is denoted as $\mathbb{V}_h^*(x_h,a_h)$. They 
use an iterative-Bellman-type-Law of Total Variance so that the sum of the variances of $V_{h+1}^*$ over $H$ steps is bounded by the variance of the sum of $H$-step rewards; see Equation~(26) in \citet{azar2017minimax} and Lemma C.5 in \citet{jin2018q} for a proof of this result. This is a key technical result in obtaining improved bounds in $H$. However, this result does not hold in our setting for two reasons: first, our value function is not the expected sum of $H$-step rewards; second, while the value $V_{h+1}^*$ satisfies a Bellman recursion, the quantity $u(V_{h+1}^*(s')-\lambda^*_{h+1})$ (that appears in the variance operator) does not. Therefore, we may still have to use the crude bound for the variance term in the Bernstein-type bonus by using a maximum bound for $u(V_{h+1}^*(s')-\lambda^*_{h+1})$. This leads to the same bound as in our current paper and we do not obtain improvements in the regret with respect to $H$.

\section{Proof Sketch of Theorem~\ref{THM1}}\label{sec:sketch}
The structure of the proof of Theorem~\ref{THM1} follows the optimism principle in provably efficient risk-neutral RL (see, e.g., \citep[Chapter 7]{agarwal2019reinforcement}), however, we provide two new ingredients in our analysis: (a) concentration bounds for the OCE of the next-state
value function under estimated transitions (see \eqref{eq:oce-concen1} and \eqref{eq:oce-concen2}), and (b) a change-of-measure technique to bound the OCE of the estimated value function (under the
true transition) with an affine function (see \eqref{eq:measure-change}), and bound the the Radon-Nikodym derivative (see \eqref{eq:radon}). For notational simplicity, we use $P_h$ to denote $P_h(s_{h+1}^k\vert s_h^k,a_h^k)$ when there is no ambiguity.

% omit the state and action in quantities 

% and the key difference lies in the linearization concerning OCE. The expected regret to be bounded is $E[\sum_{k=1}^K(V_1^*(s_1^k)-V_1^{\pi^k}(s_1^k))]$. For notation simplicity, we leave out  $s_h^k,a_h^k$, abbreviate $P_h(s_{h+1}^k\vert s_h^k,a_h^k)$ to $P_h$ or $P_h(s_{h+1}^k),\forall k,h$ and view $OCE^u(X)$ as an optimization problem defined in \eqref{OCE} when there is no ambiguity. 

\noindent\textbf{Step 1: Optimism.} We can first show optimism, i.e., the event $\hat{V}_{h}^k\geq V_{h}^*$ for all $h,k$ holds with a high probability, where $\hat{V}_{h}^k$ is the estimated value function in our algorithm in episode $k$.
 This step relies on a concentration bound of the OCE of the optimal value function under the estimated transitions $\hat P_h$: with probability $1-\delta$ (where $\delta \in (0,1)$),
\begin{align}\label{eq:oce-concen1}
OCE^u_{P_h}(V_{h+1}^*)-OCE^u_{\hat{P}^k_h}(V_{h+1}^*) \le b_h^k. 
\end{align}
This bound can be proved by using the representation of the OCE in \eqref{OCE}, together with similar martingale arguments used in the risk-neutral RL setting \citep[Lemma 7.3]{agarwal2019reinforcement}. By optimism, the regret in \eqref{eq:regret} is upper bounded by $E[\sum_{k=1}^K(\hat{V}_1^k-V_1^{\pi^k})]$. 

% We bound OCE error of $V_{h+1}^*$:$OCE^u_{P_h}(V_{h+1}^*)-OCE^u_{\hat{P}_h}(V_{h+1}^*),\forall k,h$ with high probability $1-\delta$, denote the event as $\mathcal{G}_1$, and prove the optimism lemma:$\hat{V}_{h+1}^k\geq V_{h+1}^*$ under event $\mathcal{G}_1$, which is similar to that in the risk-neutral setting. The minor difference is that we linearize the OCE error of $V_{h+1}^*$ by
% \begin{align*}
% &OCE^u_{P_h}(V_{h+1}^*)-OCE^u_{\hat{P}_h}(V_{h+1}^*)\\
% &\leq \lambda^*_{h+1}+E_{s'\sim P_h}[u(V_{h+1}^*(s')-\lambda^*_{h+1})]\\
% &-\lambda^*_{h+1}-E_{s'\sim\hat{P}_h}[u(V_{h+1}^*(s')-\lambda^*_{h+1})]\\
% &=\sum_{s'}(P_h(s')-\hat{P}_h(s'))u(V_{h+1}^*(s')-\lambda^*_{h+1}),
% \end{align*}
% $\lambda^*_{h+1}$ is one of the optimal solutions to $OCE_{P_h}^u(V_{h+1}^*)$. 
% We have $E[\sum_{k=1}^K(V_1^*-V_1^{\pi^k})]\leq E[\sum_{k=1}^K(\hat{V}_1^k-V_1^{\pi^k})]$ thanks to the optimism lemma.

\noindent\textbf{Step 2: Bounding $\hat{V}_h^k-V_h^{\pi^k},\forall k,h$.} By definition, 
\begin{align*}
\hat{V}_h^k-V_h^{\pi^k}&\le b_h^k+OCE^u_{\hat{P}^k_h}(\hat{V}_{h+1}^k)-OCE_{P_h}^u(V^{\pi^k}_{h+1})\\
&=b_h^k+ \left[OCE^u_{\hat{P}^k_h}(\hat{V}_{h+1}^k)-OCE_{P_h}^u(\hat{V}_{h+1}^k) \right] + \left[OCE_{P_h}^u(\hat{V}_{h+1}^k)-OCE_{P_h}^u(V^{\pi^k}_{h+1})\right].
\end{align*}
% $\hat{V}_h^k-V_h^{\pi^k}$ is divided into three parts, and we have to bound $b_h^k$, $OCE^u_{\hat{P}^k_h}(\hat{V}_{h+1}^k)-OCE_{P_h}^u(\hat{V}_{h+1}^k)$ and $OCE_{P_h}^u(\hat{V}_{h+1}^k)-OCE_{P_h}^u(V^{\pi^k}_{h+1})$ separately. (I) $b_h^k$ is defined in \eqref{S11}. 

\noindent\textbf{Step 2.1:} The second term in the above equation can be bounded by using a concentration result for the OCE of the estimated value function $\hat{V}_{h+1}^k$: with probability $1-\delta,$ 
\begin{align} \label{eq:oce-concen2}
OCE^u_{\hat{P}^k_h}(\hat{V}_{h+1}^k)-OCE_{P_h}^u(\hat{V}_{h+1}^k) \le \sqrt{S} \cdot b_h^k. 
\end{align}
The extra $\sqrt{S}$ factor, compared with \eqref{eq:oce-concen1}, is because both $\hat{V}_{h+1}^k$ and $\hat{P}^k_h$ are random and we use $L^1$ concentration bounds for $||\hat{P}^k_h - {P}_h||_{1}$ as in \cite{jaksch10a}.  

% (II) $OCE^u_{\hat{P}^k_h}(\hat{V}_{h+1}^k)-OCE_{P_h}^u(\hat{V}_{h+1}^k)$ can be bounded by $2\vert u(-H+h)\vert\sqrt{\frac{S\log\left (\frac{SAHK}{\delta}\right )}{\max\{1,N_h^k\}}}$ with high probability $1-\delta$ and the event is denoted as $\mathcal{G}_2$. The proof is adapted from the risk-neutral setting with the same linearization procedure in Step 1. 

\noindent\textbf{Step 2.2:} The third term $OCE_{P_h}^u(\hat{V}_{h+1}^k)-OCE_{P_h}^u(V^{\pi^k}_{h+1})$ is more difficult to bound. % is the key difference compared with the risk-neutral setting. 
Because the OCE is a concave nonlinear functional, we expect that $$OCE_{P_h}^u(\hat{V}_{h+1}^k)-OCE_{P_h}^u(V^{\pi^k}_{h+1}) \le \ell ( \hat{V}_{h+1}^k - V^{\pi^k}_{h+1}),$$ where $\ell(\cdot)$ is a linear function of random variables and it is a subgradient of the OCE. We actually show that $\ell$ can be represented in the form of an expectation:
\begin{align}\label{eq:measure-change}
OCE_{P_h}^u(\hat{V}_{h+1}^k)-OCE_{P_h}^u(V^{\pi^k}_{h+1}) \le E_{B_h} ( \hat{V}_{h+1}^k - V^{\pi^k}_{h+1}),
\end{align}
where the expectation $E_{B_h}[\cdot]$ is taken with respect to a new probability distribution $B_h$ that is linked to the true transition distribution $P_h$ by change-of-measure. Specifically, 
using the first order optimality condition for $OCE^u_{P_h}(V_{h+1}^{\pi^k})$ as a concave optimization problem, we have $1\in E_{s'\sim P_h}[\partial u(V_{h+1}^{\pi^k}(s')-\lambda_{h+1}^k)]$ where $\lambda_{h+1}^k$ is an optimal solution. We can find a subgradient $\Lambda_{h+1}^k(s')\in \partial u(V_{h+1}^{\pi^k}(s')-\lambda_{h+1}^k)$ that satisfies $E_{P_h}[\Lambda_{h+1}^k]=1$, and define the new distribution $B_h$ by
$$B_h(s'|s,a)=P_h(s'|s,a)\Lambda_{h+1}^k(s'), \quad \forall s'\in\mathcal{S}.$$ Here, $\Lambda_{h+1}^k$ is the Radon-Nikodym derivative.

By combining Steps 2.1 and 2.2, we obtain
\begin{align}\label{eq:recur1}
\hat{V}_h^k-V_h^{\pi^k}&\le 2 \sqrt{S} \cdot b_h^k+ E_{B_h}[\hat{V}_{h+1}^k-V_{h+1}^{\pi^k}].
\end{align}

\noindent\textbf{Step 3: Bounding the regret.} 
% Under event $\mathcal{G}:=\mathcal{G}_1\cap \mathcal{G}_2$, for each episode $k\in [K]$, 
% \begin{align*}
% &\hat{V}_1^k-V_1^{\pi^k}\\
% &\leq b_1^k+OCE^u_{\hat{P}^k_1}(\hat{V}_{2}^k)-OCE^u_{P_1}(\hat{V}_{2}^k)+E_{B_1}[\hat{V}_{2}^k-V_{2}^{\pi^k}].
% \end{align*}
% Doing this recursively, we can obtain
Applying \eqref{eq:recur1} recursively over $h$ and using \eqref{S11}, we have that with probability $1- 2 \delta,$
\begin{align}\label{eq:step3}
&\sum_{k=1}^K \left(\hat{V}_1^k-V_1^{\pi^k} \right)\leq \sum_{h=1}^H \sum_{k=1}^K E_{w_{hk}^{B}}\left [2\sqrt{2} \vert u(-H+h)\vert\sqrt{\frac{S\log\left (\frac{SAHK}{\delta}\right )}{N_h^k}} \right ],
\end{align}
where $w_{hk}^B$ is the probability of $\pi^k$ visiting $(s_h^k,a_h^k)$ at step $h$ starting from $s_1^k$ under probability measures $B_i(\cdot\vert s_i^k,a_i^k),i=1,\cdots,h-1$. Specifically,
%where we can interchange the summation over $k$ and $h$ and
\begin{equation*}
E_{w_{hk}^{B}}[\cdot] :=\left\{
	\begin{aligned}
	&1 \quad &h=1, \\
	&E_{B_1}\left [E_{B_2}\left [\cdots E_{B_{h-1}}\left [\cdot\right ]\right ]\right ] \quad &h\geq 2. \\
	\end{aligned}
	\right
	.
\end{equation*}

The main difficulty in bounding $E[\sum_{k=1}^K(\hat{V}_1^k-V_1^{\pi^k})]$ is that $w_{hk}^B$ is built upon the probability measure $B_h$ for any $k\in [K],h\in [H]$ while we have to take expectation under probability measure $P_h$ outside the summation over $k\in [K]$. To address this issue, we first note that $E_{w_{hk}^B}\left[\frac{1}{\sqrt{N_h^k}}\right] = E\left[\Lambda_2^k\cdots \Lambda_h^k \frac{1}{\sqrt{N_h^k}}\Bigg\vert s_1^k,a_1^k\right]$, which implies $E\left[\sum_{k=1}^K E_{w_{hk}^B}\left[\frac{1}{\sqrt{N_h^k}}\right]\right] = 
\sum_{k=1}^K E\left[\Lambda_2^k\cdots \Lambda_h^k \frac{1}{\sqrt{N_h^k}}\right].$
% For notation simplicity, we consider $E\left[\sum_{k=1}^K E_{w_{hk}^B}\left[\frac{1}{\sqrt{N_h^k}}\right]\right]$, denoted as $(*)$, and  $E_{w_{hk}^B}\left[\frac{1}{\sqrt{N_h^k}}\right]$, denoted as $(**)$. By the definition of $B_h$, for $\forall k$, $(**)=E\left[\Lambda_2^k\cdots \Lambda_h^k \frac{1}{\sqrt{N_h^k}}\Bigg\vert s_1^k,a_1^k\right].$ Moving the summation over $k$ outside the expectation and using the law of total expectation, we derive that $(*)=\sum_{k=1}^K E\left[\Lambda_2^k\cdots \Lambda_h^k \frac{1}{\sqrt{N_h^k}}\right].$ 
Using Cauchy–Schwarz inequality, it can be upper bounded by $\sqrt{\sum_{k=1}^K E[\Lambda_2^k\cdots \Lambda_h^k]^2\cdot \sum_{k=1}^K E\left[\frac{1}{N_h^k}\right]}$. It is well-known that $\sum_{k=1}^K\frac{1}{N_h^k} \le SA \log(3K)$ \citep{azar2017minimax}. One can show that $E[\Lambda_2^k\cdots \Lambda_h^k]=1$ and $0 \le \Lambda_{i+1}^k \le u_{-}'(-H+i)$. Then we have 
\begin{align}\label{eq:radon}
\sum_{k=1}^K E[\Lambda_2^k\cdots \Lambda_h^k]^2 \le \prod\limits_{i=1}^{h-1} u_{-}'(-H+i)K.
\end{align}
Summing over $h$, choosing
$\delta=\frac{1}{2KH}$ and applying a standard argument (see, e.g., Chapter 7.3 of \cite{agarwal2019reinforcement} ), we obtain the bound on the expected regret in Theorem \ref{THM1}.

%\gao{possible critism: Overall the proof structure still seems to largely follow the optimism principle in provably efficient reinforcement learning. \cite{azar2017minimax, zanette2019tighter}. It would be great if the authors could further elaborate on how their analysis is not a combination of existing techniques in RL and existing concentration results from the literature.}

%\gao{Proofs are rather compressed. Please provide intuitive explanation general proof ideas.}
% Now, we are well-prepared for the proof of theorem \ref{THM1}. In Lemma \ref{L7}, we overcome the difficulty of the recursion with nonlinear risk measures in the regret analysis by measure change. However, the introduction of the new probability measures $B_h(\cdot\vert s_h^k,a_h^k)$ for any $h\in [H],k\in [K]$ influence the state-action distribution $w_{hk}(s_h^k,a_h^k)$, so we are unable to bound $\sum_{k=1}^K\sum_{h=1}^H\frac{1}{N_h^k(s_h^k,a_h^k)}$ directly as in  \cite{jaksch10a} and \cite{azar2017minimax}. We deal with the distorted distribution by Lemma \ref{L13}, which offers a bound as good as that in the risk-neutral setting. 

%%%%%%%%%%%%%%%%%%%%%%%%%%%%%%%%%%%%%%%%%%%%%%%
\section{Conclusion and Future Work}\label{sec:conclusion}

In this paper we have proposed a risk-sensitive RL formulation based on episodic finite MDPs with recursive OCEs. We develop a learning algorithm, OCE-VI,  and establish a worst-case regret upper bound. We also prove a regret lower bound, showing that the regret rate achieved by our proposed algorithm actually has the optimal dependence on the numbers of episodes and actions. Because OCEs encompass a wide family of risk measures, our paper generates new regret bounds for episodic risk-sensitive RL problems with those risk measures. 

Regret minimization for risk-sensitive MDPs is still largely unexplored. For future work, one important direction is to improve regret bounds in the number of states and the horizon length. Other interesting directions include, to name a few, studying large or continuous state/action spaces, considering risk measures other than OCEs, and obtaining problem-dependent regret bounds.

%%%%%%%%%%%%%%%%%%%%%%%%%%%%%%%%%%%%%%%%%%%%%%%%%%%%%%%

\newpage
\bibliographystyle{chicago}
\bibliography{main}

%%%%%%%%%%%%%%%%%%%%%%%%%%%%%%%%%%%%%%%%%%%%%%%%%%%%%%%%

\newpage
\appendix
\onecolumn
%\section{You \emph{can} have an appendix here.}

\section{Preliminary Lemmas}
In this section, we present some preliminary lemmas that will be used in the proofs of Theorems~\ref{THM1} and \ref{THM3}. % to prepare for the proofs.

Lemma \ref{L0} is \citep[Theorem 2.1]{ben2007old} and it summarizes
some fundamental properties of OCE. %, the proof of Lemma \ref{L0} can be found in \cite{ben2007old} Theorem 2.1.

\begin{lemma}(Main Properties of OCE)\label{L0}
For any utility function $u\in U_0$, and any bounded random variable $X$ the following properties hold:\\
(a)Shift Additivity. $OCE^u(X+c)=OCE^u(X)+c,\forall c\in \mathbb{R}$.\\
(b)Consistency. $OCE^u(c)=c$, for any constant $c\in \mathbb{R}$.\\
(c)Monotonicity. Let $Y$ be any random variable such that $X(w)\leq Y(w),\forall w\in \Omega$. Then,
\begin{align*}
OCE^u(X)\leq OCE^u(Y).
\end{align*}
(d)Concavity. For any random variables $X_1,X_2$ and any $\mu\in (0,1)$, we have
\begin{equation*}
OCE^u(\mu X_1+(1-\mu)X_2)\geq \mu OCE^u(X_1)+(1-\mu)OCE^u(X_2).
\end{equation*} 
\end{lemma}

The following lemma provides preliminary bounds for the value functions (see \eqref{eq:Qpi} and \eqref{eq:Vpi}) and the regret of learning algorithms.
\begin{lemma}\label{L1}
For any $s\in \mathcal{S},a\in \mathcal{A},h\in [H],\pi \in \mathcal{\varPi}$  and $u\in U_0$, we have $Q_h^{\pi}(s,a)\in [0,H-h+1]$ and $V_h^{\pi}(s)\in [0,H-h+1]$. Consequently, for each $K\geq 1$, we have $0\leq Regret(\mathcal{M},\textbf{algo},K)\leq KH$ for any \textbf{algo}.
\end{lemma}
\begin{proof} Recall that $Q_h^{\pi}(s,a)=r_h(s,a)+OCE^{u}_{s'\sim P_h(\cdot\vert s,a)}(V^{\pi}_{h+1}(s'))$ and $V_{H+1}^{\pi}(s)=0,\forall s\in\mathcal{S}$. Then, we can calculate that
\begin{align*}
&Q_H^{\pi}(s,a)=r_H(s,a)+OCE^{u}_{s'\sim P_H(\cdot\vert s,a)}(V_{H+1}^{\pi}(s'))\overset{(1)}{=}r_H(s,a)\in [0,1],\\
&V_H^{\pi}(s)=\max_{a\in \mathcal{A}}r_H(s,a)\in [0,1],
\end{align*}
where equality (1) is due to property (b) in Lemma \ref{L0}. Hence, we have 
\begin{align*}
&Q_{H-1}^{\pi}(s,a)=r_{H-1}(s,a)+OCE^{u}_{s'\sim P_{H-1}(\cdot\vert s,a)}(V_{H}^{\pi}(s'))\overset{(1)}{\leq} r_{H-1}(s,a)+OCE^{u}_{s'\sim P_{H-1}(\cdot\vert s,a)}(1)\in [1,2],\\
&Q_{H-1}^{\pi}(s,a)=r_{H-1}(s,a)+OCE^{u}_{s'\sim P_{H-1}(\cdot\vert s,a)}(V_{H}^{\pi}(s'))\overset{(2)}{\geq} r_{H-1}(s,a)+OCE^{u}_{s'\sim P_{H-1}(\cdot\vert s,a)}(0)\in [0,1],\\
&V_{H-1}^{\pi}(s)=\max_{a\in \mathcal{A}}Q_{H-1}^{\pi}(s,a)\in [0,2],\\
\end{align*}
where inequalities (1) and (2) hold due to properties (b) and (c) in Lemma \ref{L0}. Carrying out this procedure repeatedly until step $h$, we can get 
\begin{align*}
Q_h^{\pi}(s,a)\in [0,H-h+1], \quad \text{and} \quad V_h^{\pi}(s)\in [0,H-h+1].
\end{align*}
Using the definition \eqref{eq:regret}, we then immediately obtain that 
$0\leq Regret(\mathcal{M},\textbf{algo},K)\leq KH$ for any \textbf{algo}.
%Therefore, we can prove that $Q_h^{\pi}(s,a)\in [0,H-h+1],V_h^{\pi}(s)\in [0,H-h+1], \forall s\in \mathcal{S},a\in \mathcal{A},h\in [H],\pi \in \mathcal{\varPi},u\in U_0$. The bounds for $Regret(\mathcal{M},\textbf{algo},K)$ can be proved similarly.
\end{proof}

% Although the nonlinearity embedded in recursive utility and OCE, the bounds for the value functions and regret are on the same scale as those in the risk-neutral episodic MDP setting. This result is key to the following lemma, which restricts the optimal solution $\lambda^*$ of the optimization problem \eqref{OCES3} to $[0,H-h]$ for any $s\in \mathcal{S},a\in\mathcal{A},h\in [H]$. The proof of Lemma \ref{L2} is an adaptation of Proposition 2.1 in \cite{ben2007old}.

With Lemma~\ref{L1}, we can obtain the following result, which shows that the optimization problem in $OCE^{u}_{s'\sim P_h(\cdot\vert s,a)}(V_{h+1}^{\pi}(s'))$ has an optimal solution in the support of the random variable $V_{h+1}^{\pi}(s')$.

\begin{lemma}\label{L2}
For any probability measure $P_h(\cdot\vert s,a)$, any $s\in \mathcal{S},a\in\mathcal{A},h\in [H]$, suppose $V_{h+1}^{\pi}(s')\in [0,H-h]$ for $s'\sim P_h(\cdot\vert s,a)$. Then, we have 
\begin{equation}\label{S1}
OCE^{u}_{s'\sim P_h(\cdot\vert s,a)}(V_{h+1}^{\pi}(s'))=\max_{\lambda\in [0,H-h]}\{\lambda+E_{s'\sim P_h(\cdot\vert s,a)}[u(V_{h+1}^{\pi}(s')-\lambda)]\}.
\end{equation}
\end{lemma}
\begin{proof}
Note that $V_{h+1}^{\pi}(s')\in [0,H-h]$ for $s'\sim P_h(\cdot\vert s,a)$ by Lemma~\ref{L1}.
%, which states that for computing the optimal $\lambda$ in the OCE, one can restrict attention to the support of the random variable. In our setting, $V_{h+1}^{\pi}(s')\in [0,H-h]$ for $s'\sim P_h(\cdot\vert s,a)$ by Lemma~\ref{L1}.
By the concavity and continuity of $u$, we deduce that
the function $G(\lambda) :=\lambda+E_{s'\sim P_h(\cdot\vert s,a)}[u(V^{\pi}_{h+1}(s')-\lambda)]$ is concave and continuous, and moreover, $G(\lambda) \le E_{s'\sim P_h(\cdot\vert s,a)}[V^{\pi}_{h+1}(s')] < \infty$ for all $\lambda \in \mathbb{R}$ due to the fact that $u(x) \le x$ for all $x$. In addition, $\partial G(\lambda)=1-E_{s'\sim P_h(\cdot\vert s,a)}[\partial u(V_{h+1}^{\pi}(s')-\lambda)]$ due to the finite state space $\mathcal{S}$, and thus, $G(\lambda)$ will be nonincreasing when $\lambda\geq H-h$ due to the fact that $\eta\geq 1$ for all $\eta\in \partial u(x),x\leq 0$. It follows that the set of optimal solutions to the problem $\sup_{\lambda \in \mathbb{R}} G(\lambda)$ is nonempty. Hence, we can apply Proposition 2.1 in \cite{ben2007old} and obtain the desired result. 
\end{proof}
\section{Proof of Theorem \ref{THM1}}\label{app:B}

We present a series of lemmas in Section~\ref{sec:lemmasTHM1}, and prove Theorem \ref{THM1} in Section~\ref{PFTHM1}.  The relation of different lemmas is given below. %Our main discoveries are Lemmas \ref{L6} and \ref{L13}, which help us overcome the nonlinearity of the OCE in the regret analysis. 

%The main difficulty in the analysis lies in the nonlinearity of the OCE as a function of random variables.

%\gao{WH: revise later since I made some changes (combining original Lemma 6 and 5)}

\tikzstyle{pict} = [rectangle, rounded corners, minimum width = 2.2cm, minimum height = 0.6cm, text centered, draw = black]

\begin{center}
\begin{tikzpicture}[node distance=20pt]
  \node[pict](L3){Lemma \ref{L3}};
  \node[pict, right=of L3](L4){Lemma \ref{L4}};
  \node[pict, right=of L4](L5){Lemma \ref{L5}};
  \node[pict, below=of L4](PROB){Lemma \ref{PROB}};
  \node[pict, right=of PROB](LP4){Lemma \ref{LP4}};
  \node[pict, below=of PROB](L6){Lemma \ref{L6}};
  \node[pict, right=of L6](L7){Lemma \ref{L7}};
  \node[pict, below=of L6](LP13){Lemma \ref{LP13}};
  \node[pict, right=of LP13](L13){Lemma \ref{L13}};
  \node[pict, right=50pt of LP4,xshift=25pt,yshift=-12pt](THM1){Theorem \ref{THM1}};
  \draw[->] (L3)  -- (L4);
  \draw[->] (L4)  -- (L5);
  \draw[->] (PROB)  -- (LP4);
  \draw[->] (L6)  -- (L7);
  \draw[->] (LP13)  -- (L13);
  \draw[->] (L5)  -- (THM1);
  \draw[->] (LP4)  -- (THM1);
  \draw[->] (L7)  -- (THM1);
  \draw[->] (L13)  -- (THM1);
\end{tikzpicture}
\end{center}

%\gao{WH: add a picture about logic flow of all lemmas to Thm 1}

\subsection{Preparations for the Proof of Theorem \ref{THM1}}\label{sec:lemmasTHM1}

In this subsection, we state and prove a few lemmas needed for the proof of theorem \ref{THM1}. Recall that the bonus in the OCE-VI algorithm is $b_h^k(s,a)=\vert u(-H+h)\vert\sqrt{\frac{2\log\left (\frac{SAHK}{\delta}\right )}{\max\{1,N^k_h(s,a)\}}}$ for any $(s,a,h,k)\in \mathcal{S}\times \mathcal{A}\times [H]\times [K]$. %For notation simplicity, we abuse notation a bit and denote $\max\{1,N^k_h(s,a)\}$ as $N^k_h(s,a)$ when there is no ambiguity.

%Using Azuma-Hoeffding concentration inequality, 
Lemma \ref{L3} provides a bound for the difference between $E_{s'\sim P_h(\cdot\vert s,a)}\left [u(V_{h+1}^*(s')-\lambda_{h+1}^*)\right ]$ and its estimation for all $h\in [H]$, where $$\lambda_{h+1}^*\in \mathop{\arg\max}_{\lambda\in [0,H-h]}\{\lambda+E_{s'\sim P_h(\cdot\vert s,a)}[u(V_{h+1}^*(s')-\lambda)]\}.$$ 
% Note that the solution to 
% \begin{align*}
% \max_{\lambda\in [0,H-h]}\{\lambda+E_{s'\sim P_h(\cdot\vert s,a)}[u(V_{h+1}^*(s')-\lambda)]\}
% \end{align*}
% exists but is not unique, since the utility function is a concave function and the feasible set is compact and convex. 
Note that both $V_{h+1}^*$ and $\lambda_{h+1}^*$ are deterministic quantities. To facilitate the presentation, we let 
\begin{equation}\label{HIST}
\mathbb{H}_h^k=((\mathcal{S}\times \mathcal{A})^{H-1}\times \mathcal{S})^{k-1}\times(\mathcal{S}\times\mathcal{A})^{h-1}\times\mathcal{S}
\end{equation}
be the set of possible histories up to step $h$ in episode $k$. Then, one sample of the history up to step $h$ in episode $k$ is
\begin{align*}
\mathcal{H}_h^k=(s_1^1,a_1^1,s_2^1,a_2^1,\cdots,s_H^1,\cdots,s_1^k,a_1^k,\cdots,s_{h-1}^k,a_{h-1}^k,s_h^k)\in \mathbb{H}_h^k.
\end{align*}

\begin{lemma}\label{L3}
For any $\delta\in (0,1)$, we have
\begin{equation}\label{S2}
\begin{aligned}
P\Bigg (& E_{s'\sim P_h(\cdot\vert s,a)}\left [u(V_{h+1}^*(s')-\lambda_{h+1}^*)\right ]-E_{s'\sim \hat{P}^k_h(\cdot\vert s,a)}\left [u(V_{h+1}^*(s')-\lambda_{h+1}^*)\right ]\\
&\leq \left\vert u(H-h-\lambda_{h+1}^*)-u(-\lambda_{h+1}^*)\right\vert \sqrt{\frac{2\log(\frac{SAHK}{\delta})}{\max\{1,N^k_h(s,a)\}}},\\
&V_{h+1}^*:\mathcal{S}\to [0,H-h],\lambda_{h+1}^*\in [0,H-h],\forall (s,a,h,k)\in \mathcal{S}\times \mathcal{A}\times [H]\times [K]\Bigg)\geq 1-\delta.
\end{aligned}
\end{equation}
\end{lemma}
\begin{proof}
We adapt the proof of Lemma 7.3 in \cite{agarwal2019reinforcement} who consider the risk neural episodic RL setting. For each fixed $(s,a,h,k)\in \mathcal{S}\times\mathcal{A}\times [H]\times [K]$, we have to consider two cases.

Firstly, according to section \ref{SECALG}, when $N_h^k(s,a)=0$, we have $\hat{P}_h^k(s'\vert s,a)=0$ for all $s'\in\mathcal{S}$. According to Lemma \ref{L1} and Lemma \ref{L2}, $V^*_{h+1}(s^i_{h+1})\in [0,H-h]$ and $\lambda_{h+1}^*\in [0,H-h]$. Thus, we have $u(V_{h+1}^*(s_{h+1}^i)-\lambda_{h+1}^*)\in [u(-\lambda_{h+1}^*),u(H-h-\lambda_{h+1}^*)]$, where $u(-\lambda_{h+1}^*)\leq 0$ and $u(H-h-\lambda_{h+1}^*)\geq 0$. Then, we have
\begin{align*}
&E_{s'\sim P_h(\cdot\vert s,a)}\left [u(V_{h+1}^*(s')-\lambda_{h+1}^*)\right ]-E_{s'\sim \hat{P}^k_h(\cdot\vert s,a)}\left [u(V_{h+1}^*(s')-\lambda_{h+1}^*)\right ]\\
&\overset{(1)}{=} E_{s'\sim P_h(\cdot\vert s,a)}\left [u(V_{h+1}^*(s')-\lambda_{h+1}^*)\right ]\\
&\overset{(2)}{\leq} \vert u(H-h-\lambda_{h+1}^*)-u(-\lambda_{h+1}^*)\vert\\
&\overset{(3)}{\leq} \left\vert u(H-h-\lambda_{h+1}^*)-u(-\lambda_{h+1}^*)\right\vert \sqrt{\frac{2\log(\frac{SAHK}{\delta})}{\max\{1,N^k_h(s,a)\}}},
\end{align*}
where equality (1) holds because $\hat{P}_h^k(s'\vert s,a)=0$ for all $s'\in\mathcal{S}$, inequality (2) holds because $$u(V_{h+1}^*(s_{h+1}^i)-\lambda_{h+1}^*)\leq u(H-h-\lambda^*_{H+1})\leq \vert u(H-h-\lambda_{h+1}^*)-u(-\lambda_{h+1}^*)\vert,$$
and inequality (3) holds because $N_h^k(s,a)=0$ and $\log(\frac{SAHK}{\delta})>1$.
Therefore,
\begin{align*}
P\Bigg (& E_{s'\sim P_h(\cdot\vert s,a)}\left [u(V_{h+1}^*(s')-\lambda_{h+1}^*)\right ]-E_{s'\sim \hat{P}^k_h(\cdot\vert s,a)}\left [u(V_{h+1}^*(s')-\lambda_{h+1}^*)\right ]\\
&\leq \left\vert u(H-h-\lambda_{h+1}^*)-u(-\lambda_{h+1}^*)\right\vert \sqrt{\frac{2\log(\frac{SAHK}{\delta})}{\max\{1,N^k_h(s,a)\}}},\\
&V_{h+1}^*:\mathcal{S}\to [0,H-h],\lambda_{h+1}^*\in [0,H-h]\Bigg)=1\geq 1-\frac{\delta}{SAHK}.
\end{align*}

Secondly, when $N_h^k(s,a)\geq 1$, by the definition of $\hat{P}_h^k$, we have
\begin{align*}
E_{s'\sim \hat{P}_h^k(\cdot\vert s,a)}\left [u(V_{h+1}^*(s')-\lambda_{h+1}^*)\right ]=\frac{1}{N^k_h(s,a)}\sum_{i=1}^{k-1} 1_{\{(s_h^i,a_h^i)=(s,a)\}}u(V_{h+1}^*(s_{h+1}^i)-\lambda_{h+1}^*).
\end{align*}
Remark that when $N_h^k(s,a)\geq 1$, we have $k\geq 2$.
%Then, we can define the martinagle difference sequences 
We define for $i=1, \ldots, k-1,$
\begin{align*}
X_i=E\left [1_{\{(s_h^i,a_h^i)=(s,a)\}}u(V_{h+1}^*(s_{h+1}^i)-\lambda_{h+1}^*)\Big\vert \mathcal{H}_h^i\right ]-1_{\{(s_h^i,a_h^i)=(s,a)\}}u(V_{h+1}^*(s_{h+1}^i)-\lambda_{h+1}^*).
\end{align*}
By the same argument as in the previous case, we conclude that $$u(V_{h+1}^*(s_{h+1}^i)-\lambda_{h+1}^*)\in [u(-\lambda_{h+1}^*),u(H-h-\lambda_{h+1}^*)].$$ Thus, we have $$u(-\lambda_{h+1}^*)-u(H-h-\lambda_{h+1}^*)\leq X_i\leq u(H-h-\lambda_{h+1}^*)-u(-\lambda_{h+1}^*).$$ In addition, it is evident that $E[X_i\vert \mathcal{H}_h^i]=0$, which implies that $(X_i)$ is a martingale difference sequence. 
Then, by Azuma-Hoeffding's inequality for martingales, with a probability of at least $1-\frac{\delta}{SAHK}$, we have
\begin{align*}
\sum_{i=1}^{k-1} X_i&=N_h^k(s,a)E_{s'\sim P_h(\cdot\vert s,a)}[u(V_{h+1}^*(s')-\lambda_{h+1}^*)]-\sum_{i=1}^{k-1} 1_{\{(s_h^i,a_h^i)=(s,a)\}}u(V_{h+1}^*(s_{h+1}^i)-\lambda_{h+1}^*)\\
&\leq \vert u(H-h-\lambda_{h+1}^*)-u(-\lambda_{h+1}^*)\vert \sqrt{2 N_h^k(s,a)\log(\frac{SAHK}{\delta})}
\end{align*}
Divided by $N_h^k(s,a)$ on both sides of the above inequality, combining the above two cases and using a union bound over all $(s,a,h,k)\in \mathcal{S}\times \mathcal{A}\times [H]\times [K]$, we obtain \eqref{S2}.
% \begin{equation*}
% \begin{aligned}
% P\Bigg (&\left\vert E_{s'\sim P_h(\cdot\vert s,a)}\left [u(V_{h+1}^*(s')-\lambda^*)\right ]-E_{s'\sim \hat{P}^k_h(\cdot\vert s,a)}\left [u(V_{h+1}^*(s')-\lambda^*)\right ]\right\vert\\
% &\leq 2\left\vert u(H-h-\lambda^*)-u(-\lambda^*)\right\vert \sqrt{\frac{\log(\frac{SAHK}{\delta})}{N_h^k(s,a)}},\\
% &\forall (s,a,h,k)\in \mathcal{S}\times \mathcal{A}\times [H]\times [K],V_{h+1}^*:\mathcal{S}\to [0,H-h],\lambda^*\in [0,H-h]\Bigg)\geq 1-2\delta,
% \end{aligned}
% \end{equation*}
%which completes the proof.
\end{proof}

% In order to carry out the regret analysis for the OCE-VI Algorithm, we have to derive the concentration inequality for the OCE of the value functions. 
By Lemma \ref{L3}, we can derive the following concentration bound for the OCE applied to the optimal value function at the next state (under the estimated transition distribution).  

% regret analysis for the OCE-VI Algorithm, we have to derive the concentration inequality for the OCE of the value functions. 

% an OCE error bound for $V_{h+1}^*:\mathcal{S}\to [0,H-h]$ by linearizing the difference between OCE and sample OCE.\gao{GAO: this paragraph to be polished}
\begin{lemma}\label{L4}
For any $\delta\in (0,1)$, we have that 
\begin{equation}
\begin{aligned}
P\Bigg( & OCE^{u}_{s'\sim P_h(\cdot\vert s,a)}(V_{h+1}^*(s'))-OCE^{u}_{s'\sim \hat{P}^k_h(\cdot\vert s,a)}(V_{h+1}^*(s')) \leq \vert u(-H+h)\vert\sqrt{\frac{2\log\left (\frac{SAHK}{\delta}\right )}{\max\{1,N_h^k(s,a)\}}},\\
&V_{h+1}^*:\mathcal{S}\to [0,H-h],\forall (s,a,h,k)\in \mathcal{S}\times \mathcal{A}\times [H]\times [K]\Bigg)\geq 1-\delta.
\end{aligned}
\end{equation}
\end{lemma}

\begin{proof}
According to Lemma \ref{L3}, with a probability of at least $1-\delta$, for any $k\in [K],s\in \mathcal{S},a\in \mathcal{A}, h\in [H]$, we have
\begin{align*}
&OCE^{u}_{s'\sim P_h(\cdot\vert s,a)}(V_{h+1}^*(s'))-OCE^{u}_{s'\sim \hat{P}^k_h(\cdot\vert s,a)}(V_{h+1}^*(s'))\\
&\overset{(1)}{=} \max_{\lambda\in [0,H-h]}\{\lambda+E_{s'\sim P_h(\cdot\vert s,a)}[u(V_{h+1}^*(s')-\lambda)]\}-\max_{\lambda\in [0,H-h]}\{\lambda+E_{s'\sim \hat{P}^k_h(\cdot\vert s,a)}[u(V_{h+1}^*(s')-\lambda)]\}\\
&\overset{(2)}{\leq}  \lambda_{h+1}^*+E_{s'\sim P_h(\cdot\vert s,a)}[u(V_{h+1}^*(s')-\lambda_{h+1}^*)]-\lambda_{h+1}^*-E_{s'\sim \hat{P}^k_h(\cdot\vert s,a)}[u(V_{h+1}^*(s')-\lambda_{h+1}^*)]\\
&\overset{(3)}{\leq} \left\vert u(H-h-\lambda_{h+1}^*)-u(-\lambda_{h+1}^*)\right\vert \sqrt{\frac{2\log(\frac{SAHK}{\delta})}{\max\{1,N^k_h(s,a)\}}}\\
&\leq \sup_{\lambda\in [0,H - h]}\vert u(H-h-\lambda)-u(-\lambda)\vert\sqrt{\frac{2\log\left (\frac{SAHK}{\delta}\right )}{\max\{1,N^k_h(s,a)\}}}, %\\
% &\overset{(2)}{=} 2\vert u(-H+h)\vert\sqrt{\frac{\log\left (\frac{SAHK}{\delta}\right )}{N_h^k(s,a)}},
\end{align*}
where equality (1) follows from Lemma \ref{L2}, inequality (2) holds because $\lambda_{h+1}^*$ is the optimal solution to $\max_{\lambda\in [0,H-h]}\{\lambda+E_{s'\sim P_h(\cdot\vert s,a)}[u(V_{h+1}^*(s')-\lambda)]\}$ and inequality (3) follows from Lemma \ref{L3}. One can check that $u(H-h-\lambda)-u(-\lambda)$ is a nondecreasing function of $\lambda\in [0,H-h]$. To see this, note that the subdifferential of $u(H-h-\lambda)-u(-\lambda)$ is $\partial u(-\lambda)-\partial u(H-h-\lambda)$ and for any $z\in \partial u(-\lambda)-\partial u(H-h-\lambda)$, we have $z\geq 0$, because the utility function $u$ is concave. This implies that the function $u(H-h-\lambda)-u(-\lambda)$ is nondecreasing. In addition, this function is non-negative because $u$ is nondcreasing and thus $u(H-h)-u(0)\geq 0$ for $h \in [H]$. Thus, we can deduce that $\sup_{\lambda\in [0,H - h]}\vert u(H-h-\lambda)-u(-\lambda)\vert \le u(0) - u(-H+h) = |u(-H+h)|$ since $u(0)=0$. The proof is then complete. 
% and equality (3) holds by the fact that $u(H-h-\lambda)-u(-\lambda)$ is a nondecreasing function over $\lambda\in [0,H-h]$ and $u(H-h)-u(0)\geq 0$. Note that the subdifferential of $u(H-h-\lambda)-u(-\lambda)$ is $\partial u(-\lambda)-\partial u(H-h-\lambda)$ and for any $z\in \partial u(-\lambda)-\partial u(H-h-\lambda)$, we have $z\geq 0$, because the utility function $u$ is concave, thus $u(H-h-\lambda)-u(-\lambda)$ is nondecreasing.
% We complete the proof by the symmetry of the problem.
\end{proof}

Lemma~\ref{L4} immediately implies the following result. To facilitate the presentation, we define the following event from Lemma~\ref{L4}:
\begin{align}\label{G1}
\mathcal{G}_1=\Bigg\{& OCE^{u}_{s'\sim P_h(\cdot\vert s,a)}(V_{h+1}^*(s'))-OCE^{u}_{s'\sim \hat{P}^k_h(\cdot\vert s,a)}(V_{h+1}^*(s')) \leq \vert u(-H+h)\vert\sqrt{\frac{2\log\left (\frac{SAHK}{\delta}\right )}{\max\{1,N_h^k(s_h^k,a_h^k)\}}},\nonumber\\
&V_{h+1}^*:\mathcal{S}\to [0,H-h],\forall (s,a,h,k)\in \mathcal{S}\times \mathcal{A}\times [H]\times [K]\Bigg\}.
\end{align}

\begin{lemma}[Optimism]\label{L5}
Conditional on the event $\mathcal{G}_1$, we have
$\hat{V}_h^k(s)\geq V^*_h(s)$ for any $k\in [K],s\in \mathcal{S}, h\in [H]$.

\end{lemma}
\begin{proof}
We prove the result by induction. 
Set $\hat{V}^k_{H+1}(s)=V^*_{H+1}(s)=0,\forall s\in \mathcal{S}$. Conditional on the occurrence of the event $\mathcal{G}_1$, assume $\hat{V}^{k}_{h+1}(s')\geq V^*_{h+1}(s'),\forall s'\in\mathcal{S}$. Then, under event $\mathcal{G}_1$, for step $h$, we have
\begin{align}
&b_h^k(s,a)+r_h(s,a)+OCE^{\phi}_{s'\sim \hat{P}^k_h(\cdot\vert s,a)}(\hat{V}^k_{h+1}(s'))-r_h(s,a)-OCE^{\phi}_{s'\sim P_h(\cdot\vert s,a)}(V^*_{h+1}(s')) \nonumber \\
& \overset{(1)}{\geq} b_h^k(s,a)+OCE^{\phi}_{s'\sim \hat{P}^k_h(\cdot\vert s,a)}(V^*_{h+1}(s'))-OCE^{\phi}_{s'\sim P_h(\cdot\vert s,a)}(V^*_{h+1}(s')) \nonumber \\
& \overset{(2)}{\geq} b_h^k(s,a)-\vert u(-H+h)\vert\sqrt{\frac{2\log\left (\frac{SAHK}{\delta}\right )}{\max\{1,N^k_h(s,a)\}}} \nonumber\\
& =0, \label{ineq:inter}
\end{align}
where inequality (1) follows from the assumption $\hat{V}^{k}_{h+1}(s')\geq V^*_{h+1}(s'),\forall s'\in\mathcal{S}$ and property (c) in Lemma \ref{L0}, and inequality (2) holds due to Lemma \ref{L4}. 
Recall that
\begin{align*}
&\hat{Q}_h^k(s,a)=\min\{b_h^k(s,a)+r_h(s,a)+OCE^{\phi}_{s'\sim \hat{P}^k_h(\cdot\vert s,a)}(\hat{V}^k_{h+1}(s')),H-h+1\},\\
&Q_h^*(s,a)=r_h(s,a)+OCE^{\phi}_{s'\sim P_h(\cdot\vert s,a)}(V^*_{h+1}(s')). 
\end{align*}
By \eqref{ineq:inter} and Lemma~\ref{L1}, we can immediately obtain 
% where $b_h^k(s,a)+r_h(s,a)+OCE^{\phi}_{s'\sim \hat{P}^k_h(\cdot\vert s,a)}(\hat{V}^k_{h+1}(s'))\geq r_h(s,a)+OCE^{\phi}_{s'\sim P_h(\cdot\vert s,a)}(V^*_{h+1}(s'))$ from \eqref{ineq:inter} and $H-h+1\geq r_h(s,a)+OCE^{\phi}_{s'\sim P_h(\cdot\vert s,a)}(V^*_{h+1}(s'))$ by Lemma \ref{L1} and property (b) and (c) in Lemma \ref{L0}, we have
\begin{align*}
\hat{Q}_h^k(s,a)-Q^*_h(s,a)\geq 0.
\end{align*}
Because $\hat{V}_h^k(s)=\max_{a'\in\mathcal{A}}\hat{Q}^k_h(s,a')$, we have $\hat{V}_h^k(s)\geq V_h^*(s)$. The result then follows by induction. 
% So by induction, with probability at least $1-2\delta$, for $\forall k\in [K],s\in \mathcal{S}, h\in [H]$, we have $\hat{V}_h^k(s)\geq V^*_h(s)$.
\end{proof}

We next state a concentration bound (Lemma~\ref{LP4}) for the OCE of the estimated next-state value function $\hat{V}_{h+1}^k$ under the estimated transition distribution $\hat{P}_h^k$. This is different from Lemma~\ref{L4} in that $\hat{V}_{h+1}^k$ is a random quantity depending on the data while the optimal value function is deterministic. The proof of Lemma~\ref{LP4} relies on the following well-known result on the $L^1$ concentration bound for the empirical transition probabilities (see, e.g., Lemma 17 in \cite{jaksch10a}):
\begin{lemma}\label{PROB}
For any $\delta\in (0,1)$, we have
\begin{align*}
P\Bigg(\left\Vert \hat{P}_h^k(\cdot\vert s,a)-P_h(\cdot\vert s,a)\right\Vert_1 \leq \sqrt{\frac{2S\log\left(\frac{SAHK}{\delta}\right)}{\max\{1,N_h^k(s,a)\}}}, \forall (s,a,h,k)\in \mathcal{S}\times \mathcal{A}\times [H]\times [K]\Bigg )\geq 1-\delta.
\end{align*}
\end{lemma}

\begin{lemma}\label{LP4}
For any $\delta\in (0,1)$, we have  %\gao{WH check: $\sqrt{2}$ or 2 below? double check the proof since I made changes.}
\begin{equation}
\begin{aligned}
P\Bigg( & \left\vert OCE^{u}_{s'\sim P_h(\cdot\vert s,a)}(\hat{V}_{h+1}^k(s'))-OCE^{u}_{s'\sim \hat{P}^k_h(\cdot\vert s,a)}(\hat{V}_{h+1}^k(s')) \right\vert \leq \vert u(-H+h)\vert\sqrt{\frac{2S\log\left (\frac{SAHK}{\delta}\right )}{\max\{1,N_h^k(s,a)\}}},\\
&\forall (s,a,h,k)\in \mathcal{S}\times \mathcal{A}\times [H]\times [K]\Bigg)\geq 1-\delta.
\end{aligned}
\end{equation}
\end{lemma}
\begin{proof}
With probability at least $1-\delta$, we have that for any $k\in [K],s\in \mathcal{S},a\in \mathcal{A}, h\in [H]$,
\begin{align*}
&\left|OCE^{u}_{s'\sim P_h(\cdot\vert s,a)}(\hat{V}_{h+1}^k(s'))-OCE^{u}_{s'\sim \hat{P}^k_h(\cdot\vert s,a)}(\hat{V}_{h+1}^k(s'))\right|\\
&= \left| \max_{\lambda\in [0,H-h]}\{\lambda+E_{s'\sim P_h(\cdot\vert s,a)}[u(\hat{V}_{h+1}^k(s')-\lambda)]\}-\max_{\lambda\in [0,H-h]}\{\lambda+E_{s'\sim \hat{P}^k_h(\cdot\vert s,a)} u(\hat{V}_{h+1}^k(s')-\lambda)]\}\right| \\
&\leq  \max_{\lambda\in [0,H-h]}\left| E_{s'\sim P_h(\cdot\vert s,a)}[u(\hat{V}_{h+1}^k(s')-\lambda)]-E_{s'\sim \hat{P}^k_h(\cdot\vert s,a)}[u(\hat{V}_{h+1}^k(s')-\lambda) \right|\\
&=\max_{\lambda\in [0,H-h]}\left\vert \sum_{s'\in\mathcal{S}}\left(\hat{P}_h^k(s'\vert s,a)-P_h(s'\vert s,a)\right)\cdot u(\hat{V}_{h+1}^k(s')-\lambda)\right\vert\\
&\overset{(1)}{\leq} \max_{\lambda\in [0,H-h]}\left\Vert \hat{P}_h^k(\cdot\vert s,a)-P_h(\cdot\vert s,a)\right\Vert_1 \cdot\left\Vert u(\hat{V}_{h+1}^k(\cdot)-\lambda)\right\Vert_{\infty}\\
&\overset{(2)}{\leq} \sqrt{\frac{2S\log\left(\frac{SAHK}{\delta}\right)}{\max\{1,N_h^k(s,a)\}}}\cdot \max_{\lambda\in [0,H-h]}\left\Vert u(\hat{V}_{h+1}^k(\cdot)-\lambda)\right\Vert_{\infty},
% &\overset{(1)}{\leq} 2\vert u(-H+h)\vert\sqrt{\frac{S\log\left (\frac{SAHK}{\delta}\right )}{N_h^k(s,a)}},
\end{align*}
where inequality (1) follows from H\"{o}lder's inequality and inequality (2) follows from Lemma \ref{PROB}. Because $\hat{V}_{h+1}^k(s') \in [0, H-h]$ for any $s'$ by the design of the OCE-VI algorithm and because $\lambda \in [0, H-h]$, we can immediately obtain that  
\begin{align*}
\max_{\lambda\in [0,H-h]}\left\Vert u(\hat{V}_{h+1}^k(\cdot)-\lambda)\right\Vert_{\infty} \le \vert u(-H+h)\vert,
\end{align*}
where we use the fact that $u$ is nondecreasing and concave. The proof is then completed. 
% Similar to Lemma \ref{L3}, for all $s'\in \mathcal{S}$, $u(V_{h+1}(s')-\lambda_{h+1})\in [u(-\lambda_{h+1}),u(H-h-\lambda_{h+1})]$, and thus $\vert u(V_{h+1}(s')-\lambda_{h+1})\vert\leq \vert u(H-h-\lambda_{h+1})-u(-\lambda_{h+1})\vert$. By the definition of $\Vert\cdot\Vert_{\infty}$, we have $\left\Vert u(V_{h+1}(\cdot)-\lambda_{h+1})\right\Vert_{\infty}\leq \max_{s'\in \mathcal{S}}\vert u(V_{h+1}(s')-\lambda_{h+1})\vert\leq \vert u(H-h-\lambda_{h+1})-u(-\lambda_{h+1})\vert$. Similar to Lemma \ref{L4}, we can derive that 
% \begin{align*}
% &\sup_{V_{h+1}\in\mathcal{V}_{h+1}}\max_{\lambda_{h+1}\in [0,H-h]}\left\Vert u(V_{h+1}(\cdot)-\lambda_{h+1})\right\Vert_{\infty}\\
% & \quad \leq \max_{\lambda_{h+1}\in [0,H-h]}\vert u(H-h-\lambda_{h+1})-u(-\lambda_{h+1})\vert\\
% & \quad \leq \vert u(-H+h)\vert,
% \end{align*}
% which completes the proof.
\end{proof}

%The bonus added after the reward in Algorithm \ref{ALG} ensures the exploration of the algorithm. 
% The following lemma shows the result of "optimism under uncertainty" in Algorithm \ref{ALG}. 

%%%%%%%%%%%%%%%%%%%%%%%%%%%%%%%%%

%%%%%%%%%%%%%%%%%%%%%%%%%%%%Lemma9 and measure change
In the next lemma, we will bound the following difference
\begin{equation}\label{EQ16}
\begin{aligned}
&OCE^{u}_{s'\sim P_h(\cdot\vert s,a)}\left (\hat{V}_{h+1}^k(s')\right )-OCE^{u}_{s'\sim P_h(\cdot\vert s,a)}\left (V_{h+1}^{\pi^k}(s')\right )
\end{aligned}
\end{equation}
 for any $(s,a,h,k)\in \mathcal{S}\times \mathcal{A}\times [H]\times [K]$, which is the key step in the recursion of the regret analysis. 
 
 % Due to the nonlinearity incorporated in OCE, we have to carry out measure change on the physical probability measure $P_h(\cdot\vert s,a)$ so as to get a linear bound. 
 
We first introduce some notations. 
Pick any $\lambda_{h+1}^k \in [0,H-h]$ such that
\begin{align}\label{eq:opt-lambda2}
   \lambda_{h+1}^k \in {\arg \max}_{\lambda\in [0,H-h]}\{\lambda+E_{s'\sim P_h(\cdot\vert s,a)}[u(V_{h+1}^{\pi^k}(s')-\lambda)]\}.
\end{align}
% Note that the optimal solutions of the above optimization problem exist but is not unique. The existence of the solutions of the optimization problem is guaranteed by the concavity of the utility function $u$ and the nonnegativity of the probability measure $P_h(\cdot\vert s,a)$. The solution is not always unique, because the utility function $u$ is not strictly concave. 
By the first order optimality condition of the above optimization problem and the fact that the state space $\mathcal{S}$ is finite, we have 
\begin{align}\label{S222}
&1\in E_{s'\sim P_{h}(\cdot\vert s,a)}[\partial u(V_{h+1}^{\pi^k}(s')-\lambda_{h+1}^k)].
\end{align}
Thus, we can find $\Lambda^k_{h+1}(s')\in\partial u(V_{h+1}^{\pi^k}(s')-\lambda_{h+1}^k), s'\in\mathcal{S}$ such that $E_{s'\sim P_h(\cdot\vert s,a)}\left[\Lambda_{h+1}^k(s')\right]=1$. In addition, because the utility function $u$ is nondecreasing, we have $\Lambda_{h+1}^k(s')\geq 0,\forall s'\in\mathcal{S}$.
Now we can define the following new probability measure $B_h(\cdot\vert s,a)$: for any $(s,a,s',h,k)\in \mathcal{S}\times \mathcal{A}\times\mathcal{S}\times [H]\times [K]$, define
\begin{equation}\label{EQ17}
\begin{aligned}
B_{h}(s'\vert s,a) :=P_{h}(s'\vert s,a)\Lambda^k_{h+1}(s'),
\end{aligned}
\end{equation}
where $\sum_{s'\in\mathcal{S}}B_h(s'\vert s,a)=1$ because $E_{s'\sim P_h(\cdot\vert s,a)}\left[\Lambda_{h+1}^k(s')\right]=1$. Now we can state the following important result.

% Note that $\Lambda^k_{h+1}(s_{h+1}^k)$ is a function of the random state $s^k_{h+1}\in\mathcal{S}$ at time step $h+1$ in episode $k$ and it's also a random variable. For notational simplicity, we denote $\Lambda^k_{h+1}(s_{h+1}^k)$ as $\Lambda$ and $B_h(s_{h+1}^k\vert s_h^k,a_h^k)=P_h(s_{h+1}^k\vert s_h^k,a_h^k)\Lambda_{h+1}^k(s_{h+1}^k)$ as $B_h(\cdot\vert s_h^k,a_h^k)=P_h(\cdot\vert s_h^k,a_h^k)\Lambda$ for any $(h,k)\in [H]\times [K]$ when there is no ambiguity. The above results coincide with the conditions of Radon-Nikodym Theorem in the page 98 of \cite{williams1991probability}. 

\begin{lemma}\label{L6}
For any $(h,k)\in [H]\times [K]$ and functions $\hat{V}_{h+1}^k,V_{h+1}^{\pi^k},V_{h+1}^{*}:\mathcal{S}\to [0,H-h]$, we have
\begin{equation}\label{EQ18}
\begin{aligned}
&OCE^{u}_{s_{h+1}^k\sim P_h(\cdot\vert s_h^k,a_h^k)}\left (\hat{V}_{h+1}^k(s_{h+1}^k)\right )-OCE^{u}_{s_{h+1}^k\sim P_h(\cdot\vert s_h^k,a_h^k)}\left (V_{h+1}^{\pi^k}(s_{h+1}^k)\right )\\
&\leq E_{s_{h+1}^k\sim B_h(\cdot\vert s_h^k,a_h^k)}\left [\hat{V}_{h+1}^k(s_{h+1}^k)-V_{h+1}^{\pi^k}(s_{h+1}^k)\right ],
\end{aligned}
\end{equation}
where $B_h(\cdot\vert s_h^k,a_h^k)$ is the new probability measure given in \eqref{EQ17}.
\end{lemma}

\begin{proof}
Pick any $\mu_{h+1}^k\in [0,H-h]$ such that
$$\mu_{h+1}^k \in {\arg \max}_{\lambda\in [0,H-h]}\{\lambda+E_{s_{h+1}^k\sim P_h(\cdot\vert s_h^k,a_h^k)}[u(\hat{V}_{h+1}^{k}(s_{h+1}^k)-\lambda)]\},$$ and recall $\lambda_{h+1}^k\in [0,H-h]$ given in \eqref{eq:opt-lambda2}. We can then compute
%$$\max_{\lambda\in [0,H-h]}\{\lambda+E_{s_{h+1}^k\sim P_h(\cdot\vert s_h^k,a_h^k)}[u(V_{h+1}^{\pi^k}(s_{h+1}^k)-\lambda)]\}.$$ By the first order optimality condition for the above optimization problems, we can get
% \begin{align*}
% & 0\in 1-E_{s_{h+1}^k\sim P_{h}(\cdot\vert s_h^k,a_h^k)}[\partial u(\hat{V}_{h+1}^k(s_{h+1}^k)-\lambda_1)],\\
% & 0\in 1-E_{s_{h+1}^k\sim P_{h}(\cdot\vert s_h^k,a_h^k)}[\partial u(V_{h+1}^{\pi^k}(s_{h+1}^k)-\lambda_2)].
% \end{align*}
% So under probability measure $P_{h}(\cdot\vert s_h^k,a_h^k)$, there exists a nonnegative $\Lambda\in \partial u(V_{h+1}^{\pi^k}(s_{h+1}^k)-\lambda_2)$ such that $E_{s_{h+1}^k\sim P_{h}(\cdot\vert s_h^k,a_h^k)}[\Lambda]=1$. The new probability measure is $B_h(\cdot\vert s_h^k,a_h^k)=P_h(\cdot\vert s_h^k,a_h^k)\Lambda$.
\begin{align*}
&OCE^{u}_{s_{h+1}^k\sim P_h(\cdot\vert s_h^k,a_h^k)}\left (\hat{V}_{h+1}^k(s_{h+1}^k)\right )-OCE^{u}_{s_{h+1}^k\sim P_h(\cdot\vert s_h^k,a_h^k)}\left (V_{h+1}^{\pi^k}(s_{h+1}^k)\right )\\
&\overset{(1)}{=} \max_{\lambda\in [0,H-h]}\left\{\lambda+E_{s_{h+1}^k\sim P_{h}(\cdot\vert s_h^k,a_h^k)}[u(\hat{V}_{h+1}^k(s_{h+1}^k)-\lambda)] \right\} \\
& \quad \quad -\max_{\lambda\in [0,H-h]}\left\{\lambda+E_{s_{h+1}^k\sim P_{h}(\cdot\vert s_h^k,a_h^k)}[u(V_{h+1}^{\pi^k}(s_{h+1}^k)-\lambda)] \right\}\\
&= \mu_{h+1}^k+E_{s_{h+1}^k\sim P_{h}(\cdot\vert s_h^k,a_h^k)}[u(\hat{V}_{h+1}^k(s_{h+1}^k)-\mu_{h+1}^k)]-\lambda_{h+1}^k-E_{s_{h+1}^k\sim P_{h}(\cdot\vert s_h^k,a_h^k)}[u(V_{h+1}^{\pi^k}(s_{h+1}^k)-\lambda_{h+1}^k)]\\
&\overset{(2)}{\leq}  \mu_{h+1}^k-\lambda_{h+1}^k+E_{s_{h+1}^k\sim P_{h}(\cdot\vert s_h^k,a_h^k)}\left [\Lambda^k_{h+1}(s_{h+1}^k)\cdot(\hat{V}_{h+1}^k(s_{h+1}^k)-V_{h+1}^{\pi^k}(s_{h+1}^k)-(\mu_{h+1}^k-\lambda_{h+1}^k))\right ]\\
&= \left (1-E_{s_{h+1}^k\sim P_{h}(\cdot\vert s_h^k,a_h^k)}\left [\Lambda^k_{h+1}(s_{h+1}^k) \right ]\right )(\mu_{h+1}^k-\lambda_{h+1}^k)\\
& \quad \quad +E_{s_{h+1}^k\sim P_{h}(\cdot\vert s_h^k,a_h^k)}\left [\Lambda^k_{h+1}(s_{h+1}^k) \cdot (\hat{V}_{h+1}^k(s_{h+1}^k)-V_{h+1}^{\pi^k}(s_{h+1}^k))\right ]\\
&\overset{(3)}{=}E_{s_{h+1}^k\sim B_h(\cdot\vert s_h^k,a_h^k)}[\hat{V}_{h+1}^k(s_{h+1}^k)-V_{h+1}^{\pi^k}(s_{h+1}^k)],
\end{align*}
where equality (1) holds due to Lemma \ref{L2}, inequality (2) holds due to the fact that $u(y)\leq u(x)+z(y-x)$ for any $x,y\in [-H+h,H-h],z\in \partial u(x)$ when $u(x)$ is a concave function, and equality (3) follows from \eqref{EQ17} and the fact that $E_{s_{h+1}^k\sim P_h(\cdot\vert s_h^k,a_h^k)}[\Lambda^k_{h+1}(s_{h+1}^k)]=1$. The proof is therefore completed. 
\end{proof}

% We can explain the meaning behind the measure change and its connection with the trajectory of the MDP. The new probability measure $B_h(\cdot\vert s,a)$ reflects the agent's attitude towards gains and losses. If $\hat{V}_{h+1}^{\pi^k}(s_{h+1}^k)-\lambda_2\leq 0$, i.e. the agent suffers a loss, by the definition of $u$, we know that for a nonnegative $\Lambda\in \partial u(V_{h+1}^{\pi^k}(s_{h+1}^k)-\lambda_2)$ and $E_{s'\sim P_h(\cdot\vert s,a)}[\Lambda]=1$, we have $B_{h}(s_{h+1}^k\vert s_h^k,a_h^k)\geq P_{h}(s_{h+1}^k\vert s_h^k,a_h^k)$. The agent will punish the loss more so as to reduce loss when he makes decisions. If $\hat{V}_{h+1}^{\pi^k}(s_{h+1}^k)-\lambda_2\geq 0$, it can be explained similarly.

In the next lemma, we bound the term $\hat{V}_1^k(s_1^k)-V_1^{\pi^k}(s_1^k)$ by using a recursive procedure. Lemma \ref{L7} below is an extension of the so-called simulation lemma in the risk-neutral RL (see, e.g, \cite{agarwal2019reinforcement}) to our risk-sensitive RL setting. The key to overcome the difficulty in the recursion setting due to the nonlinearity of the OCE is Lemma \ref{L6}. To facilitate the presentation, we first introduce the following notations.

For any $k\in [K],h\in [H]$, let $w_{hk}(s_h^k,a_h^k)$ be the state-action distribution induced by $\pi^k$ at time step $h$ starting from $s_1^k$, i.e., the probability of $\pi^k$ visiting $(s_h^k,a_h^k)$ at time step $h$ starting from $s_1^k$.
Mathematically, the formula of $w_{hk}(s_h^k,a_h^k)$ is given by
\begin{equation}\label{SSS4}
w_{hk}(s_h^k,a_h^k)=\left\{
	\begin{aligned}
	&1, \quad &h=1,\\
	&P_{1}(s_{2}^k\vert s_1^k,a_1^k), \quad &h=2,\\
	&\sum_{s_{2}^k\in \mathcal{S}}\cdots \sum_{s_{h-1}^k\in \mathcal{S}}P_{1}(s_{2}^k\vert s_1^k,a_1^k)\cdots P_{h-1}(s_{h}^k\vert s_{h-1}^k,a_{h-1}^k), \quad &h\geq 3.
	\end{aligned}
	\right
	.
\end{equation}
% which satisfies $\sum_{(s,a)\in \mathcal{S}\times\mathcal{A}}w_{hk}(s,a)=1$ and corresponds to $d_h^{\pi^k}$ in the page 77 of \cite{agarwal2019reinforcement}.
Similarly, let $w_{hk}^{B}(s_h^k,a_h^k)$ be the probability of $\pi^k$ visiting $(s_h^k,a_h^k)$ at step $h$ starting from $s_1^k$ under probability measures $B_i(\cdot\vert s_i^k,a_i^k),i=1,\cdots,h$. 
%Note that if $w_{hk}(s_h^k,a_h^k)=0$, then $w_{hk}^{B}(s_h^k,a_h^k)=0$. For any $k\in [K],h\in [H]$,
The explicit formula of $w_{hk}^{B}(s_h^k,a_h^k)$ is given by
\begin{equation}\label{S4}
w_{hk}^{B}(s_h^k,a_h^k)=\left\{
	\begin{aligned}
	&1, \quad &h=1,\\
	&B_{1}(s_{2}^k\vert s_1^k,a_1^k), \quad &h=2,\\
	&\sum_{s_{2}^k\in \mathcal{S}}\cdots \sum_{s_{h-1}^k\in \mathcal{S}}B_{1}(s_{2}^k\vert s_1^k,a_1^k)\cdots B_{h-1}(s_{h}^k\vert s_{h-1}^k,a_{h-1}^k), \quad &h\geq 3.\\
	\end{aligned}
	\right
	.
\end{equation}
%which satisfies $\sum_{(s,a)\in \mathcal{S}\times\mathcal{A}}w_{hk}^{B}(s,a)=1$.
% We leave out the probability of choosing actions $\pi(a_h^k\vert s_h^k)$ for any $k\in [K],h\in [H]$ in \eqref{SSS4} and \eqref{S4}, because the policies are deterministic once they are chosen. To better understand $w_{hk}^B(s_h^k,a_h^k)$, we write down its formula with $P_h(s_h^k\vert s_{h-1}^k,a_{h-1}^k)$ and $\Lambda_h^k(s_h^k)$ for any $k\in [K],h\in [H]$, where $\Lambda_h^k(s_h^k)$ is defined in \eqref{S222}
Equivalently, by \eqref{EQ17} we have
\begin{equation}\label{SKEY}
\begin{aligned}
w_{hk}^{B}(s_h^k,a_h^k)
% &=\left\{
% 	\begin{aligned}
% 	&1 \quad &h=1\\
% 	&P_{1}(s_{2}^k\vert s_1^k,a_1^k)\Lambda_2^k(s_2^k) \quad &h=2\\
% 	&\sum_{s_{2}^k\in \mathcal{S}}\cdots \sum_{s_{h-1}^k\in \mathcal{S}}P_{1}(s_{2}^k\vert s_1^k,a_1^k)\Lambda_2^k(s_2^k)\cdots P_{h-1}(s_{h}^k\vert s_{h-1}^k,a_{h-1}^k)\Lambda_h^k(s_h^k) \quad &h\geq 3\\
% 	\end{aligned}
% 	\right
% 	.\\
	&=\left\{
	\begin{aligned}
	&1, \quad &h=1,\\
	&P_{1}(s_{2}^k\vert s_1^k,a_1^k)\Lambda_2^k(s_2^k), \quad &h=2,\\
	&\sum_{s_{2}^k\in \mathcal{S}}\cdots \sum_{s_{h-1}^k\in \mathcal{S}}P_{1}(s_{2}^k\vert s_1^k,a_1^k)\cdots P_{h-1}(s_{h}^k\vert s_{h-1}^k,a_{h-1}^k)\Lambda_2^k(s_2^k)\cdots \Lambda_h^k(s_h^k), \quad &h\geq 3.\\
	\end{aligned}
	\right
	.
\end{aligned}
\end{equation}
Finally, given $(s_1^k ,a_1^k)$, we define
\begin{equation}\label{LS10}
E_{(s_h^k,a_h^k)\sim w_{hk}^{B}}[\cdot] :=\left\{
	\begin{aligned}
	&1, \quad &h=1, \\
	&E_{s_2^k\sim B_1(\cdot\vert s_1^k,a_1^k)}\left [E_{s_3^k\sim B_2(\cdot\vert s_2^k,a_2^k)}\left [\cdots E_{s_h^k\sim B_{h-1}(\cdot\vert s_{h-1}^k,a_{h-1}^k)}\left [\cdot\right ]\right ]\right ], \quad &h\geq 2. \\
	\end{aligned}
	\right
	.
\end{equation}
%which is related to the formula \eqref{S4} of $w_{hk}^{B}(s_h^k,a_h^k)$.

\begin{lemma}\label{L7}
For each episode $k\in [K]$, we have
\begin{align} \label{eq:sim-lemma}
&\hat{V}_1^k(s_1^k)-V_1^{\pi^k}(s_1^k)\\
&\leq \sum_{h=1}^H E_{(s_h^k,a_h^k)\sim w_{hk}^{B}}\Big [b^k_h(s_h^k,a_h^k)\\
&\quad \quad \quad +OCE^u_{s_{h+1}^k\sim \hat{P}^k_h(\cdot\vert s_h^k,a_h^k)}(\hat{V}_{h+1}^k(s_{h+1}^k))-OCE^u_{s_{h+1}^k\sim P_h(\cdot\vert s_h^k,a_h^k)}(\hat{V}_{h+1}^k(s_{h+1}^k))\Big ]. \nonumber
\end{align}
\end{lemma}

\begin{proof}
For any $k\in [K]$, let $a^k_h=\mathop{\arg\max}_{a\in\mathcal{A}}\hat{Q}^k_h(s_h^k,a),\ h\in [H]$. Then, we can compute 
\begin{align*}
&\hat{V}_1^k(s_1^k)-V_1^{\pi^k}(s_1^k)\\
& \overset{(1)}{\leq} \hat{Q}^k_1(s_1^k,a_1^k)-Q_1^{\pi^k}(s_1^k,a_1^k)\\
%& = \min\left\{b_1^k(s_1^k,a_1^k)+OCE^u_{s_{2}^k\sim \hat{P}_1^k(\cdot\vert s_1^k,a_1^k)}(\hat{V}_2^k(s_2^k))-OCE^u_{s_{2}^k\sim P_1(\cdot\vert s_1^k,a_1^k)}(V_2^{\pi^k}(s_2^k)),H\right\}\\
&\leq  b_1^k(s_1^k,a_1^k)+OCE^u_{s_{2}^k\sim \hat{P}_1^k(\cdot\vert s_1^k,a_1^k)}(\hat{V}_2^k(s_2^k))-OCE^u_{s_{2}^k\sim P_1(\cdot\vert s_1^k,a_1^k)}(V_2^{\pi^k}(s_2^k))\\
&=b_1^k(s_1^k,a_1^k)+OCE^u_{s_{2}^k\sim \hat{P}_1^k(\cdot\vert s_1^k,a_1^k)}(\hat{V}_2^k(s_2^k))-OCE^u_{s_{2}^k\sim P_1(\cdot\vert s_1^k,a_1^k)}(\hat{V}_2^{k}(s_2^k))\\
&+OCE^u_{s_{2}^k\sim P_1(\cdot\vert s_1^k,a_1^k)}(\hat{V}_2^{k}(s_2^k))-OCE^u_{s_{2}^k\sim P_1(\cdot\vert s_1^k,a_1^k)}(V_2^{\pi^k}(s_2^k))\\
&\overset{(2)}{\leq}  b_1^k(s_1^k,a_1^k)+OCE^u_{s_{2}^k\sim \hat{P}_1^k(\cdot\vert s_1^k,a_1^k)}(\hat{V}_2^k(s_2^k))-OCE^u_{s_{2}^k\sim P_1(\cdot\vert s_1^k,a_1^k)}(\hat{V}_2^{k}(s_2^k))\\
&+E_{s_2^k\sim B_1(\cdot\vert s_1^k,a_1^k)}\left [\hat{V}^k_2(s_2^k)-V_2^{\pi^k}(s_2^k)\right ]\\
&\leq  b_1^k(s_1^k,a_1^k)+OCE^u_{s_{2}^k\sim \hat{P}_1^k(\cdot\vert s_1^k,a_1^k)}(\hat{V}_2^k(s_2^k))-OCE^u_{s_{2}^k\sim P_1(\cdot\vert s_1^k,a_1^k)}(\hat{V}_2^{k}(s_2^k))\\
&+E_{s_2^k\sim B_1(\cdot\vert s_1^k,a_1^k)}\Big [b_2^k(s_2^k,a_2^k)+OCE^u_{s_{3}^k\sim \hat{P}_2^k(\cdot\vert s_2^k,a_2^k)}(\hat{V}_3^k(s_3^k))-OCE^u_{s_{3}^k\sim P_2(\cdot\vert s_2^k,a_2^k)}(\hat{V}_3^{k}(s_3^k)) \\
& \hspace{3cm} +E_{s_3^k\sim B_2(\cdot\vert s_2^k,a_2^k)}\left [\hat{V}^k_3(s_3^k)-V_3^{\pi^k}(s_3^k)\right ] \Big ],
\end{align*}
where inequality (1) holds because $\hat{V}_1^k(s_1^k)= \max_{a\in \mathcal{A}} \hat{Q}^k_1(s_1^k,a)=\hat{Q}^k_1(s_1^k,a_1^k)$ and inequality (2) holds due to Lemma \ref{L6}. Applying the above procedure recursively and using the fact that $\hat{V}_{H+1}^k(s)=V_{H+1}^*(s)=0$ for any $s\in\mathcal{S}$, we immediately obtain \eqref{eq:sim-lemma}. 
% \begin{align*}
% &\hat{V}_1^k(s_1^k)-V_1^{\pi^k}(s_1^k)\\
% &\leq  \sum_{h=1}^H E_{(s_h^k,a_h^k)\sim w_{hk}^{B}}\left [b^k_h(s_h^k,a_h^k)+OCE^u_{s_{h+1}^k\sim \hat{P}^k_h(\cdot\vert s_h^k,a_h^k)}(\hat{V}_{h+1}^k(s_{h+1}^k))-OCE^u_{s_{h+1}^k\sim P_h(\cdot\vert s_h^k,a_h^k)}(\hat{V}_{h+1}^k(s_{h+1}^k))\right ].
% \end{align*}
% The proof is complete. 
\end{proof}

%%%%%%%%%%%%%%%%%%%%%%%%%%%%%%%%%%
From Lemma~\ref{L7}, it is clear that we need to bound the sum of bonuses in order to bound the regret. We present such a bound in Lemma \ref{L13}. To this end, we first state Lemma \ref{LP13}, which is adapted from a well-known result heavily used in the risk-neutral setting (see page 24-25 of \cite{azar2017minimax} or page 21 of \cite{jin2018q}). Lemma \ref{L13} is a nontrivial extension of Lemma \ref{LP13} due to the new probability measure $w_{hk}^{B}$ invovled in the expectation. 

% Define $\tau^k=\{s_h^k,a_h^k\}_{h=1}^H,k=1,\cdots,K$ as the sequence of trajectories generated from $\pi^k,k=1,\cdots,K$, and recall the History defined in \eqref{HIST}. We show the following lemma to prepare for the proof of Lemma \ref{L13}. Lemma \ref{LP13} is adapted from a well-known result, which has been heavily used in the risk-neutral setting, see page 21 of \cite{jin2018q} for the cases dependent on time step $h$, page 24-25 of \cite{azar2017minimax} for the cases independent of time step $h$. \gao{Gao; paragraph to be polish later}

\begin{lemma}\label{LP13}
Consider arbitrary $K$ sequences of trajectories $\tau^k=\{s_h^k,a_h^k\}_{h=1}^H$ for $k=1,\cdots,K$, we have
\begin{align*}
\sum_{k=1}^K\frac{1}{\max\{1,N_h^k(s_h^k,a_h^k)\}}\leq  SA\log(3K).
\end{align*}
\end{lemma}

\begin{lemma}\label{L13}  
% Consider the state-action pairs $(s_h^k,a_h^k)$ at step $h$ in episode $k$ satisfying $N_h^k(s_h^k,a_h^k)\geq 1$. Note that if the process never transits to some state-action pair $(s_h^k,a_h^k)$ at step $h$ until episode $k$, then $N_h^k(s_h^k,a_h^k)=0$, $w_{hk}(s_h^k,a_h^k)=0$, and thus $w_{hk}^{B}(s_h^k,a_h^k)=0$. So 
%Consider arbitrary $K$ sequences of trajectories $\tau^k=\{s_h^k,a_h^k\}_{h=1}^H$ for $k=1,\cdots,K$, 
We have
\begin{align}\label{EQ25}
& E\left [\sum_{k=1}^K\sum_{h=1}^H E_{(s_h^k,a_h^k)\sim w_{hk}^{B}}\left [\frac{\vert u(-H+h)\vert}{\sqrt{\max\{1,N_h^k(s_h^k,a_h^k)\}}}\right ]\right ]\nonumber\\
&\leq \sum_{h=1}^H \vert u(-H+h)\vert  \sqrt{\prod\limits_{i=1}^{h-1} u_{-}'(-H+i)SAK\log(3K)},
\end{align}
where $E_{(s_h^k,a_h^k)\sim w_{hk}^{B}}[\cdot]$ given in \eqref{LS10} is taken over $(s_h^k,a_h^k)$ conditional on $(s_1^k,a_1^k)$ and $u_{-}'(\cdot)$ is the left derivative of $u$.
\end{lemma}
\begin{proof}
By \eqref{SKEY} and \eqref{LS10}, we have %for any $2\leq h\leq H$,
\begin{align}\label{SS28}
&E_{(s_h^k,a_h^k)\sim w_{hk}^{B}}\left [\frac{\vert u(-H+h)\vert}{\sqrt{\max\{1,N_h^k(s_h^k,a_h^k)\}}}\right ]\nonumber\\
&=\sum_{s_{2}^k\in\mathcal{S}}\cdots \sum_{s_{h}^k\in \mathcal{S}}P_{1}(s_{2}^k\vert s_1^k,a_1^k)\cdots P_{h-1}(s_{h}^k\vert s_{h-1}^k,a_{h-1}^k)\Lambda_2^k(s_2^k)\cdots \Lambda_h^k(s_h^k)\frac{\vert u(-H+h)\vert}{\sqrt{\max\{1,N_h^k(s_h^k,a_h^k)\}}}.
\end{align}

This implies
\begin{align*}
&E\left [ E_{(s_h^k,a_h^k)\sim w_{hk}^{B}}\left [\frac{\vert u(-H+h)\vert}{\sqrt{\max\{1,N_h^k(s_h^k,a_h^k)\}}}\right ]\right]\\
% &\overset{(1)}{=}  \frac{\vert u(-H+1)\vert}{\sqrt{N_1^k(s_1^k,a_1^k)}}+\sum_{h=2}^H E\left [\Lambda_2^k(s_2^k)\cdots\Lambda_h^k(s_h^k)\frac{\vert u(-H+h)\vert}{\sqrt{N_h^k(s_h^k,a_h^k)}}\Bigg \vert s_1^k,a_1^k\right ]\\
&= E\left[ E\left [\Lambda_2^k(s_2^k)\cdots\Lambda_h^k(s_h^k)\frac{\vert u(-H+h)\vert}{\sqrt{\max\{1,N_h^k(s_h^k,a_h^k)\}}}\Bigg \vert s_1^k,a_1^k\right ]\right]\\
&= E\left [\Lambda_2^k(s_2^k)\cdots\Lambda_h^k(s_h^k)\frac{\vert u(-H+h)\vert}{\sqrt{\max\{1,N_h^k(s_h^k,a_h^k)\}}}\right ].
\end{align*}

Then, we have
\begin{align*}
&E\left [\sum_{k=1}^K\sum_{h=1}^H E_{(s_h^k,a_h^k)\sim w_{hk}^{B}}\left [\frac{\vert u(-H+h)\vert}{\sqrt{\max\{1,N_h^k(s_h^k,a_h^k)\}}}\right ]\right ]\\
&=\sum_{h=1}^H\sum_{k=1}^K E\left [ E_{(s_h^k,a_h^k)\sim w_{hk}^{B}}\left [\frac{\vert u(-H+h)\vert}{\sqrt{\max\{1,N_h^k(s_h^k,a_h^k)\}}}\right ]\right ]\\
&=\sum_{h=1}^H\vert u(-H+h)\vert\sum_{k=1}^K E\left [\Lambda_2^k(s_2^k)\cdots\Lambda_h^k(s_h^k)\frac{1}{\sqrt{\max\{1,N_h^k(s_h^k,a_h^k)\}}}\right ]\\
&\overset{(1)}{\leq} \sum_{h=1}^H\vert u(-H+h)\vert\sum_{k=1}^K \sqrt{E\left [\Lambda_2^k(s_2^k)\cdots\Lambda_h^k(s_h^k)\right ]^2}\cdot \sqrt{E\left[\frac{1}{\max\{1,N_h^k(s_h^k,a_h^k)\}}\right]} \\
&\overset{(2)}{\leq} \sum_{h=1}^H\vert u(-H+h)\vert\sqrt{\sum_{k=1}^K E\left [\Lambda_2^k(s_2^k)\cdots\Lambda_h^k(s_h^k)\right ]^2}\cdot \sqrt{\sum_{k=1}^K E\left[\frac{1}{\max\{1,N_h^k(s_h^k,a_h^k)\}}\right]}
\end{align*}
where the inequalities (1) and (2) follow from Cauchy–Schwarz inequality.

Recall that $\Lambda^k_{h+1}(s')\in\partial u(V_{h+1}^{\pi^k}(s')-\lambda_{h+1}^k), s'\in\mathcal{S}$ and it satisfies $E_{s'\sim P_h(\cdot\vert s,a)}\left[\Lambda_{h+1}^k(s')\right]=1$, where $\lambda_{h+1}^k$ is defined in \eqref{eq:opt-lambda2}. By Lemma \ref{L1} and Lemma \ref{L2} and the concavity of the utility function $u$, we have $0\leq\Lambda_{h+1}^k\leq u_{-}'(-H+h)$. Because $$E_{(s_h^k,a_h^k)\sim w_{hk}^{B}}[1]=\sum_{s_{2}^k\in\mathcal{S}}\cdots \sum_{s_{h}^k\in \mathcal{S}}P_{1}(s_{2}^k\vert s_1^k,a_1^k)\cdots P_{h-1}(s_{h}^k\vert s_{h-1}^k,a_{h-1}^k)\Lambda_2^k(s_2^k)\cdots \Lambda_h^k(s_h^k)=1,$$ taking the expectation on both sides, we have $E\left[\Lambda_2^k(s_2^k)\cdots\Lambda_h^k(s_h^k)\right ]=1$. Then, we have
\begin{align*}
&\sum_{k=1}^K E\left [\Lambda_2^k(s_2^k)\cdots\Lambda_h^k(s_h^k)\right ]^2\\
&\leq \sum_{k=1}^K E\left [\Lambda_2^k(s_2^k)\cdots\Lambda_h^k(s_h^k)\right ]\cdot \prod\limits_{i=1}^{h-1} u_{-}'(-H+i)\\
&=K\cdot \prod\limits_{i=1}^{h-1} u_{-}'(-H+i).
\end{align*}
Combining this inequality and Lemma \ref{LP13}, we have
\begin{align*}
&E\left [\sum_{k=1}^K\sum_{h=1}^H E_{(s_h^k,a_h^k)\sim w_{hk}^{B}}\left [\frac{\vert u(-H+h)\vert}{\sqrt{\max\{1,N_h^k(s_h^k,a_h^k)\}}}\right ]\right ]\\
&\leq \sum_{h=1}^H \vert u(-H+h)\vert \sqrt{\prod\limits_{i=1}^{h-1} u_{-}'(-H+i)SAK\log(3K)},
\end{align*}
which completes the proof.
% since $$\sum_{s_{h+1}^l\in\mathcal{S}}B_h(s_{h+1}^l\vert s_h^l,a_h^l)=\sum_{s_{h+1}^l\in\mathcal{S}}P_h(s_{h+1}^l\vert s_h^l,a_h^l)\Lambda_{h+1}^l(s_{h+1}^l)=1$$ for any $h\in [H-1],l\in [K]$, in detail 
% thus similar to inequality (2) in the previous part, we put \eqref{SS30} into $$E\left [\Lambda_2^k(s_2^k)\cdots\Lambda_h^k(s_h^k)\cdots\Lambda_H^k(s_H^k)\sum_{h=1}^H\frac{\vert u(-H+h)\vert}{\sqrt{N_h^k(s_h^k,a_h^k)}}\right ]$$ for any $k\in [K]$ to get equality (4) 
\end{proof}

\subsection{Proof of Theorem \ref{THM1}}\label{PFTHM1}

Now we are ready to prove Theorem~\ref{THM1}. Recall $\mathcal{G}_1$ in \eqref{G1} and we define
\begin{align*}
\mathcal{G}_2=\Bigg\{&\left\vert OCE^{u}_{s'\sim P_h(\cdot\vert s,a)}(\hat{V}_{h+1}^k(s'))-OCE^{u}_{s'\sim \hat{P}^k_h(\cdot\vert s,a)}(\hat{V}_{h+1}^k(s'))\right\vert\\ &\leq \vert u(-H+h)\vert\sqrt{\frac{2S\log\left (\frac{SAHK}{\delta}\right )}{\max\{1,N_h^k(s_h^k,a_h^k)\}}},\forall (s,a,h,k)\in \mathcal{S}\times \mathcal{A}\times [H]\times [K]\Bigg\}.
\end{align*}
We also define $\mathcal{G}=\mathcal{G}_1\cap \mathcal{G}_2$.
From Lemmas \ref{L4} and \ref{LP4}, we know that $P(\mathcal{G}_1)\geq 1-\delta$ and $P(\mathcal{G}_2)\geq 1-\delta$, which implies that $P(\mathcal{G})\geq 1-2\delta$.  %Therefore, the event $\mathcal{G}$ occurs with probability at least $1-4\delta$.

\begin{proof}[Proof of Theorem \ref{THM1}]
For any $k\in [K]$, let $a_h^k=\mathop{\arg\max}_{a\in\mathcal{A}}\hat{Q}^k_h(s_1^k,a),\ h\in [H]$. Then, when the event $\mathcal{G}$ holds, we can compute
\begin{align}
&V^*_1(s_1^k)-V_1^{\pi^k}(s_1^k) \nonumber \\
&\overset{(1)}{\leq} \hat{V}_1^k(s_1^k)-V_1^{\pi^k}(s_1^k) \nonumber \\
&\overset{(2)}{\leq} \sum_{h=1}^H E_{(s_h^k,a_h^k)\sim w_{hk}^{B}}\Big [b^k_h(s_h^k,a_h^k)\nonumber\\
&\quad \quad \quad +OCE^u_{s_{h+1}^k\sim \hat{P}^k_h(\cdot\vert s_h^k,a_h^k)}(\hat{V}_{h+1}^k(s_{h+1}^k))-OCE^u_{s_{h+1}^k\sim P_h(\cdot\vert s_h^k,a_h^k)}(\hat{V}_{h+1}^k(s_{h+1}^k))\Big ] \nonumber \\
&\overset{(3)}{\leq}  \sum_{h=1}^H E_{(s_h^k,a_h^k)\sim w_{hk}^{B}}\left [2\sqrt{2} \vert u(-H+h)\vert\sqrt{\frac{S\log\left (\frac{SAHK}{\delta}\right )}{\max\{1,N_h^k(s_h^k,a_h^k)\}}} \right ],  \label{eq:highPbound}
\end{align}
where inequality (1) follows from Lemma \ref{L5}, inequality (2) holds due to Lemma \ref{L7}, inequality (3) holds due to Lemma \ref{LP4} and the fact that $b^k_h(s_h^k,a_h^k)\leq \vert u(-H+h)\vert\sqrt{\frac{2S\log\left (\frac{SAHK}{\delta}\right )}{\max\{1,N_h^k(s_h^k,a_h^k)\}}}$.

We can write the expected regret as follows: 
\begin{align*}
&Regret(\mathcal{M},\textbf{OCE-VI},K)\\
&= E\left[\sum_{k=1}^K\left(V_1^*(s_1^k)-V_1^{\pi^k}(s_1^k)\right)\right]\\
&= E\left[1_{\mathcal{G}}\cdot\sum_{k=1}^K\left(V_1^*(s_1^k)-V_1^{\pi^k}(s_1^k)\right)\right]+E\left[1_{\mathcal{G}^c}\cdot\sum_{k=1}^K\left(V_1^*(s_1^k)-V_1^{\pi^k}(s_1^k)\right)\right]\\
&\overset{(4)}{\leq}  E\left[1_{\mathcal{G}}\cdot\sum_{k=1}^K\left(V_1^*(s_1^k)-V_1^{\pi^k}(s_1^k)\right)\right]+2\delta KH,
\end{align*}
where inequality (4) holds because $P(\mathcal{G}^c)\leq 2\delta$ and $0\leq V_1^{\pi^k}(s_1^k)\leq V_1^*(s_1^k)\leq H$ by Lemma \ref{L1}. Using \eqref{eq:highPbound} and Lemma~\ref{L13}, we deduce that
\begin{align*}
&E\left[1_{\mathcal{G}}\cdot\sum_{k=1}^K\left(V_1^*(s_1^k)-V_1^{\pi^k}(s_1^k)\right)\right]\\
&\leq E\left[\sum_{k=1}^K\sum_{h=1}^H E_{(s_h^k,a_h^k)\sim w_{hk}^{B}}\left [2\sqrt{2} \vert u(-H+h)\vert\sqrt{\frac{S\log\left (\frac{SAHK}{\delta}\right )}{\max\{1,N_h^k(s_h^k,a_h^k)\}}} \right ]\right]\\
&\leq 2\sqrt{2}\sum_{h=1}^H\vert u(-H+h)\vert S\sqrt{\prod\limits_{i=1}^{h-1} u_{-}'(-H+i)AK\log(3K)\log\left (\frac{SAHK}{\delta}\right )}.
\end{align*}
Then, we have 
\begin{align*}
&Regret(\mathcal{M},\textbf{OCE-VI},K)\\
&\leq 2\sqrt{2}\sum_{h=1}^H\vert u(-H+h)\vert S\sqrt{\prod\limits_{i=1}^{h-1} u_{-}'(-H+i)AK\log(3K)\log\left (\frac{SAHK}{\delta}\right )}+2\delta KH \\
&\leq 2\sqrt{2}\sum_{h=1}^H\vert u(-H+h)\vert S\sqrt{\prod\limits_{i=1}^{h-1} u_{-}'(-H+i)AK\log(3K)\log\left (2SAH^2K^2\right )}+1,
\end{align*}
 where the last inequality follows by choosing $\delta=\frac{1}{2KH}$. The proof is then completed.  
% \begin{align*}
% &Regret(\mathcal{M},\textbf{OCE-VI},K)\\
% &\leq 20\sum_{h=1}^H\vert u(-H+h)\vert S\sqrt{AK\log\left (4SAH^2K^2\right )}+1,
% \end{align*}
% which completes the proof.
\end{proof}

%%%%%%%%%%%%%%%%%%%%%%%%%%%%%%%%%%%%%%%%%%%%%%%%%

%%%%%%%%%%%%%%%%%%%%%%%%%%%%%%%%%%%%%%%%%%%%%%

\section{Proof of Theorem  \ref{THM3}} \label{sec:THM-lowerB}

%\subsection{Proof of Theorem \ref{THM3}}
% We adapt the proof of theorem 9 in \cite{domingues2021episodic} to our risk-sensitive setting.

% for episodic RL in risk-neutral MDPs. They design a class of hard MDP instances that are based on the construction of MDPs in \cite{lattimore2020bandit} and \cite{jaksch10a}. We make necessary modifications of such MDP instances to fit into the risk-sensitive setting and prove our regret lower bound.

%Now, we begin to prove theorem \ref{THM3}.
\begin{proof} We adapt the proof of Theorem 9 in \cite{domingues2021episodic} to our risk-sensitive setting. The proof of Theorem  \ref{THM3} is long, so we divide it into a few steps. 

% Because we refer to the proofs in  \cite{domingues2021episodic}, we only provides the outline of the class of the hard MDP instances.

\begin{itemize}
\item \textbf{Step 1: Constructing the hard MDP instances.}

We first construct hard MDP instances, which are almost the same as the ones in \cite{domingues2021episodic} except one small yet important difference: the transition probabilities in \eqref{eq:p}.  

Based on assumption \ref{ASSU2}, we can construct a full $A$-ary tree of depth $d-1$ with root $\tilde{s}_{root}$, which has $S-3$ states. In this rooted tree, each node has exactly $A$ children and the total number of nodes is given by $\sum_{i=0}^{d-1}A^i=S-3$. We add three special states to the tree: a ``waiting" state $\tilde{s}_w$ where the agent starts and can choose action $\tilde{a}_w$ to stay up to a stage $\bar{H}< H-d$, a ``good" state $\tilde{s}_g$ where the agent obtains rewards, and a ``bad" state $\tilde{s}_b$ that gives no reward. Note that $\bar{H}$ is a parameter to be chosen later. Both $\tilde{s}_g$ and $\tilde{s}_b$ are absorbing states. For any state in the tree, the transitions are deterministic, the $a$-th action in a node leads to the $a$-th child of that node. The agent stays or leaves $\tilde{s}_w$ with probability
\begin{equation*}
P_h(\tilde{s}_w\vert \tilde{s}_w,a):=1\{a=\tilde{a}_w,h\leq \bar{H}\},\ P_h(\tilde s_{root}\vert \tilde{s}_w,a):=1-P_h(\tilde{s}_w\vert \tilde{s}_w,a).
\end{equation*}
Then, the agent transverses the tree until she arrives at the leaves. Let $L$ be the number of leaves, and the set of the leaves is $\mathcal{L}=\{\tilde{s}_1,\cdots,\tilde{s}_L\}$. For any $\tilde{s}_i\in \mathcal{L}$, any action will lead to a transition to either $\tilde{s}_g$ or $\tilde{s}_b$ with the transition probability
\begin{equation} \label{eq:p}
P_h(\tilde{s}_g\vert \tilde{s}_i,a)=p+\Delta_{(h^*,s^*,a^*)}(h,\tilde{s}_i,a),\ P_h(\tilde{s}_b\vert \tilde{s}_i,a)=1-p-\Delta_{(h^*,s^*,a^*)}(h,\tilde{s}_i,a),
\end{equation}
where the parameter $p$ and the function $\Delta$ will be specified later. In \cite{domingues2021episodic}, $p$ is set to be $0.5$ in the risk-neutral setting, whereas we will tune $p$ in our risk-sensitive setting to obtain a tighter regret lower bound.  

The reward function is defined as 
\begin{equation*}
r_h(s,a):=1\{s=\tilde{s}_g,h\geq \bar{H}+d+1\},\ \forall a\in \mathcal{A} .
\end{equation*}
For each 
\begin{equation*}
(h^*,s^*,a^*)\in\{1+d,\cdots,\bar{H}+d\}\times\mathcal{L}\times\mathcal{A} =: \mathcal{Z},
\end{equation*}
we define an MDP $\mathcal{M}_{(h^*,s^*,a^*)}$, where $\Delta_{(h^*,s^*,a^*)}(h,\tilde{s}_i,a)=1\{h=h^*,\tilde{s}_i=s^*,a=a^*\}\epsilon$ and $\epsilon$ is a parameter to be chosen later. Denote by $P_{(h^*,s^*,a^*)}$ and $E_{(h^*,s^*,a^*)}$ the probability measure and expectation, respectively, in the MDP $\mathcal{M}_{(h^*,s^*,a^*)}$.
Let $\mathcal{M}_0$ be the MDP with $\Delta_0(h,\tilde{s}_i,a)=0$ for all $(h,\tilde{s}_i,a)\in [H]\times\mathcal{L}\times\mathcal{A}$. Denote by $P_0$ and $E_0$ the probability measure and expectation, respectively, in the MDP $\mathcal{M}_0$. 

\item  \textbf{Step 2: Computing the Expected Regret of an Algorithm in $\mathcal{M}_{(h^*,s^*,a^*)}$.}

We now compute
$Regret(\mathcal{M}_{(h^*,s^*,a^*)},\textbf{algo},K)$ for a learning algorithm \textbf{algo}, which executes policy $\pi^k$ in episode $k \in [K]$. By \eqref{eq:regret}, we need to compute the optimal value function, $V^*_1(s_1^k)$, and the value function under policy $\pi^k$, $V_1^{\pi^k}(s_1^k)$, for $k\in [K]$. 
Unlike \cite{domingues2021episodic}, these quantities can not be computed explicitly in general in our risk-sensitive setting because the OCE is defined by an optimization problem. Hence, in the following, we will bound $V_1^*(s_1^k)-V_1^{\pi^k}(s_1^k)$ in order to lower bound the regret. 

% In order to calculate the regret of $\pi$ of some algorithm in $\mathcal{M}_{(h^*,s^*,a^*)}$, we have to calculate the optimal value function, $V^*_1(s_1)$, and the value function under policy $\pi^k$, $V_1^{\pi^k}(s_1)$, $k\in [K]$. Due to the special structure of OCE, we are unable to calculate the value functions explicitly. We can only write down the formulas of regret and then bound them.

We first compute the value function $V_1^{\pi^k}(s_1^k)$ under policy $\pi^k$. For notational simplicity, we denote $\pi_h^k(s_h^k)$ as $a_h^k$ for all $h\in [H],k\in [K]$. Under policy $\pi^k$, we use $\hat{H}$ to denote the 
number of time steps in which the agent stays in the ``waiting" state, which is no larger than $\bar{H}$. Because there is no reward collected before step $\hat{H}+d$, we can obtain
\begin{align}\label{eq:V1pik}
V_1^{\pi^k}(s_1^k)=V_{\hat{H}+d}^{\pi^k}(s_{\hat{H}+d}^k).
\end{align}
Next, we compute $V_{\hat{H}+d}^{\pi^k}(s_{\hat{H}+d}^k).$ To this end, we first show 
\begin{align}\label{eq:VHd}
V_{\bar{H}+d+1}^{\pi^k}(s_{\bar{H}+d+1}^k)=(H-\bar{H}-d)\times 1\{s^k_{\bar{H}+d+1}=\tilde{s}_g\}.
\end{align}
We prove it by recursion. It is clear that 
\begin{align*}
V_{H}^{\pi^k}(s_{H}^k)&=r_H(s_{H}^k,a_{H}^k)+OCE_{s_{H+1}^k\sim P_H(\cdot\vert s_{H}^k,a_{H}^k)}(V_{H+1}^{\pi^k}(s_{H+1}^k))\overset{(1)}{=}1\{s_H^k=\tilde{s}_g\}\\
&\overset{(2)}{=}1\{s^k_{\bar{H}+d+1}=\tilde{s}_g\},
\end{align*}
where equality (1) holds because $V_{H+1}^{\pi^k}(s_{H+1}^k)=0$, and equality (2) follows from the fact that the agent is in the absorbing states when $h\geq \bar{H}+d+1$. Then, we can compute 
\begin{align*}
V_{H-1}^{\pi^k}(s_{H-1}^k)&=r_{H-1}(s_{H-1}^k,a_{H-1}^k)+OCE_{s^k_{H}\sim P_{H-1}(\cdot\vert s_{H-1}^k,a_{H-1}^k)}(V_{H}^{\pi^k}(s_{H}^k))\\
&\overset{(3)}{=}2\times 1\{s^k_{\bar{H}+d+1}=\tilde{s}_g\},
\end{align*}
where equality (3) holds because the random variable $1\{s^k_{\bar{H}+d+1}=\tilde{s}_g\}$ is known at step $H-1$, and we use property (b) in Lemma \ref{L0}. Repeating this procedure until time step $h=\bar{H}+d+1$, we obtain \eqref{eq:VHd}.  

Given \eqref{eq:VHd}, we next compute the value function under policy $\pi^k$ at time $\hat{H}+d+1$. Note that at time step $\hat{H}+d$, the agent is at the leaf of the rooted tree, where $\hat{H} \le \bar H.$ Hence, the probability that the agent is at good state $\tilde{s}_g$ at step $h=\hat{H}+d+1$ is the same as that at step $h=\bar{H}+d+1$. In addition, the reward function is given by $r_h(s,a)=1\{s=\tilde{s}_g,h\geq \bar{H}+d+1\},\forall a\in\mathcal{A}$. Hence, we obtain 
\begin{align*}
V_{\hat{H}+d+1}^{\pi^k}(s_{\hat{H}+d+1}^k)=V_{\bar{H}+d+1}^{\pi^k}(s_{\bar{H}+d+1}^k) 
&=(H-\bar{H}-d)\times 1\{s^k_{\bar{H}+d+1} 
=\tilde{s}_g\} \\
&=(H-\bar{H}-d)\times 1\{s^k_{\hat{H}+d+1}=\tilde{s}_g\}.
\end{align*}

% When $h= \bar{H}+d$, the agent might be at the leaf of the rooted tree or already at the absorbing state with no reward, because the reward function is $r_h(s,a)=1\{s=\tilde{s}_g,h\geq \bar{H}+d+1\},\forall a\in\mathcal{A}$. If the agent is already at the absorbing state, the time step $h$ keeps moving forward until the agent is at the leaf of the rooted tree. Denote the time step $h=\hat{H}+d$ as the time the agent is at the leaf of the rooted tree, where $\hat{H}=1,\cdots,\bar{H}$, thus the probability that the agent is at good state $\tilde{s}_g$ at step $h=\hat{H}+d+1$ is the same as that at step $h=\bar{H}+d+1$, which is $P_{(h^*,s^*,a^*)}(s^k_{\hat{H}+d+1}=\tilde{s}_g)=P_{(h^*,s^*,a^*)}(s^k_{\bar{H}+d+1}=\tilde{s}_g)$ and the value function at step $h=\hat{H}+d+1$ is 

It then follows that 
\begin{equation}\label{VLOW}
\begin{aligned}
&V_1^{\pi^k}(s_1^k) = V_{\hat{H}+d}^{\pi^k}(s_{\hat{H}+d}^k)\\
&=r_{\hat{H}+d}(s_{\hat{H}+d}^k,a_{\hat{H}+d}^k)+OCE_{s_{\hat{H}+d+1}^k\sim P_{\hat{H}+d}(\cdot\vert s_{\hat{H}+d}^k,a_{\hat{H}+d}^k)}(V_{\hat{H}+d+1}^{\pi^k}(s_{\hat{H}+d+1}^k))\\
&=OCE_{s_{\hat{H}+d+1}^k\sim P_{\hat{H}+d} (\cdot\vert s_{\hat{H}+d}^k,a_{\hat{H}+d}^k)}\left((H-\bar{H}-d)\times 1\{s^k_{\hat{H}+d+1}=\tilde{s}_g\}\right )\\
&=\sup_{\lambda\in [0,H-\bar{H}-d]}\{\lambda+P_{(h^*,s^*,a^*)}(s^k_{\hat{H}+d+1}=\tilde{s}_g)u(H-\bar{H}-d-\lambda)\\
&\quad \quad \quad \quad \quad \quad +(1-P_{(h^*,s^*,a^*)}(s^k_{\hat{H}+d+1}=\tilde{s}_g))u(-\lambda)\}\\
&=\sup_{\lambda\in [0,H-\bar{H}-d]}\{\lambda+P_{(h^*,s^*,a^*)}(s^k_{\bar{H}+d+1}=\tilde{s}_g)u(H-\bar{H}-d-\lambda)\\
&\quad \quad \quad \quad \quad \quad +(1-P_{(h^*,s^*,a^*)}(s^k_{\bar{H}+d+1}=\tilde{s}_g))u(-\lambda)\}.
\end{aligned}    
\end{equation}

Similar to Equation (7) in \cite{domingues2021episodic}, we can derive
\begin{equation}\label{S26}
\begin{aligned}
&P_{(h^*,s^*,a^*)}(s^k_{\bar{H}+d+1}=\tilde{s}_g)\\
&= \sum_{h=1+d}^{\bar{H}+d}pP_{(h^*,s^*,a^*)}(s^k_{h}\in\mathcal{L})+1\{h=h^*\}P_{(h^*,s^*,a^*)}(s_h^k=s^*,a_h^k=a^*)\epsilon\\
&= p+\epsilon \cdot P_{(h^*,s^*,a^*)}(s_{h^*}^k=s^*,a_{h^*}^k=a^*).
\end{aligned}
\end{equation}
Together with \eqref{VLOW}, we obtain an expression of the value function $V_1^{\pi^k}(s_1^k)$.

We next compute the optimal value function $V_1^{*}(s_1^k)$. Based on \eqref{VLOW}, one can easily show that the optimal policy is to let $P_{(h^*,s^*,a^*)}(s_{h^*}^k=s^*,a_{h^*}^k=a^*)=1$. Specifically, in the MDP $\mathcal{M}_{(h^*,s^*,a^*)}$, the optimal policy is to traverse the tree at step $h^*-d$, so that the agent visits the leaf state $s^*$ at time step $h^*$ and takes the action $a^*$ at this leaf state. Thus, the optimal value function is given by
\begin{align*}
V_1^{*}(s_1^k)&=\sup_{\lambda\in [0,H-\bar{H}-d]}\{\lambda+(p+\epsilon)u(H-\bar{H}-d-\lambda)+(1-p-\epsilon)u(-\lambda)\}.
\end{align*}

Now, we can compute that for each episode $k\in [K]$,
\begin{align*}
&V_1^*(s_1^k)-V_1^{\pi^k}(s_1^k)\\
&=\sup_{\lambda\in [0,H-\bar{H}-d]}\{\lambda+(p+\epsilon)u(H-\bar{H}-d-\lambda)+(1-p-\epsilon)u(-\lambda)\}\\
&-\sup_{\lambda\in [0,H-\bar{H}-d]}\{\lambda+P_{(h^*,s^*,a^*)}(s^k_{\bar{H}+d+1}=\tilde{s}_g)u(H-\bar{H}-d-\lambda)\\
&\quad \quad \quad \quad \quad \quad +(1-P_{(h^*,s^*,a^*)}(s^k_{\bar{H}+d+1}=\tilde{s}_g))u(-\lambda)\}\\
&\overset{(1)}{\geq} \rho+(p+\epsilon)u(H-\bar{H}-d-\rho)+(1-p-\epsilon)u(-\rho)\\
&\quad -\rho-P_{(h^*,s^*,a^*)}(s^k_{\bar{H}+d+1}=\tilde{s}_g)u(H-\bar{H}-d-\rho)-(1-P_{(h^*,s^*,a^*)}(s^k_{\bar{H}+d+1}=\tilde{s}_g))u(-\rho)\\
&\overset{(2)}{=} \epsilon [u(H-\bar{H}-d-\rho)-u(-\rho)]\times [1-P_{(h^*,s^*,a^*)}(s_{h^*}^k=s^*,a_{h^*}^k=a^*)],
\end{align*}
where inequality $(1)$ holds by setting 
\begin{align}
\rho\in &\mathop{\arg\max}_{\lambda\in [0,H-\bar{H}-d]}\{\lambda+P_{(h^*,s^*,a^*)}(s^k_{\bar{H}+d+1}=\tilde{s}_g)u(H-\bar{H}-d-\lambda)\nonumber\\
&\quad \quad \quad \quad \quad +(1-P_{(h^*,s^*,a^*)}(s^k_{\bar{H}+d+1}=\tilde{s}_g))u(-\lambda)\},\label{eq:rho}
\end{align}
and equality $(2)$ holds by applying \eqref{S26}. % into the formula and doing some simple calculations.

Therefore, the regret of a learning algorithm \textbf{algo} in $\mathcal{M}_{(h^*,s^*,a^*)}$ can be lower bounded as follow: 
\begin{align} \label{eq:regret-lb}
&Regret(\mathcal{M}_{(h^*,s^*,a^*)},\textbf{algo},K) \nonumber \\
&= \sum_{k=1}^K E_{(h^*,s^*,a^*)} [V_1^*(s_1^k)-V_1^{\pi^k}(s_1^k)] \nonumber \\
&\geq  \epsilon [u(H-\bar{H}-d-\rho)-u(-\rho)]\sum_{k=1}^K\left (1-P_{(h^*,s^*,a^*)}(s_{h^*}^k=s^*,a_{h^*}^k=a^*) \right ) \nonumber \\
&=  \epsilon [u(H-\bar{H}-d-\rho)-u(-\rho)]\left (K-E_{(h^*,s^*,a^*)}\left[N^K_{(h^*,s^*,a^*)}\right]\right ),
\end{align}
where 
\begin{align}\label{eq:N-star}
N^K_{(h^*,s^*,a^*)}=\sum_{k=1}^K 1\{s_{h^*}=s^*,a_{h^*}=a^*\}.
\end{align}

\item \textbf{Step 3: Bounding Maximum Regret over all possible $\mathcal{M}_{(h^*,s^*,a^*)}$.}

We can deduce from \eqref{eq:regret-lb} that
the maximum regret of an algorithm \textbf{algo} over all possible  $\mathcal{M}_{(h^*,s^*,a^*)}$ is lower bounded by 
\begin{equation}\label{S20}
\begin{aligned}
&\mathop{\max}_{(h^*,s^*,a^*) \in \mathcal{Z}} Regret(\mathcal{M}_{(h^*,s^*,a^*)},\textbf{algo},K)\\
&\geq \frac{1}{\bar{H}LA}\sum_{(h^*,s^*,a^*) \in \mathcal{Z}}Regret(\mathcal{M}_{(h^*,s^*,a^*)},\textbf{algo},K)\\
&\geq  K[u(H-\bar{H}-d-\rho)-u(-\rho)]\epsilon \left (1-\frac{1}{K\bar{H}LA}\sum_{(h^*,s^*,a^*)} E_{(h^*,s^*,a^*)}\left[N^K_{(h^*,s^*,a^*)}\right]\right ).
\end{aligned}
\end{equation}
So to lower bound the regret, we have to upper bound $\sum_{(h^*,s^*,a^*) \in \mathcal{Z} }E_{(h^*,s^*,a^*)} \left[N^K_{(h^*,s^*,a^*)}\right]$.

\item \textbf{Step 4: Bounding $\sum_{(h^*,s^*,a^*) \in \mathcal{Z} }E_{(h^*,s^*,a^*)}\left[N^K_{(h^*,s^*,a^*)}\right]$.}

For this step, we use similar arguments to those used in \cite{domingues2021episodic}; see page 13 therein.
Fix $(h^*,s^*,a^*)\in [H]\times \mathcal{S}\times \mathcal{A}$. Because $\frac{1}{K}N_{(h^*,s^*,a^*)}^K\in [0,1]$, one can obtain from Lemma 1 of \cite{garivier2019explore} that 
\begin{align*}
kl\left (\frac{1}{K}E_0\left[N_{(h^*,s^*,a^*)}^K\right],\frac{1}{K}E_{(h^*,s^*,a^*)}\left[N_{(h^*,s^*,a^*)}^K\right]\right )\leq KL(P_0,P_{(h^*,s^*,a^*)}),
\end{align*}
where KL denotes the Kullback-Leibler divergence between two probability measures and $kl(p,q)$ denotes the KL divergence between two Bernoulli distributions with success probabilities $p$ and $q$ respectively; see Definition 4 in \cite{domingues2021episodic}. It then follows from Pinsker's inequality, $(p-q)^2\leq \frac{1}{2}kl(p,q)$, that 
\begin{align*}
\frac{1}{K}E_{(h^*,s^*,a^*)}\left[N_{(h^*,s^*,a^*)}^K\right]\leq \frac{1}{K}E_0\left[N_{(h^*,s^*,a^*)}^K\right]+\sqrt{\frac{1}{2}KL(P_{0},P_{(h^*,s^*,a^*)})}.
\end{align*}
Because $\mathcal{M}_0$ and $\mathcal{M}_{(h^*,s^*,a^*)}$ differ at stage $h^*$ when $(s_{h^*},a_{h^*})=(s^*,a^*)$, by Lemma 5 of \cite{domingues2021episodic} and Lemma \ref{L16} in Appendix~\ref{sec:aux}, we can prove that 
\begin{align*}
KL\left(P_{0},P_{(h^*,s^*,a^*)}\right)=E_0\left[N_{(h^*,s^*,a^*)}^K\right]kl(p,p+\epsilon)\leq E_0\left[N_{(h^*,s^*,a^*)}^K\right]\frac{c_1\epsilon^2}{p},
\end{align*}
where $c_1\geq 2$ is a certain positive constant, $p\in [0,1-\frac{1}{c_1}]$ and $\epsilon$ satisfies
\begin{align}\label{eq:eps-constraint}
\epsilon\in \left [0,\frac{(1-2p)+\sqrt{1-\frac{4p}{c_1}}}{2}\right ].
\end{align}

Thus, 
\begin{align*}
\frac{1}{K}E_{(h^*,s^*,a^*)}\left[N_{(h^*,s^*,a^*)}^K\right]\leq \frac{1}{K}E_0\left[N_{(h^*,s^*,a^*)}^K\right]+\sqrt{\frac{c_1}{2p}}\epsilon\sqrt{E_0\left[N_{(h^*,s^*,a^*)}^K\right]}.
\end{align*}
According to the definition of $N_{(h^*,s^*,a^*)}^K$ in \eqref{eq:N-star}, we know that $\mathop{\sum}_{(h^*,s^*,a^*) \in \mathcal{Z} }N_{(h^*,s^*,a^*)}^K\leq K$. Then, by Cauchy-Schwarz inequality, we have
\begin{align}\label{S21}
\frac{1}{K}\mathop{\sum}_{(h^*,s^*,a^*) \in \mathcal{Z} }E_{(h^*,s^*,a^*)}\left[N_{(h^*,s^*,a^*)}^K\right]\leq 1+\sqrt{\frac{c_1}{2p}}\epsilon\sqrt{\bar{H}LAK}.
\end{align}

\item \textbf{Step 5: Optimizing $\epsilon$ and Choosing $\bar{H}$ and $p$.}

By combining \eqref{S20} with \eqref{S21}, we have 
\begin{align}
&\mathop{\max}_{(h^*,s^*,a^*) \in \mathcal{Z} } Regret(\mathcal{M}_{(h^*,s^*,a^*)},\textbf{algo},K)\nonumber\\
&\geq K[u(H-\bar{H}-d-\rho)-u(-\rho)]\epsilon \left (1-\frac{1}{\bar{H}LA}-\sqrt{\frac{c_1}{2p}}\epsilon\frac{\sqrt{\bar{H}LAK}}{\bar{H}LA}\right ),\label{Q39}
\end{align}
where the right-hand side of the inequality is a quadratic function of $\epsilon$. Maximizing this function by taking 
\begin{align}\label{eq:eps-choice}
\epsilon=\sqrt{\frac{p}{2c_1}}\left(1-\frac{1}{\bar{H}LA}\right)\sqrt{\frac{\bar{H}LA}{K}},
\end{align}
we derive
% The lower bound is maximized by taking $\epsilon=\sqrt{\frac{p}{2c_1}}\left(1-\frac{1}{\bar{H}LA}\right)\sqrt{\frac{\bar{H}LA}{K}}$. Hence, we take $\epsilon$ in to \eqref{Q39} and get
\begin{align}
&\mathop{\max}_{(h^*,s^*,a^*) \in \mathcal{Z}} Regret(\mathcal{M}_{(h^*,s^*,a^*)},\textbf{algo},K)\nonumber\\
&\geq \frac{1}{2\sqrt{2}}\sqrt{\frac{p}{c_1}}[u(H-\bar{H}-d-\rho)-u(-\rho)]\sqrt{\bar{H}LAK}(1-\frac{1}{\bar{H}LA})^2.\label{Q40}
\end{align}
According to Assumption \ref{ASSU2}, we have $A\geq 2$, $S\geq 6$, and thus $L=(1-\frac{1}{A})(S-3)+\frac{1}{A}\geq \frac{S}{4}$. Then, we can deduce from \eqref{Q40} that
\begin{align}
&\mathop{\max}_{(h^*,s^*,a^*) \in \mathcal{Z} } Regret(\mathcal{M}_{(h^*,s^*,a^*)},\textbf{algo},K)\nonumber\\
&\geq \frac{1}{2\sqrt{2}}\cdot \sqrt{\frac{p}{c_1}}[u(H-\bar{H}-d-\rho)-u(-\rho)]\sqrt{\bar{H}\cdot\frac{S}{4}AK}\cdot \frac{4}{9}\nonumber\\
&= \frac{1}{9\sqrt{2}}\cdot \sqrt{\frac{p}{c_1}}[u(H-\bar{H}-d-\rho)-u(-\rho)]\sqrt{SA\bar{H}K}.\label{Q41}
\end{align}

The bound in \eqref{Q41} is not explicit in the sense that the quantity $\rho$ defined in \eqref{eq:rho} depends on the unknown probability $P_{(h^*,s^*,a^*)}(s^k_{\bar{H}+d+1}=\tilde{s}_g)$, which equals $p+\epsilon \cdot P_{(h^*,s^*,a^*)}(s_{h^*}^k=s^*,a_{h^*}^k=a^*)$. We next lower bound the term on the right-hand-side of \eqref{Q41} in order to derive the explicit bound given in Theorem \ref{THM3}. 

Because $\rho$ satisfies \eqref{eq:rho}, by the first-order optimality condition, we have
\begin{align}\label{F38}
1\in P_{(h^*,s^*,a^*)}(s^k_{\bar{H}+d+1}=\tilde{s}_g)\partial u(H-\bar{H}-d-\rho)+(1-P_{(h^*,s^*,a^*)}(s^k_{\bar{H}+d+1}=\tilde{s}_g))\partial u(-\rho).
\end{align}
According to Assumption \ref{ASSU2}, we have $H\geq c_2 d$ with $c_2 >2$. We choose 
\begin{align}\label{eq:barH}
\bar H = \frac{H}{c_2}.
\end{align}
Then, by the monotonicity of the subgradients of the concave function $u$, we obtain from \eqref{F38} that
\begin{align}\label{F55}
1 \le P_{(h^*,s^*,a^*)}(s^k_{\bar{H}+d+1}=\tilde{s}_g)\partial u\left( \left(1- \frac{2}{c_2}\right)H-\rho\right)+(1-P_{(h^*,s^*,a^*)}(s^k_{\bar{H}+d+1}=\tilde{s}_g))\partial u(-\rho),
\end{align}
where the inequality means every element in the set on the right-hand-side is greater than one. Recall that  $P_{(h^*,s^*,a^*)}(s^k_{\bar{H}+d+1}=\tilde{s}_g) =p+\epsilon \cdot P_{(h^*,s^*,a^*)}(s_{h^*}^k=s^*,a_{h^*}^k=a^*)$. Then, we can deduce from \eqref{F38} that
\begin{align}\label{F38-2}
1 \le p \cdot \partial u\left( \left(1- \frac{2}{c_2}\right)H-\rho\right)+(1- p) \cdot \partial u(-\rho),
\end{align}
where we use the fact that $\partial u$ is monotone so that all
elements in the set $\partial u( (1- \frac{2}{c_2})H-\rho) - \partial u(-\rho)$ are all non-negative. Now consider the function $p \cdot \partial u( (1- \frac{2}{c_2})H-\lambda)+(1- p) \cdot \partial u(-\lambda)$ for $\lambda\in [0, \rho]$. When $\lambda=0$, it is clear that $ p \cdot \partial u( (1- \frac{2}{c_2})H)+(1- p) \cdot \partial u(0)$ contains an element that is smaller than one, because $1 \in \partial u(0)$ and the elements in $\partial u( (1- \frac{2}{c_2})H)$ are smaller than one. Together with \eqref{F38-2} and the continuity of the subdifferential mapping, we then deduce that 
% there exists some $\xi\in [0, \rho]$ such that
% \begin{equation}\label{Z48}
% 1 \in  p u'(H-\bar{H}-d-\xi)+(1-p) u'(-\xi).
% \end{equation}
there exists some $\lambda^* \in [0, \rho]$ such that
\begin{align}\label{eq:p-eq}
1 \in  p \cdot \partial u \left(\left(1-\frac{2}{c_2}\right)H-\lambda^*\right)+(1-p)\cdot \partial u(-\lambda^*).
\end{align}

Now, we are ready to lower bound the right-hand-side of \eqref{Q41}. Note that 
$ u(H-\bar{H}-d-\lambda) - u(-\lambda)$ is nondecreasing in $\lambda\in [0,H-\bar{H}-d]$. Using \eqref{eq:barH} and the assumption $H\geq 2c_2 d$,
we then have
\begin{align}
u(H-\bar{H}-d-\rho)-u(-\rho) \ge u \left(\left(1-\frac{2}{c_2}\right)H-\lambda^*\right) - u(-\lambda^*).  
\end{align}
For fixed $c_1 \ge 4$, we can choose 
\begin{align}
p = 1 - \frac{2}{c_1} \ge \frac{1}{2}.
\end{align}
It follows from \eqref{Q41} and \eqref{eq:barH} that 
\begin{align}
&\mathop{\max}_{(h^*,s^*,a^*) \in \mathcal{Z} } Regret(\mathcal{M}_{(h^*,s^*,a^*)},\textbf{algo},K)\nonumber\\
& \ge \frac{1}{9\sqrt{2}}\cdot \sqrt{\frac{p}{c_1}}[u(H-\bar{H}-d-\rho)-u(-\rho)]\sqrt{SA\bar{H}K} \nonumber \\
&\geq \frac{1}{18\sqrt{2c_1 c_2}}\cdot \left[u\left(\left(1-\frac{2}{c_2}\right) H-\lambda^*\right)-u(-\lambda^*)\right]\sqrt{SAHK}.
\end{align}
Finally, we need to make $\epsilon$ in \eqref{eq:eps-choice} satisfy the constraint \eqref{eq:eps-constraint}. It is easy to check that  $\epsilon \leq \sqrt{\frac{HSA}{2c_1 c_2 K}}$. Moreover, we have $\frac{(1-2p)+\sqrt{1-\frac{4p}{c_1}}}{2}\geq \frac{1}{c_1}$. Hence, we can choose $K\geq \frac{c_1 HSA}{2c_2}$ to make 
 $\epsilon$ in \eqref{eq:eps-choice} feasible. The proof is therefore completed. 

% Similar to the previous argument in \eqref{Q43}, continue from \eqref{Q45}, we have
% \begin{align}
% &\mathop{\max}_{(h^*,s^*,a^*)} Regret(\mathcal{M}_{(h^*,s^*,a^*)},\textbf{algo},K)\nonumber\\
% &\geq \frac{1}{18\sqrt{2c_1 c_2}}\cdot \left[u\left(\left(1-\frac{1}{c_2}\right) H-\lambda^*\right)-u(-\lambda^*)\right]\sqrt{SAHK}.
% \end{align}
% Note that $c_1\geq 4,c_2> 1,p=1-\frac{2}{c_1},\bar{H}=\frac{H}{2c_2}$, thus $\epsilon=\sqrt{\frac{p}{2c_1}}(1-\frac{1}{\bar{H}LA})\sqrt{\frac{\bar{H}LA}{K}}\leq \sqrt{\frac{HSA}{4c_1 c_2 K}}$. From Lemma \ref{L16}, we know $\epsilon\in \left [0,\frac{(1-2p)+\sqrt{1-\frac{4p}{c_1}}}{2}\right ]$. By some simple calculation, we have $\frac{(1-2p)+\sqrt{1-\frac{4p}{c_1}}}{2}\geq \frac{1}{c_1},$
% thus we can choose $K\geq \frac{c_1 HSA}{4c_2}$ to make $\epsilon$ feasible. 
\end{itemize}

\end{proof}

\subsection{An Auxiliary Lemma and Its Proof}\label{sec:aux}
% Here is one lemma for the proof of theorem \ref{THM3}. Before presenting Lemma \ref{L16}, we introduce some notations, which have been defined in the Definition 4 of \cite{domingues2021episodic}. 

Recall that for any $p,q\in (0,1)$ with $p+ q=1$, $kl(p,q)$ denotes the KL divergence between two Bernoulli distributions with success probabilities $p$ and $q$ respectively, i.e.,
\begin{align*}
    kl(p,q)=p\log\left(\frac{p}{q}\right)+q\log\left(\frac{1-p}{1-q}\right).
\end{align*}

\begin{lemma}\label{L16}
Fix any constant $c_1\geq 2$. If $p\in [0,1-\frac{1}{c_1}]$ and $\epsilon\in \left [0,\frac{(1-2p)+\sqrt{1-\frac{4p}{c_1}}}{2}\right ]$, then we have %$p+\epsilon\in [0,1]$ and 
$kl(p,p+\epsilon)\leq \frac{c_1 \epsilon^2}{p}$. % In addition, when $p=1-\frac{v}{c_1},v\geq 1$, we can calculate that $\frac{p}{c_1}\geq \frac{1}{2c_1}$ for any $c_1\geq 2v$.
\end{lemma}
\begin{proof}
Using the inequality $\log(1+x)\leq x$ for any $x>-1$, we have
\begin{align*}
kl(p,p+\epsilon)&=p\log\left (\frac{p}{p+\epsilon}\right )+(1-p)\log\left (\frac{1-p}{1-p-\epsilon}\right )\\
&\leq p\left (\frac{p}{p+\epsilon}-1 \right )+(1-p)\left (\frac{1-p}{1-p-\epsilon}-1 \right )\\
&=\frac{\epsilon^2}{(p+\epsilon)(1-p-\epsilon)}\\
&\overset{(1)}{\leq} \frac{c_1\epsilon^2}{p},
\end{align*}
where inequality (1) holds if we have
\begin{align*}
\frac{p}{c_1}\leq p(1-p)+(1-2p)\epsilon-\epsilon^2.
\end{align*}
One can easily verify that the above inequality holds if $p\in [0,1-\frac{1}{c_1}]$ and $\epsilon\in \left [0,\frac{(1-2p)+\sqrt{1-\frac{4p}{c_1}}}{2}\right ]$. The proof is then completed. 
% when $\epsilon\in \left [\frac{(1-2p)-\sqrt{1-\frac{4p}{c_1}}}{2},\frac{(1-2p)+\sqrt{1-\frac{4p}{c_1}}}{2}\right ]$. Note that $1-\frac{4p}{c_1}\geq 0$, i.e. $p\in [0,\frac{c_1}{4}]$. Because $\frac{c_1}{4}\geq 1-\frac{1}{c_1}$ for any $c_1\geq 2$, the range of $\epsilon$ holds for any $p\in [0,1-\frac{1}{c_1}]$. The upper bound of $\epsilon$ is positive when $p\in [0,1-\frac{1}{c_1}]$. When $c_1\geq 2$, we can see that
% \begin{align*}
% \frac{(1-2p)-\sqrt{1-\frac{4p}{c_1}}}{2}\leq \frac{(1-2p)-\sqrt{1-2p}}{2}\leq \frac{(1-2p)+\sqrt{1-2p}}{2} \leq \frac{(1-2p)+\sqrt{1-\frac{4p}{c_1}}}{2}.
% \end{align*}
% Since $\frac{(1-2p)-\sqrt{1-2p}}{2}\leq 0$ for any $p\in [0,\frac{1}{2}]$ and $\frac{(1-2p)-\sqrt{1-\frac{4p}{c_1}}}{2}\leq 0$ for any $p\in [\frac{1}{2},1-\frac{1}{c_1}]$, we have $\frac{(1-2p)-\sqrt{1-\frac{4p}{c_1}}}{2}\leq 0$ for any $p\in [0,1-\frac{1}{c_1}]$, and thus the range of $\epsilon$ can be $\left [0,\frac{(1-2p)+\sqrt{1-\frac{4p}{c_1}}}{2}\right ]$. Then, we can calculate
% \begin{align*}
% p+\epsilon&\leq p+\frac{(1-2p)+\sqrt{1-\frac{4p}{c_1}}}{2}\\
% &=\frac{1+\sqrt{1-\frac{4p}{c_1}}}{2}\\
% &\leq 1.
% \end{align*}
% Finally, choose $p=1-\frac{v}{c_1},v\geq 1$, we have $\frac{p}{c_1}=\frac{1}{c_1}-\frac{v}{c_1^2}\geq \frac{1}{2c_1}$ for any $c_1\geq 2v$, which completes the proof.
\end{proof}
% Remark that $c_1,v$ and $p$ are three parameters that will be tuned in the proof of Theorem \ref{THM3}. When $p=1-\frac{1}{c_1}$, we can calculate that $\frac{p}{c_1}=\frac{1}{c_1}-\frac{1}{c_1^2}\geq \frac{1}{2c_1}$ for any $c_1\geq 2$. However, when $p=1-\frac{1}{c_1}$, the upper bound of $\epsilon$ is $\frac{(1-2p)+\sqrt{1-\frac{4p}{c_1}}}{2}=\frac{-1+\frac{2}{c_1}+\sqrt{(1-\frac{2}{c_1})^2}}{2}=0$, i.e. $\epsilon=0$, which is difficult to obtain. So, we choose $v>1$ in the following part.

%%%%%%%%%%%%%%%%%%%%%%%%%%%%%%%%

\section{Numerical Experiments}\label{sec:experiment}
In this section, we conduct numerical experiments to illustrate the performance of the OCE-VI algorithm on randomly generated MDPs. 
%We construct a simple MDP and a complex MDP. In order to illustrate the necessity of risk-sensitive learning algorithms, we compare the mean-variance OCE-VI Algorithm to UCBVI-CH Algorithm and UCBVI-BF Algorithm in \cite{azar2017minimax}. We also compare entropic risk OCE-VI Algorithm to RSVI2 Algorithm in \cite{fei2021exponential} to test the performance of our algorithm. 

We adopt the methods in \citep[Section 4.7]{dann2019strategic} to randomly generate MDPs with state space $\mathcal{S}=\{1,\cdots,S\}$, action space $\mathcal{A}=\{1,\cdots,A\}$ and episode length $H$. For each $h=1, 2, \ldots, H$, the transition probabilities $P_h(\cdot\vert s,a)$ are generated independently from the Dirichlet distribution $Dir(0.1,\cdots,0.1)$. Reward functions $r_h(s,a)$ are set to $0$ with probability $85\%$ and 
generated independently from the uniform distribution $U[0,1]$ with probability $15\%$. In comparing the performance of different learning algorithms, we assume that the reward functions are known, but the transition probabilities are unknown.

% . We choose two sets of parameters $(H,S,A)=(3,6,3)$ and $(H,S,A)=(6,20,3)$. The state space is $\mathcal{S}=\{1,\cdots,S\}$ and the action space is $\mathcal{A}=\{1,\cdots,A\}$. Transition probabilities $P_h(\cdot\vert s,a)$ are generated independently from the Dirichlet distribution $Dir(0.1,\cdots,0.1)$. Reward functions $r_h(s,a)$ are generated independently from the uniform distribution $U[0,1]$ and are set to $0$ with probability $85\%$. 

%So the transition probabilities concentrate on certain transition pair $(s,a,s')$ at step $h$ for any $(s,a)\in \mathcal{S}\times\mathcal{A},h\in [H]$, and the reward functions are sparse. 

In our experiments we consider two different OCEs\footnote{For CVaR, our OCE-VI algorithm is essentially the ICVaR algorithm \cite{du2022risk}, so we do not compare their performances in the experiments.
}: entropic risk and mean-variance models. For entropic risk, we compare the performance of our OCE-VI algorithm with the RSVI2 and RSQ2 algorithm in \cite{fei2021exponential}. For mean-variance models, because there is no existing benchmark algorithm in the episodic RL setting with recursive mean-variance criterion, we compare our OCE-VI algorithm with the UCBVI-CH (with Chernoff-Hoeffding bonus)
and UCBVI-BF (with Bernstein bonus) algorithms in \cite{azar2017minimax} designed for the risk-neutral episodic RL. While the original UCBVI algorithms in \cite{azar2017minimax} are developed for MDPs with stationary transitions, we adapt them to our non-stationary MDP setting with time-dependent transition probabilities. 

% In terms of the environment, we assume that the rewards are known, but the transition probabilities are unknown. 

We consider two sets of parameters. The first one is $(H,S,A)=(3,6,3)$, and  we use the risk-aversion parameter $\beta=-0.6$ for the entropic risk and $c=\frac{1}{6}$ for the mean-variance models. We set
$K=10^6$ and $\delta=\frac{1}{2KH}$ for all algorithms. The second one is $(H,S,A)=(6,20,3)$, and  we use $\beta=-0.6$ for the entropic risk and $c=\frac{1}{12}$ for the mean-variance models. Because the size of the MDP becomes larger and learning can be more difficult in the second setting, we consider $K=10^7$ to show the sublinear regret (in $K$) of algorithms. 

% We consider the episodic learning and every episode restarts from a fixed initial state $s_1=1$. When $(H,S,A)=(3,6,3)$, we input $K=10^6,\delta=\frac{1}{2KH}$, $c=\frac{1}{6}$ for mean-variance and $\beta=-0.6$ for entropic risk. When $(H,S,A)=(6,20,30)$, we input $K=10^7,\delta=\frac{1}{2KH}$, $c=\frac{1}{12}$ for mean-variance and $\beta=-0.6$ for entropic risk. 
% We adjust the parameters for different $(H,S,A)$ to satisfy the problem formulation in section \ref{sec:formulation} and ensure the enough exploration of the algorithms. Since the transition probabilities in \cite{azar2017minimax} are independent of the time step $h$, we adapt their transition probabilities and the corresponding bonus to satisfy our time-dependent setting. 

%  with time-dependent transition dynamic

Figures \ref{F1} and \ref{F2} illustrate the performance comparisons of the OCE-VI algorithm with other algorithms, where we plot the average regret of each algorithm as a function of the number of episodes $K$. We compute the expected regret of each algorithm by averaging over 30 independent runs, but we do not plot the confidence intervals since the confidence intervals estimated from the 30 samples are very narrow compared with the magnitude of the regret and are almost invisible in the figures. We can observe from Figures \ref{F1} and \ref{F2} that for episodic RL with recursive entropic risk, our algorithm can outperform the RSVI2 algorithm in \cite{fei2021exponential} on randomly generated MDPs in the same risk-sensitive RL setting. For episodic RL with recursive mean-variance models, we find that our algorithm performs better than UCBVI algorithms in \cite{azar2017minimax}, though this is not surprising given that UCBVI is designed for the risk-netural RL setting.

% Figure \ref{F1} shows that the regrets depend on $K$ sublinearly, which is in accordance with Theorem \ref{THM1}. (a) and (b) are plotted under the situation $(H,S,A)=(3,6,3)$ while (c) and (d) are plotted under the situation $(H,S,A)=(6,20,3)$. When $(H,S,A)$ becomes larger, the MDP becomes more complex and the algorithms need more exploration to achieve the optimal policies, so the expected regrets in (c) and (d) become much larger than those in (a) and (b), which match our theoretical bounds in Theorem \ref{THM1}. (a) and (c) compares the mean-variance OCE-VI Algorithm with UCBVI-CH Algorithm and UCBVI-BF Algorithm in \cite{azar2017minimax}. They show that under the risk-sensitive setting, the mean-variance OCE-VI Algorithm can achieve much smaller regret than the risk-neutral algorithms even though their theoretical regret upper bounds are much smaller. Therefore, risk-sensitive learning algorithms are essential in many scenarios. Figure (b) and (d) compares the entropic risk OCE-VI Algorithm with RSVI2 Algorithm in \cite{fei2021exponential}. They show that our RL algorithm based on entropic risk can achieve much smaller expected regret than RSVI2 Algorithm under the same risk-sensitive environment.

\begin{figure}[htbp]
	\centering
 \begin{minipage}{0.49\linewidth}
		\centering
		\includegraphics[width=0.9\linewidth]{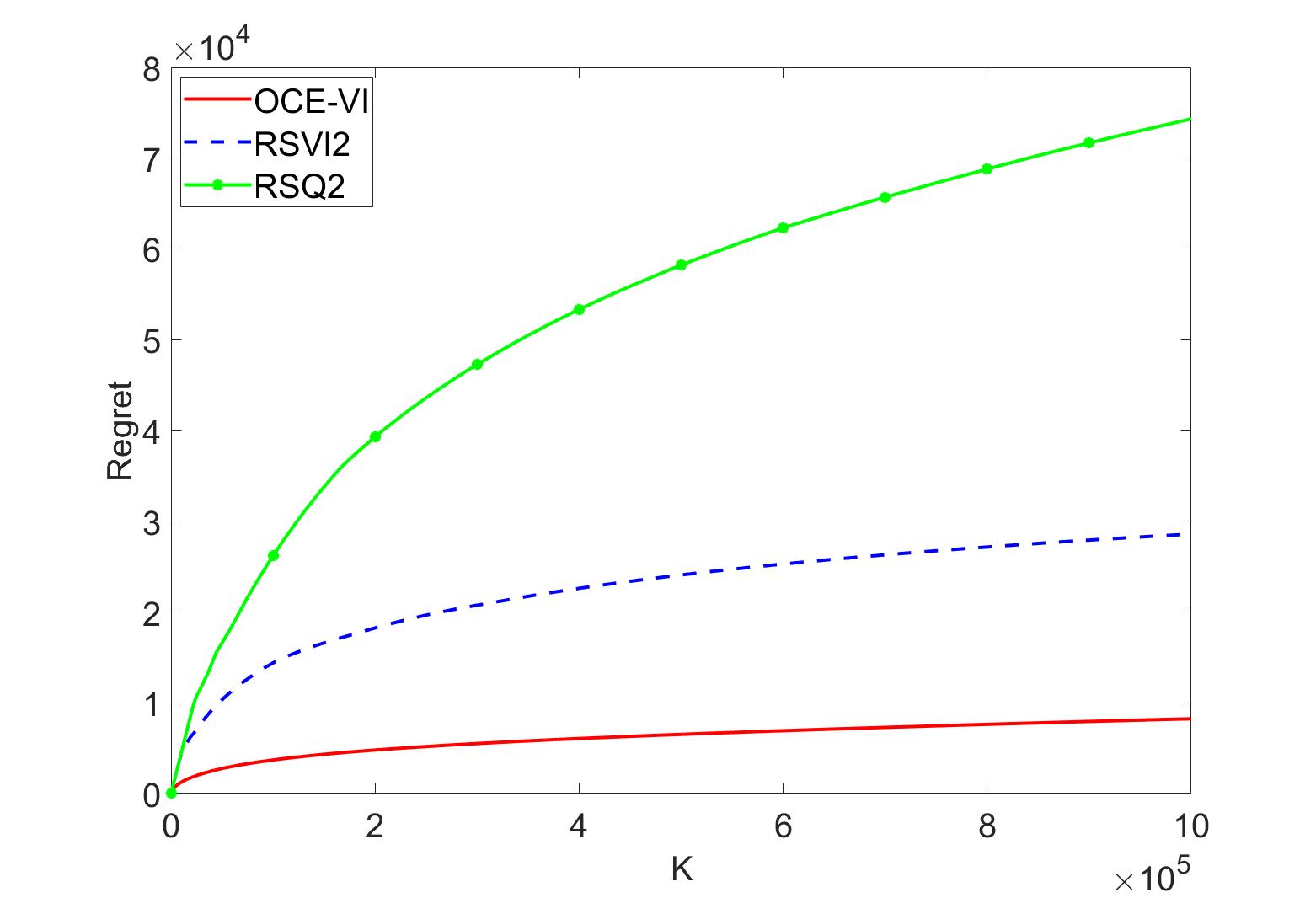}
		\subcaption{\centering}
		\label{OCEvsfei1}
	\end{minipage}
	\begin{minipage}{0.49\linewidth}
		\centering
		\includegraphics[width=0.9\linewidth]{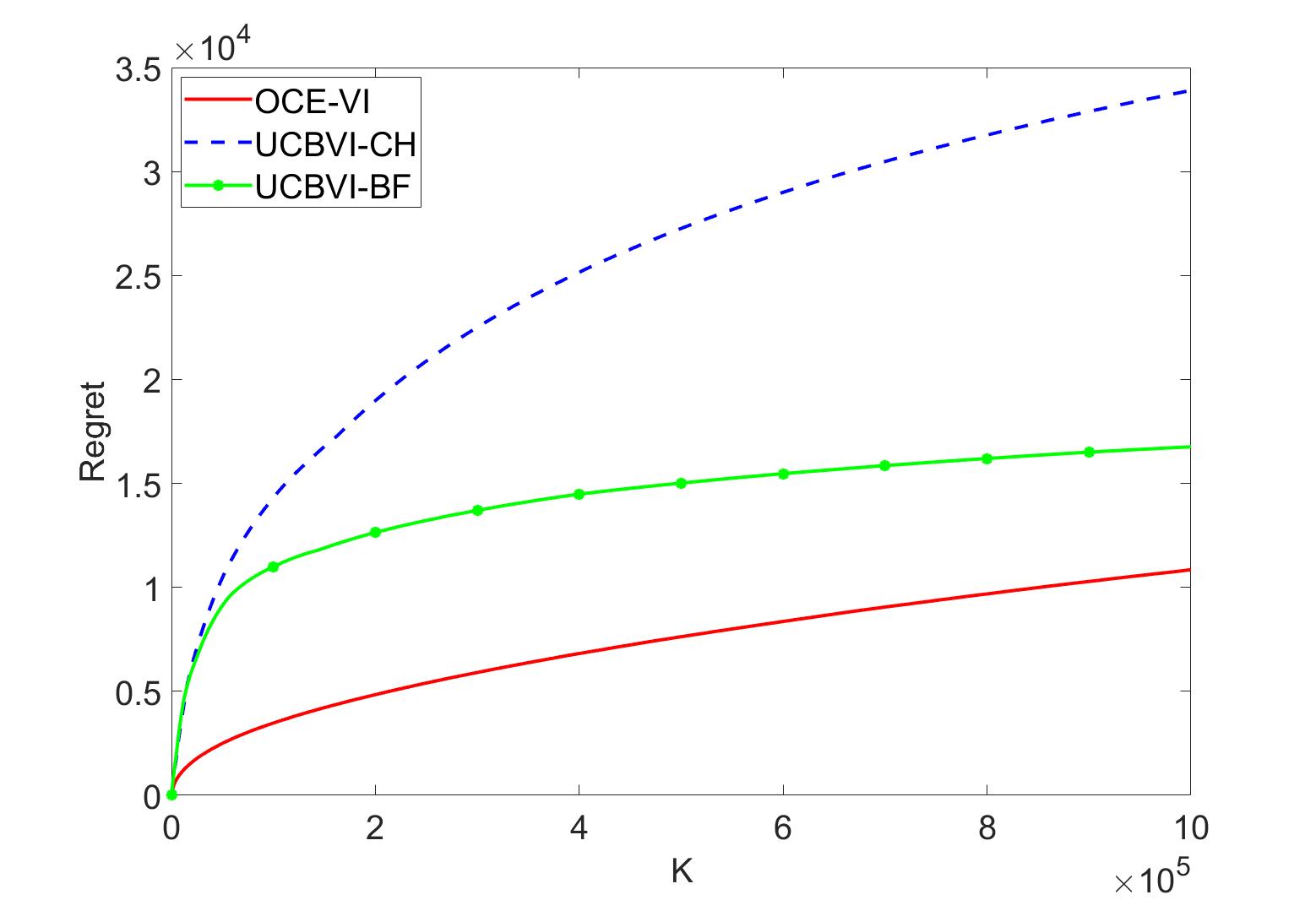}
		\subcaption{\centering}
		\label{poly1}
	\end{minipage}
	%\qquad
   \caption{Performance comparison of OCE-VI algorithm with other algorithms 
on a randomly generated MDP with $(H,S,A)=(3,6,3)$. Figure~\ref{OCEvsfei1} is for episodic RL with recursive entropic risk and Figure~\ref{poly1} is for the mean-variance models.  }
 \label{F1}
\end{figure}

\begin{figure}[htbp]
	\centering
 \begin{minipage}{0.49\linewidth}
		\centering
		\includegraphics[width=0.9\linewidth]{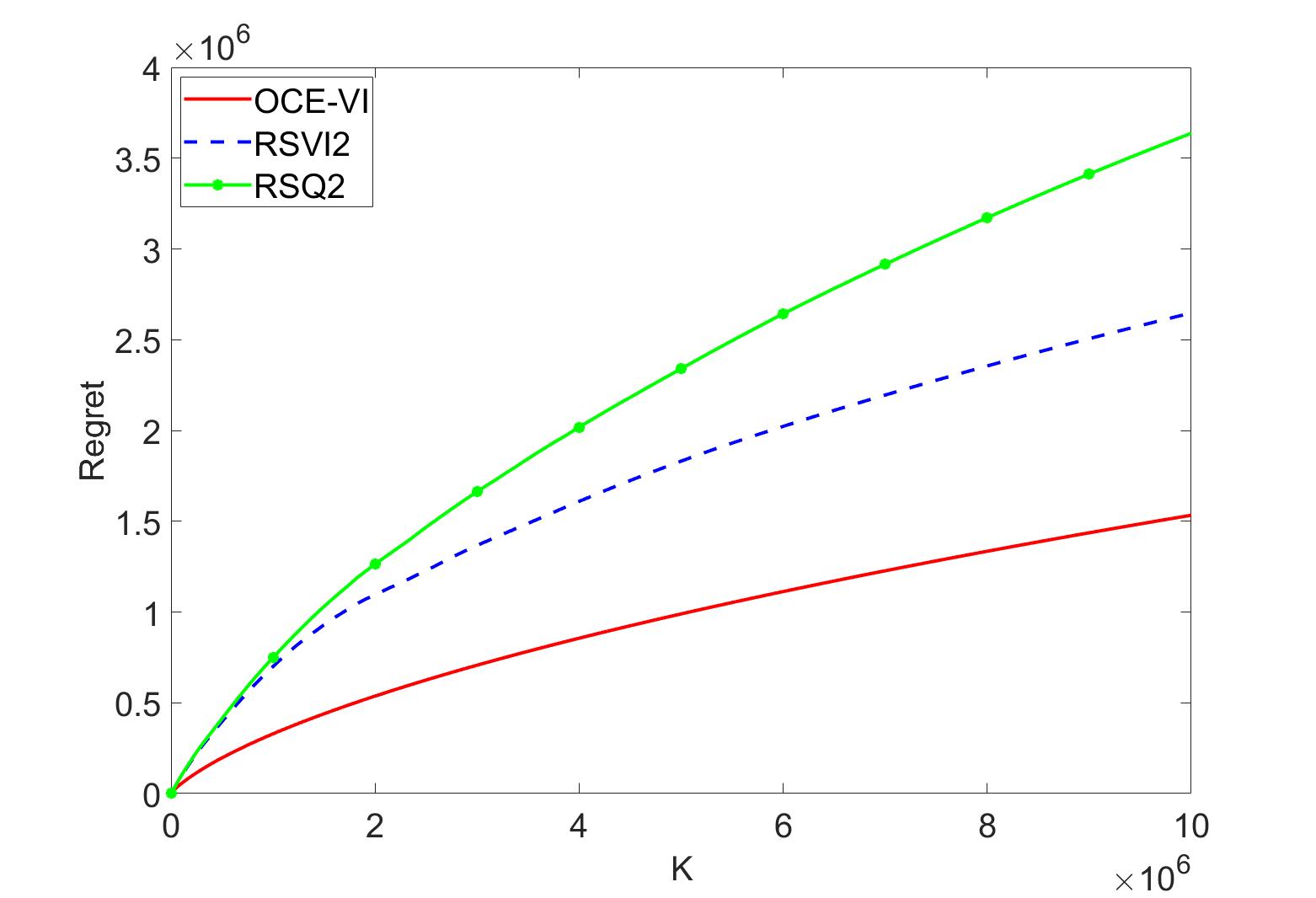}
		\subcaption{\centering}
		\label{OCEvsfei2}
	\end{minipage}
	\begin{minipage}{0.49\linewidth}
		\centering
		\includegraphics[width=0.9\linewidth]{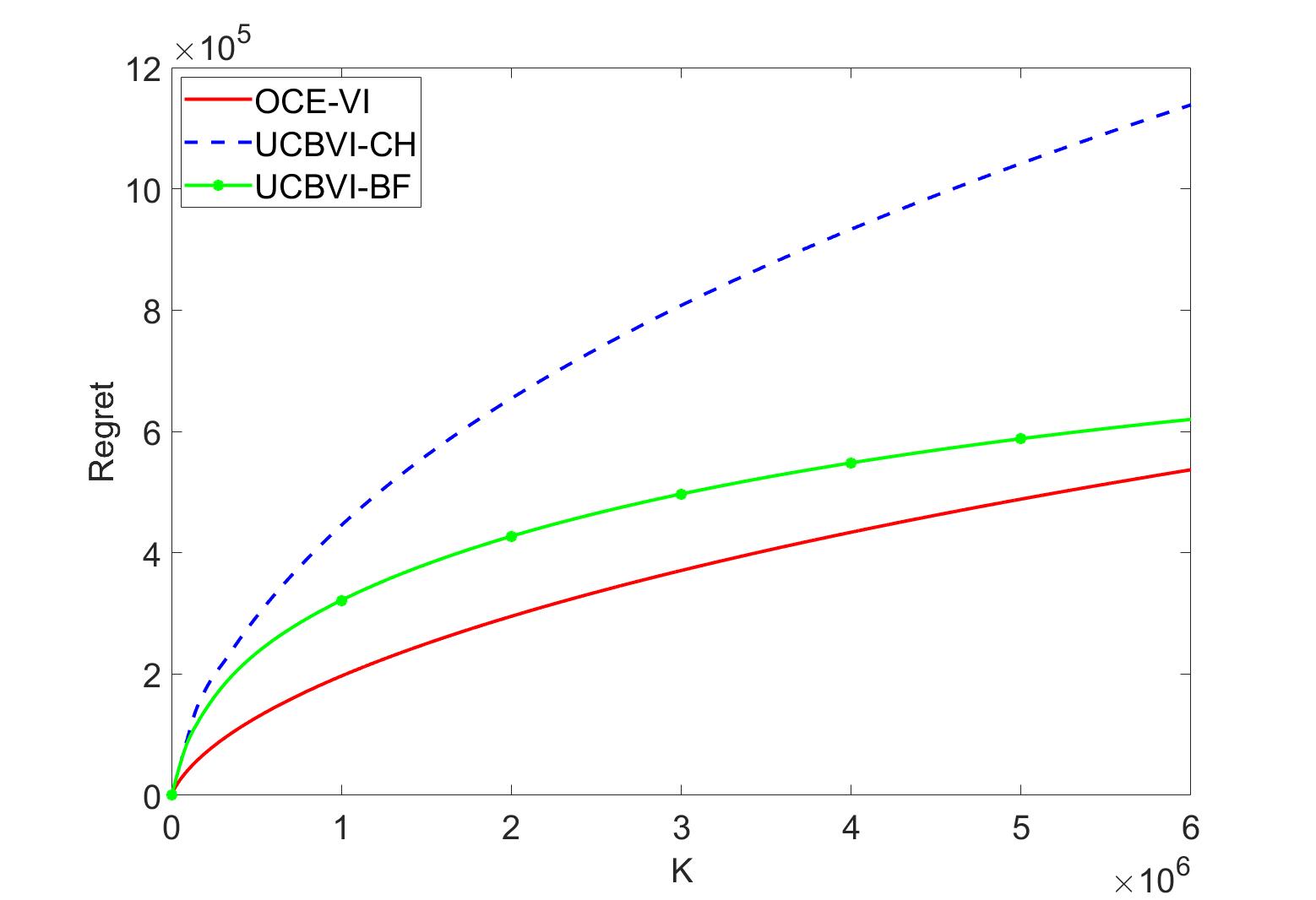}
		\subcaption{\centering}
		\label{poly2}
	\end{minipage}
	%\qquad
   \caption{Performance comparison of OCE-VI algorithm with other algorithms 
on a randomly generated MDP with $(H,S,A)=(6,20,3)$. Figure~\ref{OCEvsfei2} is for episodic RL with recursive entropic risk and Figure~\ref{poly2} is for the mean-variance models.  }
 \label{F2}
\end{figure}

\end{document}